\documentclass[opre,copyedit]{informs3}              

\usepackage{natbib}
\usepackage{xcolor}
\usepackage{amsbsy}
\usepackage{algpseudocode}
\usepackage{algorithm}
\usepackage{bbm}
\usepackage{appendix}
\usepackage{accents}
\newcommand{\banditMNL}{MNL-Bandit}
\usepackage{enumitem,amsmath,amssymb,graphicx}
 \NatBibNumeric
 \def\bibfont{\small}%
 \def\bibsep{\smallskipamount}%
 \def\bibhang{24pt}%
\usepackage{natbib}
 \bibpunct[, ]{(}{)}{,}{a}{}{,}%
 \def\bibfont{\small}%
 \def\bibsep{\smallskipamount}%
 \def\bibhang{24pt}%

\newtheorem{theorem1}{Theorem}
\newtheorem{remark}{Remark}

\newtheorem{lemma}{Lemma}[section]

\newtheorem{corollary}{Corollary}[section]
\newtheorem{assumption}{Assumption}[section]
\newtheorem{definition}{Definition}[section]

\def\ep#1{\mathcal{#1}}
\def\Halmos{\square}
\def\mb#1{\mathbf{#1}}
\newcommand{\change}[1]{\textcolor{black}{#1}}

\makeatother
\newcommand{\suchthat}{\;\ifnum\currentgrouptype=16 \middle\fi|\;}
\newenvironment{proofof}[1]{{\noindent\bf Proof of {#1}  }}{\hfill~$\Halmos$\vspace{3mm}}
\graphicspath{{plots/}}

\EquationsNumberedBySection 

\ABSTRACT{
We consider a dynamic assortment selection problem, where in every round the retailer offers a  subset (assortment) of $N$ substitutable products to a consumer, who selects one of these products according to  a  multinomial logit (MNL) choice model.  The retailer observes this choice and the  objective is to  dynamically learn the model parameters, while optimizing cumulative revenues over a selling horizon of length $T$. We refer to this exploration-exploitation formulation as the {\it MNL-Bandit problem}. Existing methods for this problem follow an \emph{explore-then-exploit} approach, which estimate parameters to a desired accuracy and then, treating these estimates as if they are the correct parameter values,  offers the optimal assortment based on these estimates. These approaches require certain a priori knowledge of ``separability,'' determined by the true parameters of the underlying MNL model,  and this in turn is critical in determining the length of the exploration period. (Separability refers to the distinguishability of the true optimal assortment from the other sub-optimal alternatives.) In this paper, we give an efficient algorithm that {\it simultaneously} explores and exploits, {without a priori knowledge of any problem parameters.} Furthermore, the algorithm is adaptive  in the sense that its performance is near-optimal in both the ``well separated'' case, as well as the general parameter setting  where this separation need not hold. 

}


\begin{document}

\TITLE{\Large MNL-Bandit: A Dynamic Learning Approach to Assortment Selection} 
\ARTICLEAUTHORS{%
\AUTHOR{ Shipra Agrawal} \AFF{Industrial Engineering and Operations Research, Columbia University, New York, NY. sa3305@columbia.edu}
\AUTHOR{Vashist Avadhanula} \AFF{Decision Risk and Operations, Columbia Business School, New York, NY. vavadhanula18@gsb.columbia.edu}
\AUTHOR{Vineet Goyal}\AFF{Industrial Engineering and Operations Research, Columbia University, New York, NY. vg2277@columbia.edu} 
\AUTHOR{Assaf Zeevi} \AFF{Decision Risk and Operations, Columbia Business School, New York, NY. assaf@gsb.columbia.edu}
}
\RUNTITLE{MNL-Bandit: A Dynamic Learning Approach to Assortment Selection}
\RUNAUTHOR{Agrawal, Avadhanula, Goyal and Zeevi}

\maketitle

\section{Introduction}
\subsection{Overview of the problem}
Assortment optimization problems arise widely in many industries including retailing and online advertising where the seller needs to select a subset  from a universe of substitutable items with the objective of maximizing expected revenue. Choice models capture substitution effects among products by specifying the probability that a consumer selects a product from the offered set. Traditionally, assortment decisions are made at the start of the selling period based on a choice model that has been estimated from historical data; see \cite{KokFisher07} for a detailed review. 

In this work, we focus on the dynamic version of the problem where the retailer needs to simultaneously learn consumer preferences and maximize revenue. In many business applications such as fast fashion and online retail, new products can be introduced or removed from the offered assortments in a fairly frictionless manner and the selling horizon for a particular product can be short. Therefore, the traditional approach of  first estimating the choice model and then using a static assortment based on the estimates, is not practical in such settings. 
Rather, it is essential to experiment with different assortments to learn consumer preferences, while simultaneously attempting to maximize immediate revenues. 
Suitable balancing of this exploration-exploitation tradeoff is the focal point of this paper. 


We consider a stylized dynamic optimization problem that captures some salient features of this application domain, where our goal is to develop an exploration-exploitation policy that simultaneously learns from current observations and exploits this information gain for future decisions. In particular, we consider a constrained assortment selection problem under the Multinomial logit (MNL) model with $N$ substitutable products and a ``no purchase'' option. Our goal is to offer a sequence of assortments, $S_1, \ldots, S_T$, where $T$ is the planning horizon, such that the cumulative expected revenues over said horizon is maximized, or alternatively,  minimizing the gap between the performance of a proposed policy and that of an oracle that knows instance parameters a priori, a quantity referred to as the {\it regret}. 


\medskip \noindent {\bf Related literature.} The Multinomial Logit model (MNL), owing primarily to its tractability,  is the most widely used  choice model for assortment selection problems.  (The model  was introduced independently by  \cite{Luce59} and  \cite{Plackett75}, see also \cite{train2003discrete, McFadden78,BenLerman85} for further discussion and survey of other commonly used choice models.) If the consumer preferences (MNL parameters in our setting) are known a priori, then the problem of computing the optimal assortment, which we refer to as the {\it static assortment optimization problem}, is well studied. \cite{TV04} consider the unconstrained assortment planning problem under the MNL model  and present a greedy approach to obtain the optimal assortment. Recent works of  \cite{davis2013assortment} and \cite{desir2014near} consider assortment planning problems under MNL with various constraints. Other choice models such as Nested Logit (\citealp{williams}, \citealp{DGT12}, \citealp{gallego2014constrained} and \citealp{dlevelnl}), Markov Chain (\citealp{BlanchetGG13} and \citealp{desirmarkovchain}) and more general models (\citealp{FJS12} and \citealp{gallego}) are also considered in the literature. 

Most closely related to our work are the papers of \cite{caro07}, \cite{rusme} and \cite{Saure}, where information on consumer preferences is not known and needs to be learned over the course of the selling horizon. \cite{caro07} consider the setting under which demand for products is independent of each other. \cite{rusme} and \cite{Saure} consider the problem of minimizing regret under the MNL choice model and present an ``explore first and exploit later'' approach. In particular, a selected set of assortments are explored until parameters can be estimated to a desired accuracy and then the optimal assortment corresponding to the estimated parameters is offered for the remaining selling horizon. The exploration period depends on certain a priori knowledge about instance parameters. Assuming that the optimal and next-best assortment are ``well separated,'' \cite{Saure} show an asymptotic $O(N\log{T})$ regret bound, while \cite{rusme} establish a $O(N^2\log^2{T})$ regret bound; recall $N$ is the number of products and $T$ is the time horizon. However, their algorithm relies crucially on a priori knowledge of system parameters which is not readily available in practice. As will be illustrated later, absence of this knowledge, these algorithms can perform quite poorly. In this work, we focus on approaches that simultaneously explore and exploit demand information, do not require any a priori knowledge or assumptions, and whose performance is in some sense best possible;
thereby, making our approach more universal in its scope.

Our problem is closely related to the multi-armed bandit (MAB) paradigm (cf. \citealp{Robbins1952}). A naive mapping to that setting would consider every assortment as an arm, and as such, given the combinatorial nature of the problem would lead to exponentially many arms.
Popular extensions of MAB for large scale problems include the linear bandit (e.g., \citealp{auer2003}, \citealp{linucb}) and generalized linear bandit (\citealp{genlinucb}) formulations. However, these do not apply directly to our problem, since the revenue corresponding to an assortment is nonlinear in problem parameters.  Other works (see \citealp{Chen}) have considered versions of MAB where one can play a subset of arms in each round and the expected reward is a function of rewards for the arms played. However, this approach assumes that the reward for each arm is generated independently of the other arms in the subset.
This is not the case typically in retail settings, and in particular, in the MNL choice model where 
purchase decisions depend on the assortment of products offered in a time step. In this work, we use the structural properties of the MNL model, along with techniques from MAB literature, to optimally explore and exploit in the presence of a large number of alternatives (assortments). 

 \subsection{Contributions}

\medskip \noindent {\bf Parameter independent online algorithm and regret bounds.} We give an efficient online algorithm that judiciously balances the exploration and exploitation trade-off intrinsic to our problem and achieves a worst-case regret bound of $O(\sqrt{NT \log {N}T})$ under a mild assumption, namely that the no-purchase is the most ``frequent'' outcome. The assumption regarding no-purchase is quite natural and a norm in online retailing for example. {To the best of our knowledge, this is the first such policy with provable regret bounds that does not require prior knowledge of instance parameters of the MNL choice model.}  
Moreover, the regret bound we present for this algorithm is non-asymptotic. 
The ``big-oh'' notation is used for brevity and only hides absolute constants. 

We also show that for ``well separated'' instances, the regret of our policy is bounded by {$O\left(\min\left(N^2 \log{NT}/\Delta,\sqrt{NT\log{NT}}\right)\right)$} where $\Delta$ is the ``separability'' parameter. \change{This is comparable to the regret bounds, $O\left(N \log{T}/\Delta\right)$ and $O\left(N^2 \log^2{T}/\Delta\right)$,  established in \cite{Saure} and \cite{rusme} respectively, even though we do not require any prior information on $\Delta$ unlike the aforementioned work}.   It is also interesting to note that the regret bounds hold true for a large class of constraints, e.g., we can handle matroid constraints such as assignment, partition and more general totally unimodular constraints (see \citealp{davis2013assortment}). Our algorithm is predicated on upper confidence bound (UCB) type logic,  originally developed to balance the aforementioned exploration-exploitation trade-off in the context of the multi-armed bandit (MAB) problem {(cf.  \citealp{Lai})}. In this paper the UCB approach, also known as optimism in the face of uncertainty, is customized to  the assortment optimization problem under the MNL model.

\medskip \noindent{\bf Lower bounds.} We establish a non-asymptotic lower bound for the online assortment optimization problem under the MNL model. In particular, we show that for the cardinality constrained problem under the MNL model, any algorithm must incur a regret of $\Omega(\sqrt{{NT}/{K}})$, where $K$ is the bound on the number of products that can be offered in an assortment. This bound is derived via a reduction 
to a parametric multi-armed bandit problem, for which such lower bounds are constructed by means of information theoretic arguments.  This result establishes that our online algorithm  is nearly optimal, the upper bound being within a factor of $\sqrt{K}$ of the lower bound. {A recent work by \cite{xichen} demonstrates a lower bound of $\Omega(\sqrt{NT})$ for the \banditMNL\;problem, thus suggesting that our algorithm's performance is optimal even with respect to its dependence on $K.$}

\medskip \noindent {\bf Computational study.} We present a computational study that highlights several salient features of our algorithm. In particular, we test the performance of our algorithm over instances with varying degrees of separability between optimal and sub-optimal solutions and observe that the performance is bounded irrespective of the ``separability parameter.'' In contrast, the approach of \cite{Saure} ``breaks down'' and results in linear regret for some values of the ``separability parameter.'' \change{We also present results of a simulated study on a real world data set, where we compare the performance of our algorithm to that of \cite{Saure}. We observe that the performance of our algorithm is sub-linear, while the performance of \cite{Saure} is linear, which further emphasizes the limitations of ``explore first and exploit later'' approaches and highlights the universal applicability of our approach.}

\vspace{2mm}
\noindent {\bf Outline}. In Section \ref{sec:problem_formulation}, we give the precise problem formulation. In Section \ref{sec:ourAlgorithm}, we present our algorithm for the \banditMNL~problem, and in Section \ref{sec:regretAnalysis}, we prove the worst-case regret bound of $\tilde{O}(\sqrt{NT})$ for our policy. \change{In Section~\ref{sec:lowerBound}, we present our non-asymptotic lower bound on regret for any algorithm for \banditMNL.  In Section \ref{sec:extn}, we 
present two extensions including improved logarithmic regret bound for ``well-separated'' instances and regret bound when the ``no purchase'' assumption is relaxed. In Section \ref{sec:computation}, we present results from our computational study.}


\section{Problem formulation}\label{sec:problem_formulation}
\textbf{The basic assortment problem.} In our problem, at every time instance $t$, the seller selects  an assortment $S_t \subset \{1, \ldots, N\}$ and observes the customer purchase decision $c_t \in S_t \cup \{0\}, ${ where $\{0\}$ denotes the no-purchase alternative, which is always available for the consumer.} As noted earlier, we assume consumer preferences are modeled using a multinomial logit (MNL) model. Under this model, the probability that a consumer purchases product $i$ at time $t$ when offered an assortment $S_t = S \subset \{1,\ldots,N\}$ is given by, 
\begin{equation}\label{choice_probabilities}
p_i(S) := \mathbb{P}\left( c_t = i | S_t = S \right) = \begin{cases}
 \displaystyle \frac{v_i}{v_0+\sum_{j \in S}v_j}, &\quad \text{if} \; i \in S \cup \{0\}\\
0, & \quad \; \text{otherwise},
\end{cases}
\end{equation}
for all $t$, where $v_i$ is the {\em attraction parameter} for product $i$ in the MNL model. The random variables $\{c_t: t = 1, 2, \ldots \}$ are conditionally independent, namely, $c_t$ conditioned on the event $\left\{S_t  = S\right\}$ is independent of $c_1, \ldots, c_{t-1}$. Without loss of generality, we can assume that $v_0 = 1$.  It is important to note that the parameters of the MNL model $v_i$, are not known to the seller. From \eqref{choice_probabilities}, the expected revenue when assortment $S$ is offered and the MNL parameters are denoted by the vector $\mb{v}$ is given by
\begin{equation}\label{MNL_revenue}
R(S, \mb{v}) = \mathbb{E} \left[\sum_{i\in S} r_i \mathbbm{1}\{c_t = i | S_t = S\}\right] = \sum_{i \in S} \frac{r_iv_i}{1+\sum_{j \in S} v_j},
\end{equation} 
where $r_i$ is the revenue obtained when product $i$ is purchased and is known a priori. 


We  consider several naturally arising constraints over the assortments that the retailer can offer. These include cardinality constraints (where there is an upper bound on the number of products that can be offered in the assortment), partition matroid constraints (where the products are partitioned into segments and the retailer can select at most a specified number of products from each segment) and joint display and assortment constraints (where the retailer needs to decide both the assortment as well as the display segment of each product in the assortment and there is an upper bound on the number of products in each display segment). More generally, we consider the set of totally unimodular (TU) constraints on the assortments. Let $\mb{x}(S) \in \{0,1\}^{N}$ be the incidence vector for assortment $S \subseteq \{1, \ldots, N\}$, i.e., $x_i(S) = 1$ if product $i \in S$ and $0$ otherwise. We consider constraints of the form
\begin{equation}\label{eq:feasible_set}
{\cal S} = \left\{ S \subseteq \{1,\ldots,N\} \; \middle |\; A \; \mb{x}(S) \leq \mb{b}, \; \mb 0 \leq \mb x \leq \mb 1 \right\},
\end{equation}
where $\mb A$ is a totally unimodular matrix {and $\mb{b}$ is integral (i.e., each component of the vector $\mb{b}$ is an integer).} The totally unimodular constraints model a rich class of practical assortment planning problems including the examples discussed above. We refer the reader to \cite{davis2013assortment} for a detailed discussion on assortment and pricing optimization problems that can be formulated under the TU constraints.

\medskip \noindent {\bf Admissible Policies.} To define the set of policies that can be used by the seller, let $U$ be a random variable, which encodes any additional sources of randomization and $(\mathbb{U}, \ep{U}, \mathbb{P}_{u})$ be the corresponding probability space.   We define $\{\pi_t, \; t= 1,2, \ldots \}$ to be measurable mappings as follows:
\begin{equation*}
\begin{aligned}
\pi_1: \;&\mathbb{U} \rightarrow \ep{S}\\
\pi_t: \;&\mathbb{U} \times \ep{S}^{t-1} \times \{0,\ldots, N\}^{t-1}  \rightarrow \ep{S}, \;\text{ for each}\;\; t = 2, 3, \ldots,
\end{aligned}
\end{equation*}
where $\ep{S}$ is as defined in \eqref{eq:feasible_set}. Then the assortment selection for the seller at time $t$ is given by 
\begin{equation} \label{eq:selection}
S_t = \begin{cases}
 \pi_1(U), &\quad \; t = 1\\
\pi_t(U, c_1, \ldots, c_{t-1}, S_1, \ldots, S_{t-1}), & \quad \;  t = 2, 3, \ldots.
\end{cases}
\end{equation}
For further reference, let $\{\ep{H}_t : t = 1,2, \ldots \}$ denote the filtration associated with the policy $\pi = (\pi_1, \pi_2, \ldots, \pi_t, \ldots)$. Specifically, 
\begin{equation*}
\begin{aligned}
\ep{H}_1 &= \sigma (U) \\
\ep{H}_t & = \sigma(U, c_1, \ldots, c_{t-1}, S_1, \ldots, S_{t-1}), \;\text{ for each}\;\; t = 2, 3, \ldots.
\end{aligned}
\end{equation*}
We denote by $\mathbb{P}_{\pi}\{.\}$ and $\mathbb{E}_{\pi}\{.\}$ the probability distribution and expectation value over path space induced by the policy $\pi$.

\medskip \noindent {\bf The online assortment optimization problem.} The objective  is to design a policy $\pi = (\pi_1, \ldots, \pi_T)$ that selects a sequence of history dependent assortments $(S_1, S_2, \ldots, S_T)$ so as to maximize the cumulative expected revenue, 
\begin{equation}\label{eq:cumulative_reward}
\mathbb{E}_{\pi} \left(\sum_{t=1}^T R(S_t, \mathbf{v}) \right),
\end{equation}
where $R(S,\mathbf{v})$ is defined as in \eqref{MNL_revenue}. Direct analysis of \eqref{eq:cumulative_reward} is not tractable given that the parameters $\{v_i, i = 1, \ldots, N\}$ are not known to the seller a priori. Instead we propose to measure the performance of a policy $\pi$ via its {\it regret}. The objective then is to design a policy that approximately minimizes the {\it regret} defined  as
\begin{equation}\label{Regret}
Reg_{\pi}(T,\mb{v}) = \sum_{t=1}^T R(S^*,\mb{v}) - \mathbb{E}_{\pi}[R(S_t, \mathbf{v})],
\end{equation}
where $S^*$ is the optimal assortment for \eqref{MNL_revenue}, namely, $S^* = \underset{S \in \mathcal{S}}{\text{argmax}} \;R(S,\mathbf{v}).$ This exploration-exploitation problem, which we refer to as {\bf \banditMNL}, is the focus of this paper.

\section{The proposed policy}
\label{sec:ourAlgorithm}
In this section, we describe our proposed policy for the \banditMNL~problem. The policy is designed using the characteristics of the MNL model based on the principle of optimism under uncertainty. 

\subsection{Challenges and overview}
A key difficulty in applying standard multi-armed bandit techniques to this problem is that the response observed on offering a product $i$ is {\it not} independent of other products in assortment $S$. Therefore, the $N$ products cannot be directly treated as $N$ independent arms. 
Our policy utilizes the specific properties of the dependence structure in MNL model to obtain an efficient algorithm 
with order $\sqrt{NT}$ regret. 

Our policy is based on a non-trivial extension of the UCB algorithm in \cite{auer}, {which is predicated on \cite{Lai}}. It uses the past observations to maintain increasingly accurate upper confidence bounds for the MNL parameters $\{v_i, i=1,\ldots, N\}$, and uses these to (implicitly) maintain an estimate of expected revenue $R(S,\mb{v})$ for every feasible assortment $S$. In every round, our policy picks the assortment $S$ with the highest optimistic revenue. There are two main challenges in implementing this scheme. First, the customer response to being offered an assortment $S$ depends on the entire set $S$, and does not directly provide an unbiased sample of demand for a product $i\in S$. In order to obtain unbiased estimates of $v_i$ for all $i\in S$, we offer a set $S$ multiple times: specifically, it is offered repeatedly until a no-purchase occurs. We show that proceeding in this manner, the average number of times a product $i$ is purchased provides an unbiased estimate of the parameter $v_i$. The second difficulty is the computational complexity of maintaining and optimizing revenue estimates for each of the exponentially many assortments. To this end, we use the structure of the MNL model and define our revenue estimates such that the assortment with maximum estimated revenue can be efficiently found by solving a simple optimization problem. This optimization problem turns out to be a static assortment optimization problem with upper confidence bounds for $v_i$'s as the MNL parameters, for which efficient solution methods are available.


\subsection{Details of the policy}
We divide the time horizon into epochs, where in each epoch we offer an assortment repeatedly  until a no purchase outcome occurs. Specifically, in each epoch $\ell$, we offer an assortment $S_\ell$ repeatedly. Let $\ep{E}_\ell$ denote the set of consecutive time steps in epoch $\ell$. $\ep{E}_\ell$ contains all time steps after the end of epoch $\ell-1$, until a no-purchase happens in response to offering $S_\ell$, including the time step at which no-purchase happens. The length of an epoch $|\ep{E}_\ell|$ conditioned on $S_\ell$ is a geometric random variable with success probability defined as the probability of no-purchase in $S_\ell$. The total number of epochs $L$ in time $T$ is implicitly defined as the minimum number for which $\sum_{\ell=1}^L |\ep{E}_\ell| \ge T$.  

At the end of every epoch $\ell$, we update our estimates for the parameters of MNL, which are used in epoch $\ell+1$ to choose assortment $S_{\ell+1}$. For any time step $t\in \ep{E}_\ell$, let $c_t$ denote the consumer's response to $S_{\ell}$, i.e., $c_t=i$ if the consumer purchased product {$i\in S_\ell$}, and $0$ if no-purchase happened. We define $\hat{v}_{i,\ell}$ as the number of times a product $i$ is purchased in epoch $\ell$,
\begin{equation}
\label{def:vhat}
\hat{v}_{i,\ell}:= \sum_{t \in \ep{E}_\ell} \mathbbm{1}(c_t = i). 
\end{equation}
For every product $i$ and epoch $\ell \leq L$, we keep track of the set of epochs before $\ell$ that offered an assortment containing product $i$, and the number of such epochs. We denote the set of epochs by $\mathcal{T}_i(\ell)$ and the number of epochs by $T_i(\ell)$. That is,
\begin{equation}
\label{def:Til}
\mathcal{T}_i(\ell) = \left\{\tau \leq \ell \, \middle| \, i \in S_\tau\right\},\ \ T_i(\ell)=|\mathcal{T}_i(\ell)|.
\end{equation}
We compute $\bar{v}_{i,\ell}$ as the number of times product $i$ was purchased per epoch, 
\begin{equation}
\label{def:vbar}
\bar{v}_{i,\ell} = \frac{1}{T_i(\ell)}\sum_{\tau \in \mathcal{T}_i(\ell)} \hat{v}_{i,\tau}.
\end{equation}
We show that for all $i\in S_\ell$, $\hat{v}_{i,\ell}$ and $\bar{v}_{i,\ell}$ are unbiased estimators of the MNL parameter $v_i$ (see Corollary \ref{unbiased_estimate} )
Using these estimates, we compute the upper confidence bounds, $v^{\sf UCB}_{i,\ell}$ for $v_i$ as, 
\begin{equation}
\label{def:vUCB}
v^{\sf UCB}_{i,\ell}  :=\bar{v}_{i,\ell} + \displaystyle \sqrt{\bar{v}_{i,\ell}\frac{48\log{(\sqrt{N}\ell+1)}}{T_i(\ell)}} + \frac{48\log{(\sqrt{N}\ell+1)}}{T_i(\ell)}.
\end{equation}
We establish that $v^{\sf UCB}_{i,\ell}$ is an upper confidence bound on the true parameter $v_i$, i.e., $v^{\sf UCB}_{i,\ell} \ge v_i$,  for all $i,\ell$ with high probability (see Lemma \ref{lem:UCBv}). The role of the upper confidence bounds is akin to their role in hypothesis testing; they ensure that the likelihood of identifying the parameter value is sufficiently large. 
We then offer the optimistic assortment in the next epoch, based on the previous updates as follows,
\begin{equation}\label{eq:optimistic_assort}
S_{\ell+1} := \underset{S \in \mathcal{S}}{\text{argmax}}\; {\max} \left\{R(S,\mb{\hat{v}}) : \hat{v}_i \leq v^{\sf UCB}_{i,\ell}\right\},
\end{equation}
where $R(S,\mb{\hat{v}})$ is as defined in \eqref{MNL_revenue}. We later show that the above optimization problem is equivalent to the following optimization problem (see Lemma \ref{UCB_bound}).
\begin{equation}\label{eq:update_assort}
S_{\ell+1} := \underset{S \in \mathcal{S}}{\text{argmax}}\;\tilde{R}_{\ell+1}(S),
\end{equation}
where $\tilde{R}_{\ell+1}(S)$ is defined as,
\begin{equation}
\label{def:RtildeS}
\tilde{R}_{\ell+1}(S) := \frac{\displaystyle\sum_{i \in S}r_i v^{\sf UCB}_{i,\ell}}{1+\displaystyle\sum_{j \in S} v^{\sf UCB}_{j,\ell}}.
\end{equation}
We summarize the steps in our policy in Algorithm \ref{learn_algo}. 
\begin{algorithm}
\begin{algorithmic}[1]
\State \textbf{Initialization:} $v^{\sf UCB}_{i,0} = 1$ for all $i=1,\ldots,N$

\State $t = 1$ ; $\ell = 1$ keeps track of the time steps and total number of epochs respectively

\While{$t < T$}

	\State Compute $S_{\ell} = \underset{S \in \mathcal{S}}{\text{argmax}}\;\tilde{R}_{\ell}(S) = \frac{\displaystyle\sum_{i \in S}r_i v^{\sf UCB}_{i,\ell-1}}{1+\displaystyle\sum_{j \in S} v^{\sf UCB}_{j,\ell-1}}$
	\State Offer assortment $S_{\ell}$, observe the purchasing decision, ${c_t}$ of the consumer

	\If{$c_t = 0$}

		\State compute $\hat{v}_{i,\ell} = \sum_{t \in \ep{E}_\ell} \mathbbm{1}(c_t = i)$, no. of consumers who preferred $i$ in epoch $\ell$, for all $i \in S_\ell$
		\State update $\mathcal{T}_i(\ell) = \left\{\tau \leq \ell \, \middle| \, i \in S_\ell\right\}, T_i(\ell)=|\mathcal{T}_i(\ell)|$, no. of epochs until $\ell$ that offered  product $i$
		\State update $\bar{v}_{i,\ell}$ = $\displaystyle \frac{1}{T_i(\ell)}\sum_{\tau \in \ep{T}_i(\ell)} \hat{v}_{i,\tau}$, sample mean of the estimates
		\State	update $ v^{\sf UCB}_{i,\ell}$  =$\bar{v}_{i,\ell} + \displaystyle \sqrt{\bar{v}_{i,\ell}\frac{48\log{({\sqrt{N}}\ell+1)}}{T_i(\ell)}} + \frac{48\log{({\sqrt{N}}\ell+1)}}{T_i(\ell)}$;\; $\ell = \ell + 1$	

	\Else

		\State $\ep{E}_\ell = \ep{E}_\ell \cup t$, time indices corresponding to epoch $\ell$

	\EndIf
		\State $t = t+1$

\EndWhile
\end{algorithmic}
\caption{Exploration-Exploitation algorithm for \banditMNL}
\label{learn_algo}
\end{algorithm}
Finally, we may remark on the computational complexity of implementing \eqref{eq:optimistic_assort}. The optimization problem~\eqref{eq:optimistic_assort} is formulated as a static assortment optimization problem under the MNL model with TU constraints, with model parameters being $v^{\sf UCB}_{i,\ell}, i=1,\ldots, N$ (see \eqref{eq:update_assort}). There are efficient polynomial time algorithms to solve the static assortment optimization problem under MNL model with known parameters (see~\citealp{avadhanula2016tightness}, \citealp{davis2013assortment}, \citealp{rusme}). {We will now briefly comment on how Algorithm \ref{learn_algo} is different from the existing approaches of \cite{Saure} and \cite{rusme} and also why other standard ``bandit techniques'' are not applicable to the \banditMNL\;problem.}
\begin{remark}{\normalfont{\textbf{(Universality)}} \normalfont { Note that Algorithm \ref{learn_algo} does not require any prior knowledge/information about the problem parameters $\mb{v}$ (other than the assumption $v_i \leq v_0$, which is subsequently relaxed in Algorithm \ref{learn_algo_extn}). This is in contrast with the approaches of \cite{Saure} and \cite{rusme}, which require the knowledge of the ``separation gap,'' namely, the difference between the expected revenues of the optimal assortment and the second-best assortment. Assuming knowledge of  this ``separation gap,'' both these existing approaches explore a pre-determined set of assortments to estimate the MNL parameters within a desired accuracy, such that the optimal assortment corresponding to the estimated parameters is the (true) optimal assortment with high probability.  This forced exploration of pre-determined assortments is avoided in Algorithm \ref{learn_algo}, which offers assortments adaptively, based on the current observed choices. The confidence regions derived for the parameters $\mb{v}$ and the subsequent assortment selection, ensure that Algorithm \ref{learn_algo} judiciously maintains the balance between exploration and exploitation that is central to the \banditMNL\;problem.} }\end{remark}

 \begin{remark}{\normalfont{\textbf{(Estimation Approach)}} \normalfont Because the \banditMNL\;problem is parameterized with parameter vector ($\mb{v}$), a natural approach is to build on standard estimation approaches like maximum likelihood (MLE), where the estimates are obtained by optimizing a loss function. \change{However, the confidence regions for estimates resulting from such approaches are either: 
 \begin{enumerate}
 \item asymptotic and are not necessarily valid for finite time with high probability, or
 \item typically depend on true parameters, which are not known a priori. For example, finite time confidence regions associated with maximum likelihood estimates require the knowledge of $\underset{\mb{v} \in \mathcal{V}}{\sup}\;I(\mb{v})$ (see \citealp{borovkov1984mathematical}), where $I$ is the Fisher information of the MNL choice model and $\ep{V}$ is the set of feasible parameters (that is not known apriori).  Note that using $I(\mb{v}^{\sf MLE})$ instead of $\underset{\mb{v} \in \mathcal{V}}{\sup}\;I(\mb{v})$ for constructing confidence intervals would only lead to asymptotic guarantees and not finite sample guarantees.
 \end{enumerate} }
 In contrast, in Algorithm \ref{learn_algo},  we solve the estimation problem by a sampling method designed to give us unbiased estimates of the model parameters. The confidence bounds of these estimates and the algorithm do not depend on the underlying model parameters. Moreover, our sampling method allows us to compute the confidence regions by simple and efficient ``book keeping'' and avoids computational issues that are typically associated with standard estimation schemes such as MLE. \change{Furthermore, the confidence regions associated with the unbiased estimates also facilitate a tractable way to compute the optimistic assortment (see \eqref{eq:optimistic_assort}, \eqref{eq:update_assort} and Step-4 of Algorithm \ref{learn_algo}), which is less accessible for the MLE estimate.}}\end{remark}
 
 \begin{remark}{\normalfont{\textbf{(Alternative Approaches)}} \change{Recently, Thompson Sampling (TS) has attracted considerable attention and several studies (\citealp{chapelle}, \citealp{may2012optimistic}) have demonstrated that TS significantly outperforms the state of the art bandit policies in practice. Typically, TS approaches proceed by assuming a prior distribution on the underlying parameters ($\mb{v}$ in the \banditMNL\;problem) and at every time step the posterior distribution on the parameters is updated based on the observed rewards and an arm (assortment) is selected with its posterior probability of it being the best arm. To implement a TS approach for the \banditMNL\;problem, one would need to specify the choice of prior, address the tractability of posterior sampling, etc. These issues also impede the analysis of such an algorithm. For example, in all existing work (\citealp{AgrawalTS_nearopt}, \citealp{AgrawalTS_lin}) on worst-case regret analysis for TS, the prior is chosen to allow a conjugate posterior, which permits theoretical analysis. For general posteriors, only Bayesian regret bounds have been proven, which are much weaker than the regret notion we consider in this paper. We return to discuss TS sampling in the concluding remarks of the paper.}} 
 \end{remark}
\section{Main results}
\label{sec:regretAnalysis}
{\normalfont In what follows, we make the following assumptions.}\vspace{-5mm}
\begin{assumption}\label{assumption1}
~~\\
\vspace{-6mm}
\begin{enumerate}
\item The MNL parameter corresponding to any product $i \in \{1,\ldots, N\}$ satisfies $v_i \leq v_0 = 1$.
\item The family of assortments ${\cal S}$ is such that $S \in {\cal S}$ and $Q \subseteq S$ implies that $Q \in {\cal S}$. 
\end{enumerate}
\end{assumption}
The first assumption is equivalent to the `no purchase option' being the most likely outcome. We note that this holds in many realistic settings, in particular,  in online retailing and online display-based  advertising. The second assumption implies that removing a product from a feasible assortment preserves feasibility. This holds for most constraints arising in practice including cardinality, and matroid constraints more generally.  We would like to note that the first assumption is made for ease of presentation of the key results and is not central to deriving bounds on the regret. In section \ref{sec:modified_policy}, we relax this assumption and derive regret bounds that hold for any parameter instance.

Our main result is the following upper bound on the regret of the policy stated in Algorithm \ref{learn_algo}.
\begin{theorem1}[Performance Bounds for Algorithm \ref{learn_algo}]\label{main_result}
For any instance $\mb{v} = (v_0, \ldots, v_N)$ of the \banditMNL~problem with $N$ products,  $r_i\in [0,1]$ and Assumption \ref{assumption1}, the regret of the policy given by Algorithm~\ref{learn_algo} at any time $T$ is bounded as,
\begin{equation*} 
Reg_\pi(T,\mb{v}) \leq C_1\sqrt{NT\log{{N}T}} + C_2N\log^2{{N}T},
\end{equation*}
where $C_1$ and $C_2$ are absolute constants (independent of problem parameters).
\end{theorem1}
\subsection{Proof Outline}\label{proof_outline_thm1}
In this section, we provide an outline of different steps involved in proving Theorem~\ref{main_result}. 

\vspace{2mm}
\medskip \noindent{\bf Confidence intervals.} The first step in our regret analysis is to prove the following two properties of the estimates $v^{UCB}_{i,\ell}$ computed as in \eqref{def:vUCB} for each product $i$.  Specifically,  that $v_i$ is bounded by $v_{i,\ell}^{\sf UCB}$ with high probability, and that as a product is offered an increasing number of times, the estimates $v^{\sf UCB}_{i,\ell}$ converge to the true value with high probability. Intuitively, these properties establish $v^{UCB}_{i,\ell}$ as upper confidence bounds converging to actual parameters $v_i$, akin to the upper confidence bounds used in the UCB algorithm for MAB in \cite{auer}. We provide the precise statements for the above mentioned properties in Lemma \ref{lem:UCBv}. These properties follow from an observation that is conceptually equivalent to the IIA (Independence of Irrelevant Alternatives) property of MNL, and shows that in each epoch $\tau$, $\hat{v}_{i,\tau}$ (the number of purchases of product $i$) provides an independent unbiased estimates of $v_i$. Intuitively, $\hat{v}_{i,\tau}$ is the ratio of probabilities of purchasing product $i$ to preferring product $0$ (no-purchase), which is independent of $S_\tau$. This also explains why we choose to offer $S_\tau$ repeatedly until no-purchase occurs. Given these unbiased i.i.d. estimates from every epoch $\tau$ before $\ell$, we apply a multiplicative Chernoff-Hoeffding bound to prove concentration of $\bar{v}_{i,\ell}$. 

\medskip \noindent{\bf Validity of the optimistic assortment.} The product demand estimates $v^{\sf UCB}_{i,\ell-1}$ were used in \eqref{def:RtildeS} to  define expected revenue estimates $\tilde{R}_\ell(S)$ for every set $S$. In the beginning of every epoch $\ell$, Algorithm \ref{learn_algo} computes the optimistic assortment as $S_\ell = \arg\max_S \tilde{R}_\ell(S)$, and then offers $S_{\ell}$ repeatedly until no-purchase happens.  \change{The next step in the regret analysis is to leverage the fact that $v^{\sf UCB}_{i,\ell}$ is an upper confidence bound on $v_i$ to prove similar, though slightly weaker, properties for the estimates $\tilde{R}_\ell(S)$.} First, we show that estimated revenue is an upper confidence bound on the optimal revenue, i.e., $R(S^*,\mb{v})$ is bounded by $\tilde{R}_\ell(S_\ell)$ with high probability. The proof for these properties involves careful use of the structure of MNL model to show that the value of $\textstyle \tilde{R}_\ell(S_\ell) $ is equal to the highest expected revenue achievable by any feasible assortment, among { all instances of the problem with parameters in the range} $[0, v^{\sf UCB}_i], i=1,\ldots, n$. Since the actual parameters lie in this range with high probability, we have $\tilde{R}_\ell(S_\ell)$ is at least $R(S^*,\mb{v})$ with high probability. Lemma \ref{lem:UCBS1} provides the precise statement. 

\medskip \noindent{\bf Bounding the regret.} The final part of our analysis is to bound the regret in each epoch. \change{First, we use the fact that $\tilde{R}_\ell(S_\ell)$ is an upper bound on $R(S^*,\mb{v})$ to bound the loss due to offering the assortment $S_\ell$.} In particular, we show that the loss is bounded by the difference between the ``optimistic'' revenue estimate, $\tilde{R}_\ell(S_\ell)$, and the actual expected revenue, $R(S_\ell)$. We then prove a Lipschitz property of the expected revenue function to bound the difference between these estimates in terms of errors in individual product estimates \mbox{$|v^{\sf UCB}_{i,\ell} - v_i|$}. Finally, we leverage the structure of the MNL model and the properties of $v^{UCB}_{i,\ell}$ to bound the regret in each epoch. Lemma \ref{lem:UCBS2} provides the precise statements of above properties. 

In the rest of this section, we make the above notions precise. Finally, in Appendix \ref{sec:completeProof}, we utilize these properties to complete the proof of Theorem \ref{main_result}.

\subsection{Upper confidence bounds }
\label{sec:UCBv}
In this section, we will show that the upper confidence bounds $v^{\sf UCB}_{i,\ell}$ converge to the true parameters $v_i$ from 
above. Specifically, we have the following result. 

\begin{lemma}\label{lem:UCBv}
For every $\ell = 1,\cdots, L$, we have:
\begin{enumerate}
\item $v^{\sf UCB}_{i,\ell}\geq v_i$ with probability at least $1-\frac{6}{{N}\ell}$ for all $i=1,\ldots,N$. 
\item There exists constants $C_1$ and $C_2$ such that \[\hspace{-25mm}\displaystyle v^{\sf UCB}_{i,\ell} - v_i \leq C_1\sqrt{\frac{v_i\log{({\sqrt{N}}\ell+1)}}{T_i(\ell)}} + C_2\frac{\log{({\sqrt{N}}\ell+1)}}{{T}_i(\ell)},\] with probability at least $1-\frac{7}{{N}\ell}$.
\end{enumerate}
\end{lemma}

\medskip \noindent  
We first establish that the estimates $\hat{v}_{i,\ell}, \; \ell \leq L$ are unbiased i.i.d estimates of the true parameter $v_i$ for all products. It is not immediately clear a priori if the estimates $\hat{v}_{i,\ell},\; \ell \leq L$ are independent.  In our setting, it is possible that the distribution of the estimate $\hat{v}_{i,\ell}$ depends on the offered assortment $S_\ell$, which in turn depends on the history and therefore, previous estimates, $\{\hat{v}_{i,\tau}, \; \tau = 1, \ldots, \ell-1 \}$. In Lemma \ref{moment_generating}, we show that the moment generating function of $\hat{v}_{i,\ell}$ conditioned on $S_\ell$ only depends on the parameter $v_i$ and not on the offered assortment $S_\ell$, there by establishing that estimates are independent and identically distributed.  Using the moment generating function, we show that $\hat{v}_{i,\ell}$ is a geometric random variable with mean $v_i$, i.e., $E(\hat{v}_{i,\ell}) = v_i$. We will use this observation and extend the classical multiplicative Chernoff-Hoeffding bounds (see  \cite{mitzenmacher} and \cite{Babaioff})  to geometric random variables. The  proof is provided in Appendix \ref{sec:UCBva}

\subsection{Optimistic estimate and convergence rates}
\label{sec:UCBS}
In this section, we show that the estimated revenue converges to the optimal expected revenue from 
above. First, we show that the estimated revenue is an upper confidence bound on the optimal revenue. In particular, we have the following result. 
\begin{lemma}\label{lem:UCBS1}
Suppose $S^* \in {\cal S}$ is the assortment with highest expected revenue, and Algorithm \ref{learn_algo} offers $S_\ell \in {\cal S}$ in each epoch $\ell$. Then, for every epoch $\ell $, we have 
\[\tilde{R}_\ell(S_\ell) \geq \tilde{R}_\ell(S^*) \geq R(S^*, \mb{v}) \; \text{with probability at least}\; \;1-\frac{6}{\ell}.\]
\end{lemma}

\medskip \noindent 
In Lemma \ref{UCB_bound}, we show that the optimal expected revenue is monotone in the MNL parameters.   It is important to note that we do not claim that the expected revenue is in general a monotone function, but only the value of the expected revenue corresponding to the optimal assortment increases with increase in the MNL parameters. The result follows from Lemma \ref{lem:UCBv}, where we established that $v^{\sf UCB}_{i,\ell}> v_{i}$ with high probability. We provide the detailed proof in Appendix \ref{sec:UCBSa}.

\medskip \noindent The following result provides the convergence rates of the estimate $\tilde{R}_\ell(S_\ell)$ to the optimal expected revenue. 

\begin{lemma}\label{lem:UCBS2}
If $r_i \in [0,1]$, there exists constants $C_1$ and $C_2$ such that for every $\ell =1,\cdots,L$, we have
\[\textstyle (1+\sum_{j \in S_\ell} v_{j})(\tilde{R}_{\ell}(S_\ell) - R(S_\ell, \mb{v})) \leq C_1\sqrt{\frac{v_i\log{({\sqrt{N}}\ell+1)}}{|\mathcal{T}_i(\ell)|}} + C_2\frac{\log{({\sqrt{N}}\ell+1)}}{|\mathcal{T}_i(\ell)|},\; \]$\text{with probability at least}\;\;1-\frac{13}{\ell}.$
\end{lemma}
\medskip \noindent 
In Lemma \ref{lipschitz}, we show that the expected revenue function satisfies a certain kind of Lipschitz condition. Specifically, the difference between the estimated, $\tilde{R}_{\ell}(S_\ell)$, and expected revenues,  $R_\ell(S_\ell)$, is bounded by the sum of errors in parameter estimates for the products, $|v^{\sf UCB}_{i,\ell} - v_i|$. This observation in conjunction with the ``optimistic estimates'' property  will let us bound the regret as an aggregated difference between estimated revenues and expected revenues of the offered assortments. Noting that we have already computed convergence rates between the parameter estimates earlier, we can extend them to show that the estimated revenues converge to the optimal revenue from above. We complete the proof in Appendix \ref{sec:UCBSa}.

\section{Lower bounds and near-optimality of the proposed policy}\label{sec:lowerBound}
In this section, we consider the special case of TU constraints, namely, a cardinality constrained assortment optimization problem, and establish that any policy must incur a regret of $\Omega(\sqrt{NT/K})$.  More precisely, we prove the following result. 

\begin{theorem1}[Lower bound on achievable performance]\label{lower_bound}
There exists a (randomized) instance of the \banditMNL~problem with $v_0 \geq v_i\;, i=1,\ldots,N$, such that {for any $N$ and $K$},  and any policy $\mathcal{\pi}$ that offers assortment $\mathcal{S}_t^{\mathcal{\pi}}$, $|{S}_t^{\mathcal{\pi}}| \leq K$ at time $t$, we have 
for all $T\geq N$ that,$$Reg_\pi(T,\mb{v}):= \mathbbm{E}_\pi\left( \sum_{t=1}^T R(S^*, \mb{v}) - R({S}_t^{\mathcal{\pi}}, \mb{v})\right) \geq C\sqrt{\frac{NT}{K}},$$
where $S^*$ is (at-most) $K$-cardinality assortment with maximum expected revenue, and $C$ is an absolute constant. 
\end{theorem1}
\begin{remark}{\normalfont{\textbf{(Optimality)}} \normalfont { \change{ Theorem \ref{lower_bound} establishes that Algorithm \ref{learn_algo} is optimal if we assume $K$ to be fixed. We note that the assumption that $K$ is fixed holds in many realistic settings, in particular, in online retailing, where there are a large number of products, but only fixed number of slots to show these products. Algorithm \ref{learn_algo} is nearly optimal if $K$ is also considered to be a problem parameter, with the upper bound being within a factor of $\sqrt{K}$ of the lower bound. In recent work, \cite{xichen} established a lower bound of $\Omega\left({\sqrt{NT}}\right)$ for the \banditMNL\;problem, when $K < N/4$, thus suggesting that Algorithm \ref{learn_algo} is optimal even with respect to its dependence on $K$. For the special case of the unconstrained \banditMNL\;problem (i.e., $K=N$), the regret bound of Algorithm \ref{learn_algo} can be improved to $O(\sqrt{|S^*|T})$, where $|S^*|$ is the size of the optimal assortment (see Appendix \ref{sec:unconstrained_mnl_bandit}) and the optimality gap for the unconstrained setting is $\sqrt{|S^*|}$.}   }}\end{remark}

\subsection{Proof overview}
{For ease of exposition, we focus here on the case where $K < N$, and present the proof for lower bound when $K=N$ in Appendix \ref{lower_boundN}. To that end, we will assume that $K < N$ for the rest of this section.} We prove Theorem \ref{lower_bound} by a reduction to a parametric multi-armed bandit (MAB) problem, for which a lower bound is known.  
\begin{definition}[MAB instance $I_{\sf MAB}$] Define $I_{\sf MAB}$ as a (randomized) instance of MAB problem with $N\ge 2$ Bernoulli arms {(reward is either $0$ or $1$) and the probability of the reward being $1$ for arm $i$ is given by,}
$$
\mu_i= \left\{\begin{array}{ll}
\alpha,  & \text{if } i\neq j,\\
\alpha + \epsilon, & \text{if } i=j,
\end{array} \right.  \text{ \ \ for all } i = 1,\ldots, N,
$$
where $j$ is set uniformly at random from $\{1,\ldots, N\}$, $\alpha <1$ and $\epsilon = \frac{1}{100}\sqrt{\frac{N\alpha}{T}}$. 
\end{definition}

{\color{black} Throughout this section we will use the terms algorithm and policy interchangeably.  An algorithm $\ep{A}$ is referred to as online if it adaptively selects a history dependent $\ep{A}_t \in \{1,\ldots,n\}$ at each time t as in \eqref{eq:selection} for the MAB problem. }

\begin{lemma}\label{lower_bound_MAB}
For any $N \geq 2$, $\alpha<1$, $T$ and any online algorithm  $\mathcal{A}$ that plays arm $\mathcal{A}_t$ at time $t$, the expected regret on instance $I_{\sf MAB}$ is at least $\displaystyle \frac{\epsilon T}{6}$. That is,
		$$Reg_{\ep{A}}(T,\pmb{\mu}):= \mathbb{E}\left[\sum_{t =1}^T (\mu_j - \mu_{\mathcal{A}_t})\right] \geq \displaystyle \frac{\epsilon T}{6},$$
where, the expectation is both over the randomization in generating the instance (value of $j$), as well as the random outcomes that result from pulled arms.
\end{lemma}
The proof of Lemma \ref{lower_bound_MAB} is a simple extension of the proof of the $\Omega(\sqrt{NT})$ lower bound for the Bernoulli instance with parameters $\frac{1}{2}$ and $\frac{1}{2} + \epsilon$ (for example, see \citealp{MAL-024}), and has been provided in Appendix \ref{lower_bound_parametric_mab} for the sake of completeness. We use the above lower bound for the MAB problem to prove that any algorithm must incur at least $\Omega(\sqrt{NT/K})$ regret on the following instance of the \banditMNL~problem.

\begin{definition}[{{\banditMNL}}~instance $I_{\sf MNL}$] Define $I_{\sf MNL}$ as the following (randomized) instance of \banditMNL~problem with $K$-cardinality constraint, $\hat{N} = NK$ products, 
 parameters $v_0 = K$ and for $i=1,\ldots, \hat{N}$,
$$
v_i = \left\{ 
\begin{array}{ll}
\alpha, & \text{ if } {\lceil}{\frac{i}{K}}{\rceil} \neq j, \\ 
\alpha + \epsilon, & \text{ if } {\lceil}{\frac{i}{K}}{\rceil} =j,
\end{array}\right.
$$
where $j$ is set uniformly at random from $\{1,\ldots,N\}$, $\alpha<1$, and $\epsilon=\frac{1}{100}\sqrt{\frac{N\alpha}{T}}$ and $r_i = 1.$
\end{definition}

\noindent We will show that any \banditMNL~algorithm has to incur a regret of $\Omega\left(\sqrt{\frac{NT}{K}}\right)$ on instance $I_{\sf MNL}$. The main idea in our reduction is to show that if there exists an algorithm $\ep{A}_{\sf MNL}$ for \banditMNL~that achieves $o(\sqrt{\frac{NT}{K}})$ regret on instance $I_{\sf MNL}$, then we can use $\ep{A}_{\sf MNL}$ as a subroutine to construct an algorithm $\ep{A}_{\sf MAB}$ for the MAB problem that achieves strictly less than $\frac{\epsilon T}{6}$ regret on instance $I_{\sf MAB}$ in time $T$, thus contradicting the lower bound of Lemma \ref{lower_bound_MAB}. This will prove Theorem \ref{lower_bound} by contradiction.

\subsection{Construction of the MAB algorithm using the MNL algorithm}

\begin{algorithm}
\begin{algorithmic}[1]
\State \textbf{Initialization:} $t=0$, $\ell = 0$
\While{$t \leq T$}

	\State Update $\ell=\ell+1$
		
	\State {\bf Call} \mbox{$\ep{A}_{\sf MNL}$}, receive assortment $S_\ell \subset [\hat{N}]$
		
	\State{\bf Repeat until `exit loop'} 
	
\State 	\hspace{0.5cm} With probability $\frac{1}{2}$, send {\bf Feedback to} $\ep{A}_{\sf MNL}$ `no product was purchased', {\bf exit loop}
			\State  \hspace{0.5cm}  Update $t=t+1$
			
	 \State  		{\hspace{0.5cm} With probability $\frac{1}{2K}$, {\bf pull} arm ${\cal A}_t=\lceil{\frac{i}{K}\rceil}$, where $i\in S_{\ell}$} 			
	 \State  		{\hspace{0.5cm} With probability $\frac{1}{2} - \frac{|S_\ell|}{2K}$, {\bf continue the loop} (go to Step-5)}

			\State \hspace{0.5cm} If reward is $1$, send {\bf Feedback to} $\ep{A}_{\sf MNL}$ `$i$ was purchased' and {\bf exit loop}
						
  \State {\bf end loop}	
\EndWhile
\end{algorithmic}
\caption{Algorithm $\ep{A}_{\sf MAB}$}
\label{mab_algo}
\end{algorithm}
\newcommand{\amnl}{$\ep{A}_{\sf MNL}$~}
\newcommand{\amab}{$\ep{A}_{\sf MAB}$~}

Algorithm \ref{mab_algo} provides the exact construction of $\ep{A}_{\sf MAB}$, which simulates the $\ep{A}_{\sf MNL}$ algorithm as a ``black-box.'' Note that \amab pulls arms at time steps $t=1,\ldots, T$. These arm pulls are interleaved by simulations of \amnl steps ({\bf Call} \amnl, {\bf Feedback to} \amnl). When step $\ell$ of \amnl is simulated, it uses the feedback from $1,\ldots, \ell-1$ to suggest an assortment $S_{\ell}$; and recalls a feedback from \amab on which product (or no product) was purchased out of those offered in $S_\ell$, where the probability of purchase of product $i\in S_{\ell}$ is ${v_i}{\big /}{(v_0+\sum_{i\in S_\ell}v_i})$. Before showing that the \amab indeed provides the right feedback to \amnl in the $\ell^{th}$ step for each $\ell$, we introduce some notation.

Let $M_\ell$ denote the length of the loop at step $\ell$, more specifically, the cumulative number of times, \amnl was executing steps 6, 8 or 9 in the $\ell^{th}$ step before exiting the loop. For every $i \in S_\ell \cup 0$, let ${\zeta^i_\ell}$ denote the event that the feedback to \amnl from \amab after step $\ell$ of \amnl is ``product $i$ is purchased''. We have, 

$$ \mathcal{P}(M_\ell=m \; \cap\; \zeta^i_\ell) = \frac{v_i}{2K}\left( \frac{1}{2K}\sum_{i \in S_\ell} (1-v_i) {+ \frac{1}{2} - \frac{|S_\ell|}{2K}}\right)^{m-1}\;\text{for each $i\in S_\ell \cup \{0\}$}. $$
Hence, the probability that \amab's feedback to \amnl is  ``product $i$ is purchased'' is,  
\begin{equation*}
\begin{aligned}
{p}_{S_\ell}(i) &= \sum_{m=1}^{\infty} \mathcal{P}(M_\ell=m \; \cap\; \zeta^i_\ell)= \displaystyle \frac{v_i}{v_0+\sum_{q \in S_\ell} v_q}.
\end{aligned}
\end{equation*}
This establish that \amab provides the appropriate feedback to \amnl. 

\subsection{ Proof of Theorem \ref{lower_bound}}
We prove the result by establishing three key results. First, we upper bound the regret for the MAB algorithm, \amab. Then, we prove a lower bound on the regret for the MNL algorithm, \amnl. Finally, we relate the regret of \amab and \amnl and use the established lower and upper bounds to show a contradiction. 

\noindent For the rest of this proof, assume that $L$ is the total number of calls to \amnl in \amab. Let $S^*$ be the optimal assortment for $I_{\sf MNL}$. For any instantiation of $I_{\sf MNL}$, it is easy to see that the optimal assortment contains $K$ items, all with parameter $\alpha+\epsilon$, i.e., it contains all $i$ such that $\lceil \frac{i}{K}\rceil=j$. Therefore, $V(S^*)= K(\alpha+\epsilon) = K\mu_j$. Note that if an algorithm offers an assortment, $S_\ell$, such that $|S_\ell| < K$, then we can improve the regret incurred by this algorithm for the \banditMNL\;instance $I_{\sf MNL}$ by offering an assortment $\hat{S}_\ell = S_\ell \cup \{i\}$ for some $i \not \in S_\ell.$  Since our focus is on lower bounding the regret, we will assume, without loss of generality, that $|S_\ell| = K$ for the rest of this section.

\medskip \noindent {\bf Upper bound for the regret of the MAB algorithm.} 
The first step in our analysis is to prove an upper bound on the regret of the MAB algorithm, \amab on the instance $I_{\sf MAB}$. Let us label the loop following the $\ell$th call to \amnl in Algorithm \ref{mab_algo} as $\ell$th loop. Note that the probability of exiting the loop is \mbox{$p =E[\frac{1}{2}+\frac{1}{2}\mu_{{\cal A}_t}]$} $= \frac{1}{2}+\frac{1}{2K} V(S_\ell),$ where $V(S_\ell) \overset{\Delta}{=} \sum_{i \in S_\ell} v_i$. In every step of the loop until exited, an arm is pulled with probability $1/2$. The optimal strategy would pull the best arm so that the total expected optimal reward in the loop is \mbox{$\sum_{r=1}^\infty (1-p)^{r-1} \frac{1}{2} \mu_j = \frac{\mu_j}{2p} = \frac{1}{2Kp}V(S^*)$}. Algorithm \amab pulls a random arm from $S_{\ell}$, so total expected algorithm's reward in the loop is $\sum_{r=1}^\infty (1-p)^{r-1} \frac{1}{2K} V(S_\ell) = \frac{1}{2Kp} V(S_\ell)$. Subtracting the algorithm's reward from the optimal reward, and substituting $p$, we obtain that the total expected regret of \amab over the arm pulls in loop $\ell$ is   
$$ \frac{V(S^*) - V(S_\ell)}{(K+ V(S_\ell))}.$$ Noting that $V(S_\ell) \geq K\alpha$, we have the following upper bound on the regret for the MAB algorithm.
\begin{equation}\label{eq:MAB_ub} Reg_{\text{\amab}}(T, \pmb{\mu}) \le \frac{1}{(1+{\alpha})} \mathbb{E}\left(\sum_{\ell=1}^L \frac{1}{K} (V(S^*) - V(S_\ell))\right),\end{equation} where the expectation in equation \eqref{eq:MAB_ub} is over the random variables $L$ and $S_\ell$.

\medskip \noindent {\bf Lower bound for the regret of the MNL algorithm.} 
Here, we derive a lower bound on the  regret of the MNL algorithm, \amnl on the instance $I_{\sf MNL}$.  Specifically,
\begin{eqnarray*} 
 Reg_{\text{\amnl}}(L, \mb{v}) & = & \mathbb{E}\left[\sum_{\ell=1}^L \frac{V(S^*)}{v_0+V(S^*)} - \frac{V(S_\ell)}{v_0+V(S_\ell)}\right]\\
& \ge & \frac{1}{K(1+\alpha)} \mathbb{E}\left[\sum_{\ell=1}^L \left(\frac{V(S^*)}{1+\frac{\epsilon}{1+\alpha}} - V(S_\ell)\right)\right].
\end{eqnarray*} 
Therefore, it follows that, 
\begin{equation}\label{eq:MNL_lb}
  Reg_{\text{\amnl}}(L, \mb{v}) \ge  \frac{1}{(1+\alpha)} \mathbb{E}\left[\sum_{\ell=1}^L \frac{1}{K} (V(S^*) - V(S_\ell)) - \frac{\epsilon v^* L}{(1+\alpha)^2}\right],
  \end{equation}
where {$v^* = \alpha + \epsilon$} and the expectation in equation \eqref{eq:MNL_lb} is over the random variables $L$ and $S_\ell$.  

\medskip \noindent {\bf Relating the regret of the MNL algorithm and the MAB algorithm.} Finally, we relate the regret of the MNL algorithm \amnl and MAB algorithm \amab to derive a contradiction.  The first step in relating the regret involves relating the length of the horizons of \amnl and $\ep{A}_{\sf MAB},$ $L$ and $T$ respectively. {Note that, after every call}  to \amnl (``{\bf Call} $\ep{A}_{\sf MNL}$" in  Algorithm \ref{mab_algo}), many iterations of the following loop may be executed; in roughly $1/2$ of those iterations, an arm is pulled and $t$ is advanced (with probability $1/2$, the loop is exited without advancing $t$). Therefore, $T$ should be at least a constant fraction of $L$. Lemma \ref{lemmaTL} in Appendix \ref{lower_bound_parametric_mab} makes this precise by showing that $\mathbb{E}(L) \leq 3T$. 

Now we are ready to prove Theorem \ref{lower_bound}. From \eqref{eq:MAB_ub} and \eqref{eq:MNL_lb}, we have 
\begin{equation*} Reg_{\text{\amab}}(T, \pmb{\mu}) \le \mathbb{E}\left(Reg_{\text{\amnl}}(L, \mb{v}) +  \frac{\epsilon v^* L}{(1+\alpha)^2}\right).\end{equation*}  
For the sake of contradiction, suppose that the regret of the \text{\amnl}, $ Reg_{\text{\amnl}}(L, \mb{v}) \le c\sqrt{\frac{\hat{N}L}{K}}$ for a constant $c$ to be prescribed below. Then, from Jensen's inequality, it follows that, 
\begin{eqnarray*}
Reg_{\text{\amab}}(T, \pmb{\mu}) & \le &  c\sqrt{\frac{\hat{N}\mathbb{E}(L)}{K}} +  \frac{\epsilon v^* \mathbb{E}(L)}{(1+\alpha)^2}.
\end{eqnarray*}
From lemma \ref{lemmaTL}, we have that $\mathbb{E}(L) \leq 3T$. Therefore, we have, $c\sqrt{\frac{\hat{N}\mathbb{E}(L)}{K}} = c\sqrt{N\mathbb{E}(L)} \le  c\sqrt{3NT} = c\epsilon T \sqrt{\frac{3}{\alpha}} < \frac{\epsilon T}{12}$ on setting $c< \frac{1}{12}\sqrt{\frac{\alpha}{3}} $. Also, using $v^*=\alpha+\epsilon \le 2\alpha$, and setting $\alpha$ to be a small enough constant, we can get that the second term above is also strictly less than $\frac{\epsilon T}{12}$. Combining these observations, we have \vspace{-0.3cm}
$$ \textstyle Reg_{\text{\amab}}(T, \pmb{\mu})  <  \frac{\epsilon T}{12} + \frac{\epsilon T}{12} = \frac{\epsilon T}{6},  \vspace{-0.3cm}	$$
thus arriving at a contradiction. \hfill $\halmos$

\section{Extensions}\label{sec:extn}
\change{In this section, we consider two extensions of the \banditMNL\;problem. In the first extension, we consider problem instances that are ``well separated'' and present an improved logarithmic regret bound. We will then consider a setting where the ``no purchase'' assumption ($v_i \leq v_0$ for all $i$) is relaxed and present a modified algorithm that works for more general class of MNL parameters and establish $\tilde{O}(\sqrt{BNT})$ regret bounds.}
\subsection{{Improved regret bounds for ``well-separated'' instances}}\label{sec:parameter_dependent_bound}
In this section, we derive an $O(\log{T})$ regret bound for Algorithm \ref{learn_algo} for instances that are ``well separated.''  In Section \ref{sec:regretAnalysis}, we established worst case regret bounds for Algorithm \ref{learn_algo} that hold for all problem instances satisfying Assumption \ref{assumption1}. Although our algorithm ensures that the exploration-exploitation tradeoff is balanced at all times, for problem instances that are ``well separated,'' our algorithm quickly converges to the optimal solution leading to better regret bounds.  More specifically, we consider problem instances where the optimal assortment and ``second best'' assortment are sufficiently ``separated'' and derive a $O(\log{T})$ regret bound that depends on the parameters of the instance. Note that, unlike the regret bound derived in Section \ref{sec:regretAnalysis} that holds for all problem instances satisfying Assumption \ref{assumption1}, the bound we derive here only holds for instances having certain separation between the revenues corresponding to optimal and second best assortments. In particular, let $\Delta(\mb{v})$ denote the difference between the expected revenues of the optimal and second-best assortment, i.e., 
\begin{equation}\label{delta_v}
\Delta (\mathbf{v}) \; = \underset{\{S \in \mathcal{S} | R(S,\mb{v}) \neq R(S^*,\mb{v})\}}{\min} \{R(S^*,\mb{v})-R(S)\}.
\end{equation} 
We have the following result.
\begin{theorem1}[Performance Bounds for Algorithm \ref{learn_algo} in ``well separated'' case]\label{assymptotic_bound}
{For any instance $\mb{v} = (v_0, \ldots, v_N)$ of the \banditMNL~problem with $N$ products,  $r_i\in [0,1]$ and Assumption \ref{assumption1}, the regret of the policy given by Algorithm~\ref{learn_algo} at any time $T$ is bounded as,
\begin{equation*} 
Reg_\pi(T,\mb{v}) \leq B_1\left(\frac{N^2\log{T}}{\Delta(\mb{v})}\right) + B_2,
\end{equation*}
where $B_1$ and $B_2$ are absolute constants.}
\end{theorem1}
\medskip \noindent {\bf Proof outline.}
In this setting, we analyze the regret by separately considering the epochs that satisfy certain desirable properties and the ones that do not. Specifically, we denote epoch $\ell$ as a ``good'' epoch if the parameters $v^{\sf UCB}_{i,\ell}$ satisfy the following property, 
$$0\leq v^{\sf UCB}_{i,\ell} -v_i \leq   C_1 \sqrt{\frac{v_i\log{({\sqrt{N}}\ell+1)}}{T_i(\ell)}} + C_2\frac{\log{({\sqrt{N}}\ell+1)}}{T_i(\ell)}, $$ and we call it a ``bad'' epoch otherwise, where $C_1$ and $C_2$ are constants as defined in Lemma \ref{lem:UCBv}. Note that every epoch $\ell$ is a good epoch with high probability $(1-\frac{13}{\ell})$ and we show that regret due to ``bad'' epochs is bounded by a constant (see Appendix \ref{parametric_bounds}).  Therefore, we focus on ``good'' epochs and show that there exists a constant $\tau$, such that after each product has been offered in at least $\tau$ ``good'' epochs, Algorithm \ref{learn_algo} finds the optimal assortment. Based on this result, we can then bound the total number of ``good'' epochs in which a sub-optimal assortment can be offered by our algorithm. Specifically, let 
\begin{equation}\label{eq:tau}
\tau = \frac{4{N}C\log{{N}T}}{\Delta^2(\mb{v})},
\end{equation} 
where $C = \max\{{C^2_1},C_2\}$. Then we have the following result. 

\begin{lemma}\label{sensitivity_analysis}
Let $\ell$ be a ``good'' epoch and $S_\ell$ be the assortment offered by Algorithm \ref{learn_algo} in epoch $\ell$. If every product in assortment $S_\ell$ is offered in at least $\tau$ ``good epochs,'' i.e. $T_i(\ell) \geq \tau \;\; \text{for all} \;\; i,$  then we have $R(S_\ell, \mb{v}) = R(S^*, \mb{v})$ .
\end{lemma}

We prove Lemma \ref{sensitivity_analysis} in Appendix \ref{parametric_bounds}. The next step in the analysis is to show that Algorithm \ref{learn_algo} will offer a small number of sub-optimal assortments in ``good'' epochs. We make this precise in the following observation whose proof amounts to a simple counting exercise using Lemma \ref{sensitivity_analysis} (see full proof in Appendix \ref{parametric_bounds}.) 

\begin{lemma}\label{no_sub_opt}
{Algorithm \ref{learn_algo} cannot offer sub-optimal assortments in more than ${N}\tau$ ``good'' epochs.}
\end{lemma}
The proof for Theorem \ref{assymptotic_bound} follows from the above result. In particular, noting that the number of epochs in which sub-optimal assortment is offered is small, we re-use the regret analysis of Section \ref{sec:regretAnalysis} to bound the regret by {$O(N^2\log{T})$}. We provide a rigorous proof in Appendix \ref{parametric_bounds} for the sake of completeness. {Note that for the special case of cardinality constraints, we have $|S_\ell| \leq K$ for every epoch $\ell$. By modifying the definition of $\tau$ in \eqref{eq:tau} to $\tau = {4{K}C\log{{N}T}}/{\Delta^2(\mb{v})}$ and following the above analysis, we can improve the regret bound to $O(NK\log{T})$ for this case. Specifically, we have the following.
\begin{corollary}[Performance bounds in well separated case under cardinality constraints]\label{parameter_bound_cardinality}
For any instance $\mb{v} = (v_0, \ldots, v_N)$ of the \banditMNL~problem with $N$ products and cardinality constraint $K$,  $r_i\in [0,1]$ and $v_0 \geq v_i$ for all $i,$ the regret of the policy given by Algorithm~\ref{learn_algo} at any time $T$ is bounded as,
\begin{equation*} 
Reg_\pi(T,\mb{v}) \leq B_1\frac{NK\log{{N}T}}{\Delta(\mb{v})} + B_2,
\end{equation*}
where, $B_1$ and $B_2$ are absolute constants and $\Delta(\mb{v})$ is defined in \eqref{delta_v}.
\end{corollary}}
{It should be noted that the bound obtained in Corollary \ref{parameter_bound_cardinality} is similar in magnitude to the regret bounds obtained by \cite{Saure}, when $K$ is assumed to be fixed, and is strictly better than the regret bound $O(N^2\log^2{T})$ established by \cite{rusme}. Moreover, our algorithm does not require the knowledge of $\Delta(\mb{v}),$ unlike the aforementioned papers which build on a conservative estimate of $\Delta(\mb{v})$ to implement their proposed policies.}

\subsection{Relaxing the ``no purchase'' assumption}\label{sec:modified_policy}
In this section, we extend our approach (Algorithm \ref{learn_algo}) to the setting where the assumption that $v_i \leq v_0$ for all $i$ is relaxed. The essential ideas in the extension remain the same as our earlier approach, specifically optimism under uncertainty, and our policy is structurally similar to Algorithm \ref{learn_algo}. The modified policy requires a small but mandatory initial exploration period. However, unlike the works of \cite{rusme} and \cite{Saure}, the exploratory period does not depend on the specific instance parameters and is constant for all problem instances. Therefore, our algorithm is parameter independent and remains relevant for practical applications. Moreover, our approach continues to simultaneously explore and exploit after the initial exploratory phase.  In particular, the initial exploratory phase is to ensure that the estimates converge to the true parameters from above particularly in cases when the attraction parameter $v_i$ (frequency of purchase), is large for certain products. We describe our approach in Algorithm \ref{learn_algo_extn}.
\begin{algorithm}
\begin{algorithmic}[1]
\State \textbf{Initialization:} $v^{\sf UCB}_{i,0} = 1$ for all $i=1,\ldots,N$

\State $t = 1$ ; $\ell = 1$ keeps track of the time steps and total number of epochs respectively

\State $T_i(1) = 0$ for all $i = 1, \ldots, N$

\While{$t < T$}

	\State Compute $S_{\ell} = \underset{S \in \mathcal{S}}{\text{argmax}}\;\tilde{R}_{\ell}(S) = \frac{\displaystyle\sum_{i \in S}r_i v^{\sf UCB}_{i,\ell-1}}{1+\displaystyle\sum_{j \in S} v^{\sf UCB}_{j,\ell-1}}$
           \If {$T_i(\ell) < {48} \log{({\sqrt{N}}\ell+1)}$ for some $i \in S_\ell$} 
		\State Consider $\hat{S}$ =$\{ i | T_i(\ell) < 48 \log{({\sqrt{N}}\ell+1)}\}$
                     	\State  Choose $S_\ell \in \mathcal{S}$ such that $ S_\ell \subset \hat{S}$
           \EndIf
	\State Offer assortment $S_{\ell}$, observe the purchasing decision, ${c_t}$ of the consumer

	\If{$c_t = 0$}
		\State compute $\hat{v}_{i,\ell} = \sum_{t \in \ep{E}_\ell} \mathbbm{1}(c_t = i)$, no. of consumers who preferred $i$ in epoch $\ell$, for all $i \in S_\ell$
		\State update $\mathcal{T}_i(\ell) = \left\{\tau \leq \ell \, \middle| \, i \in S_\ell\right\}, T_i(\ell)=|\mathcal{T}_i(\ell)|$, no. of epochs until $\ell$ that offered  \mbox{product $i$}
		\State update $\bar{v}_{i,\ell}$ = $\displaystyle \frac{1}{T_i(\ell)}\sum_{\tau \in \ep{T}_i(\ell)} \hat{v}_{i,\tau}$, sample mean of the estimates
		\State	update $ v^{\sf UCB2}_{i,\ell}$  =$\bar{v}_{i,\ell} + \max{\left\{\sqrt{\bar{v}_{i,\ell}}, \bar{v}_{i,\ell}\right\}}\sqrt{\frac{48\log{({\sqrt{N}}\ell+1)}}{T_i(\ell)}} + \frac{48 \log{({\sqrt{N}}\ell+1)}}{T_i(\ell)}$	
		\State $\ell = \ell + 1$

	\Else
		\State $\ep{E}_\ell = \ep{E}_\ell \cup t$, time indices corresponding to epoch $\ell$

	\EndIf
\State {$t = t+1$}
\EndWhile
\end{algorithmic}
\caption{Exploration-Exploitation algorithm for \banditMNL~general parameters}
\label{learn_algo_extn}
\end{algorithm}

We can extend the analysis in Section \ref{sec:regretAnalysis} to bound the regret of Algorithm \ref{learn_algo_extn} as follows. 
\begin{theorem1}[Performance Bounds for Algorithm \ref{learn_algo_extn}]\label{main_result_extn}
For any instance $\mb{v} = (v_0, \ldots, v_N)$, of the \banditMNL~problem with $N$ products,  $r_i\in [0,1]$ for all $i=1,\ldots, N$,  the regret of the policy corresponding to Algorithm~\ref{learn_algo_extn} at any time $T$ is bounded as,
\begin{equation*} 
Reg_{\pi}(T,\mb{v}) \leq  C_1\sqrt{BNT\log{{N}T}} + C_2N\log^2{{N}T}+C_3 N B \log{{N}T},
\end{equation*}
where $C_1$, $C_2$ and $C_3$ are absolute constants and $B = \max\{\max_{i} \frac{v_i}{v_0},1\}$. 
\end{theorem1}

\medskip \noindent {\bf Proof outline.}
Note that Algorithm \ref{learn_algo_extn} is very similar to Algorithm \ref{learn_algo} except for the initial exploratory phase. Hence, to bound the regret we first prove that the initial exploratory phase is indeed bounded and then follow the approach discussed in Section \ref{sec:regretAnalysis} to establish the correctness of the confidence intervals, the optimistic assortment, and finally deriving the convergence rates and regret bounds. We make the above notions precise and provide the complete proof in Appendix \ref{sec:proof_main_result_extn}.

\section{Computational study}\label{sec:computation}
In this section, we present insights from numerical experiments that test the empirical performance of our policy and highlight some of its salient features. We study the performance of Algorithm \ref{learn_algo} from the perspective of robustness with respect to the ``separability parameter'' of the underlying instance. In particular, we consider varying levels of separation between the revenues corresponding to the optimal assortment and the second best assortment and perform a regret analysis numerically. We contrast the performance of Algorithm \ref{learn_algo} with the approach in \cite{Saure} for different levels of separation. We observe that when the separation between the revenues corresponding to optimal assortment and second best assortment is sufficiently small, the approach in \cite{Saure} breaks down, i.e., incurs linear regret, while the regret of Algorithm \ref{learn_algo} only grows sub-linearly with respect to the selling horizon. We also present results from a simulated study on a real world data set. 

\subsection{Robustness of Algorithm \ref{learn_algo}}
Here, we present a study that examines the robustness of Algorithm \ref{learn_algo} with respect to the instance separability. We consider a parametric instance (see \eqref{eq:problem_instance}), where the separation between the revenues of the optimal assortment and next best assortment is specified by the parameter $\epsilon$ and compare the performance of Algorithm \ref{learn_algo} for different values of $\epsilon$. 

\medskip \noindent {\bf Experimental setup.} We consider the parametric MNL setting with $N = 10$, $K=4$, $r_i=1$ for all $i$ and utility parameters $v_0 = 1$ and for $i=1,\ldots, N$,
\begin{equation}\label{eq:problem_instance}
v_i = \left\{ 
\begin{array}{ll}
0.25+ \epsilon, & \text{ if } i \in \{1,2,9,10\} \\ 
0.25, & \text{ else },
\end{array}\right.
\end{equation}
where $0 < \epsilon < 0.25$, specifies the difference between revenues corresponding to the optimal assortment and the next best assortment. Note that this problem has a unique optimal assortment, $\{1,2,9,10\}$ with an expected revenue of ${1+4\epsilon}/{2+4\epsilon}$ and next best assortment has revenue of ${1+3\epsilon}/{2+3\epsilon}$. We consider four different values for $\epsilon$, $\epsilon = \{0.05, 0.1, 0.15, 0.25\}$, where higher value of $\epsilon$ corresponds to larger separation, and hence an ``easier'' problem instance.

\begin{figure}
\begin{center}
{\includegraphics[scale=0.3]{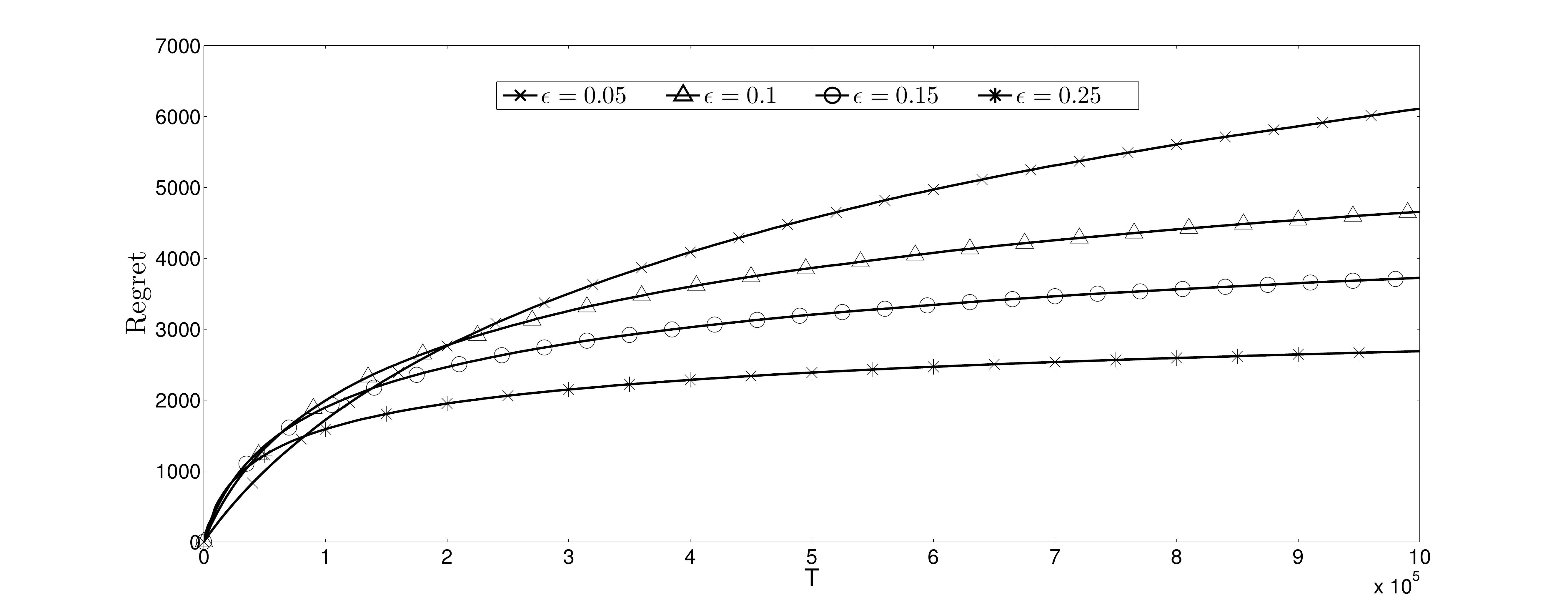}}
\caption{Performance of Algorithm \ref{learn_algo} measured as the regret  on the parametric instance \eqref{eq:problem_instance}. The graphs  illustrate the dependence of the regret on $T$ for ``separation gaps'' $\epsilon = 0.05, 0.1, 0.15 \; \text{and} \; 0.25$ respectively.  }\label{fig:regret_robust}
\end{center}
\end{figure}

\medskip \noindent {\bf Results.} Figure \ref{fig:regret_robust} summarizes the performance of Algorithm \ref{learn_algo} for different values of $\epsilon$. The results are based on running 100 independent simulations, the standard errors are within 2\%. Note that the performance of Algorithm \ref{learn_algo} is consistent across different values of $\epsilon$; with a regret that exhibits sub linear growth. Observe that as the value of $\epsilon$ increases the regret of Algorithm \ref{learn_algo} decreases. While not immediately obvious from Figure \ref{fig:regret_robust}, the regret behavior is fundamentally different in the case of ``small'' $\epsilon$ and ``large'' $\epsilon$. To see this, in Figure \ref{fig:regret_regress} we focus on the regret for $\epsilon = 0.05$ and $\epsilon = 0.25$ and fit to $\log{T}$ and $\sqrt{T}$ respectively. (The parameters of these functions are obtained via simple linear regression of the regret vs $\log{T}$ and $\sqrt{T}$ respectively). It can be observed that the regret is roughly logarithmic when $\epsilon = 0.25$, and in contrast roughly behaves like $\sqrt{T}$ when $\epsilon = 0.05$. This illustrates the theory developed in Section \ref{sec:parameter_dependent_bound}, where we showed that the regret grows logarithmically with time, if the optimal assortment and next best assortment are ``well separated,'' while the worst-case regret scales as $\sqrt T$.  

\begin{figure}
\begin{center}
\begin{tabular}{c c}
$\overset{(a)}{\includegraphics[width=3.5in,height=2in]{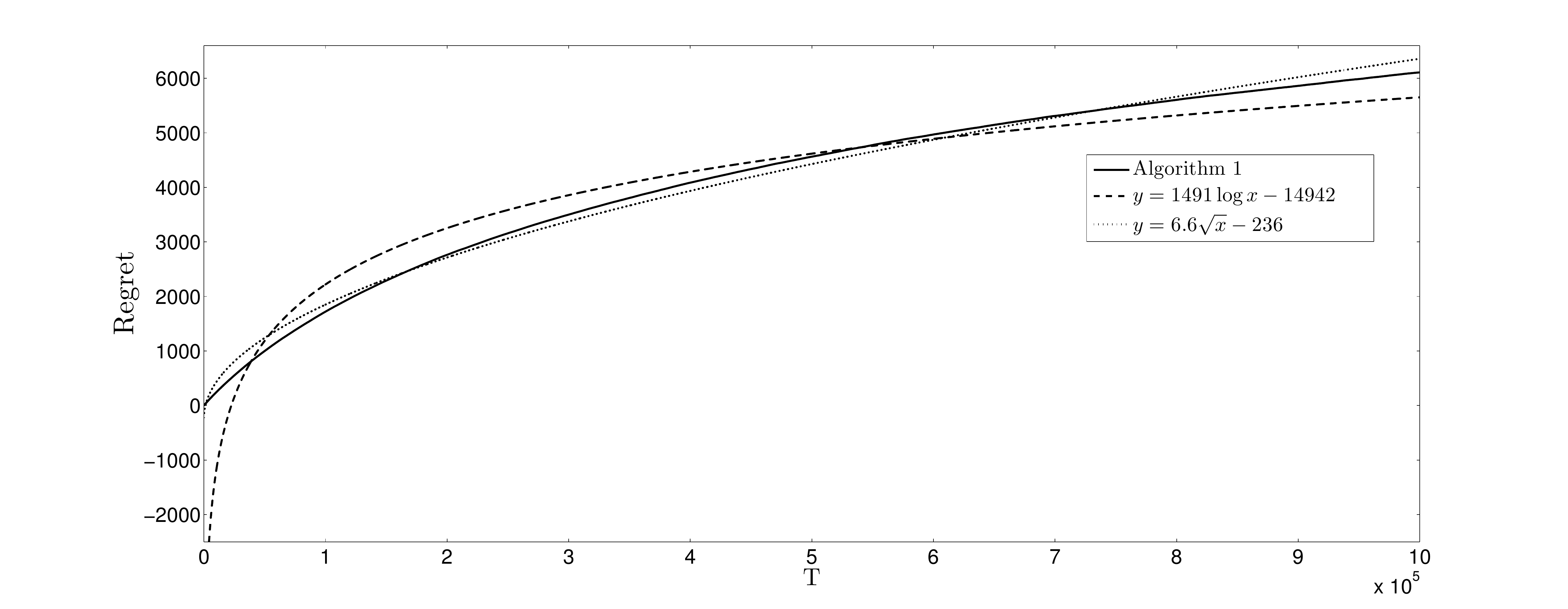}}$ & $\overset{(b)}{\includegraphics[width=3.5in,height=2in]{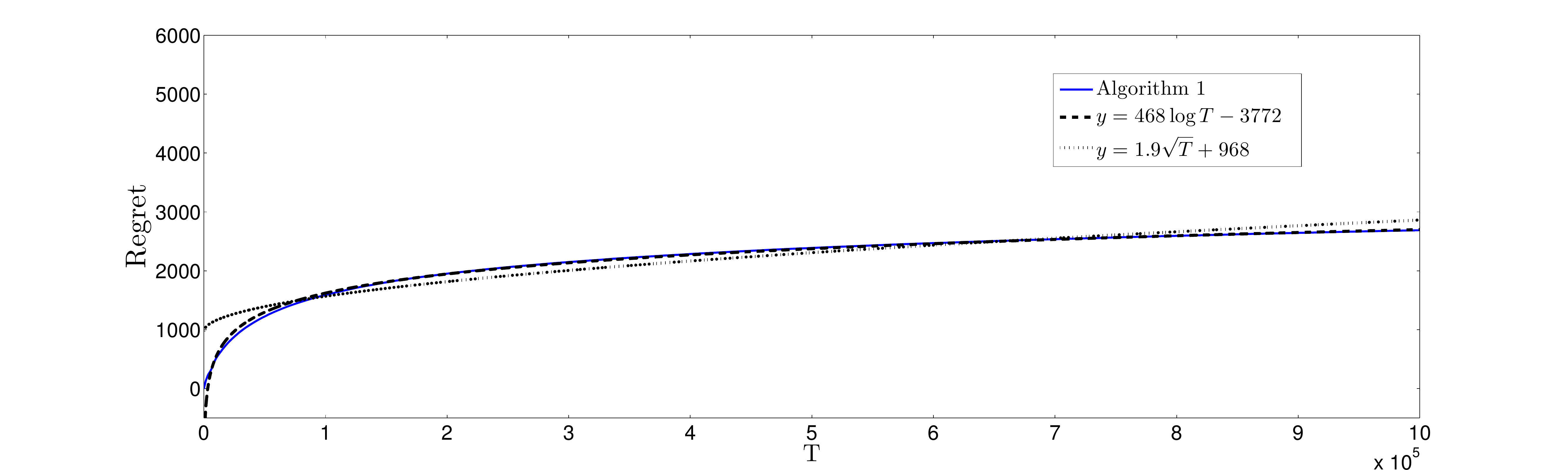}}$ 
\end{tabular}\caption{Best fit for the regret of Algorithm \ref{learn_algo} on the parametric instance \eqref{eq:problem_instance}. The graphs (a), (b) illustrate the dependence of the regret on $T$ for ``separation gaps'' $\epsilon = 0.05, \; \text{and} \; 0.25$ respectively. The best $y = \beta_1 \log{T} + \beta_0$ fit and best $y = \beta_1 \sqrt{T} + \beta_0$ fit are superimposed on the regret curve.}\label{fig:regret_regress}
\end{center}
\end{figure}

\subsection{ Comparison with existing approaches} In this section, we present a computational study comparing the performance of our algorithm to that of \cite{Saure}. (To the best of our knowledge,  \cite{Saure} is currently the best existing approach for our problem setting.) To be implemented, their approach requires certain a priori information of a ``separability parameter''; roughly speaking, measuring the degree to which the optimal and next-best assortments are distinct from a revenue standpoint. More specifically, their algorithm follows an \emph{explore-then-exploit} approach, where every product is offered for a minimum duration of time that is determined by an estimate of said ``separability parameter.'' After this mandatory exploration phase, the parameters of the choice model are estimated based on the past observations and the optimal assortment corresponding to the estimated parameters is offered for the subsequent consumers. If the optimal assortment and the next best assortment are ``well separated,''  then the offered assortment is optimal with high probability, otherwise, the algorithm could potentially incur linear regret. Therefore, the knowledge of this ``separability parameter'' is crucial. For our comparison, we consider the exploration period suggested by \cite{Saure} and compare it with the performance of Algorithm \ref{learn_algo} for different values of separation ($\epsilon$). We will see that for any given exploration period, there is an instance where the approach in \cite{Saure} ``breaks down'' or in other words incurs linear regret, while the regret of Algorithm \ref{learn_algo} grows sub-linearly ($O(\sqrt T)$, more precisely) for all values of $\epsilon$ as asserted in Theorem \ref{main_result}. 


\medskip \noindent {\bf Experimental setup and results.} We consider the parametric MNL setting as described in \eqref{eq:problem_instance} and for each value of $\epsilon \in \{0.05, 0.1, 0.15, 0.25\}$. Since the implementation of the policy in \cite{Saure} requires knowledge of the selling horizon and minimum exploration period a priori, we take the exploration period to be $20 \log{T}$ as suggested in~\cite{Saure} and the selling horizon $T = 10^6$. Figure \ref{fig:regret} compares the regret of Algorithm \ref{learn_algo} with that of \cite{Saure}.  The results are based on running 100 independent simulations with standard error of $2\%$. We observe that the regret for \cite{Saure} is better than the regret of Algorithm \ref{learn_algo} when $\epsilon = 0.25$ but is worse for other values of $\epsilon$. This can be attributed to the fact that for the assumed exploration period, their algorithm fails to identify the optimal assortment within the exploration phase with sufficient probability and hence incurs a linear regret for $\epsilon = 0.05, 0.1 \;\text{and} \; 0.15$.   Specifically, among the 100 simulations we tested, the algorithm of \cite{Saure} identified the optimal assortment for only $7\%, 40\%, 61\% \; \text{and} \; 97\%$ cases, when $\epsilon = 0.05, 0.1, 0.15, \; \text{and} \;0.25$, respectively. This highlights the sensitivity to the ``separability parameter'' and the importance of having a reasonable estimate for the exploration period. Needless to say, such information is typically not available in practice. In contrast, the performance of Algorithm \ref{learn_algo} is consistent across different values of $\epsilon$, insofar as the regret grows in a sub-linear fashion in all cases. 

\begin{figure}
\begin{center}
\begin{tabular}{c c}
$\overset{(a)}{\includegraphics[width=3.5in,height=2in]{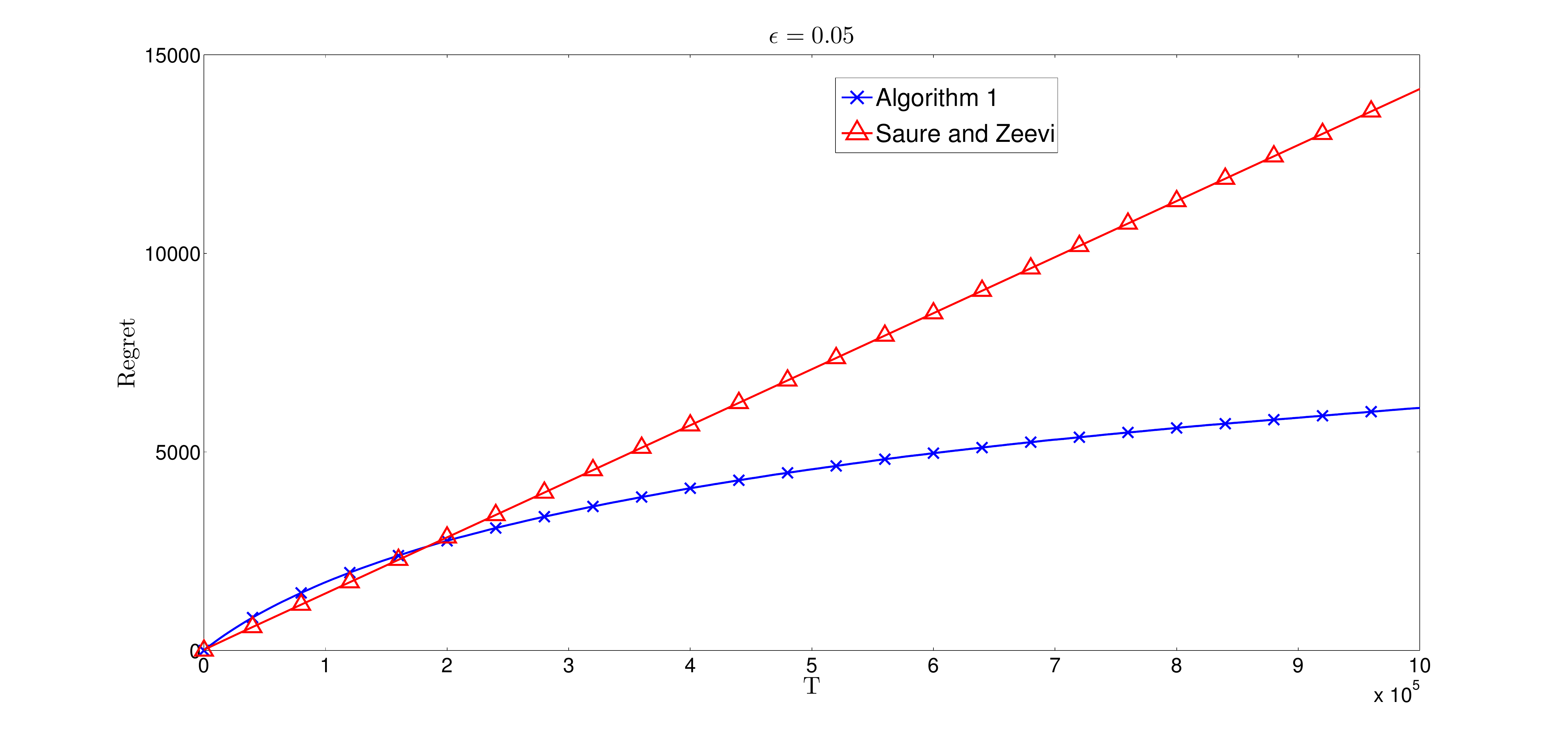}}$ & $\overset{(b)}{\includegraphics[width=3.5in,height=2in]{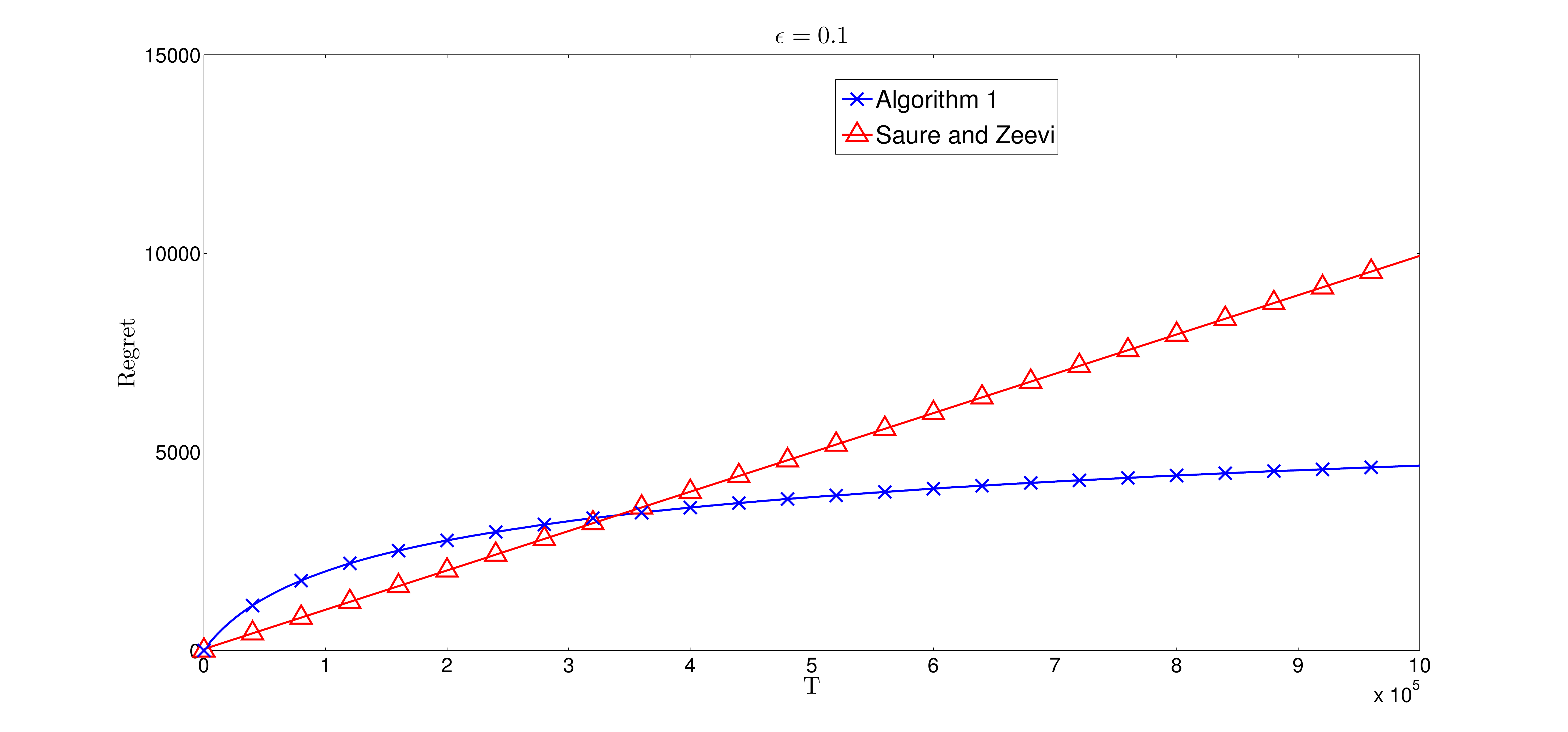}}$ \\
$\overset{(c)}{\includegraphics[width=3.5in,height=2in]{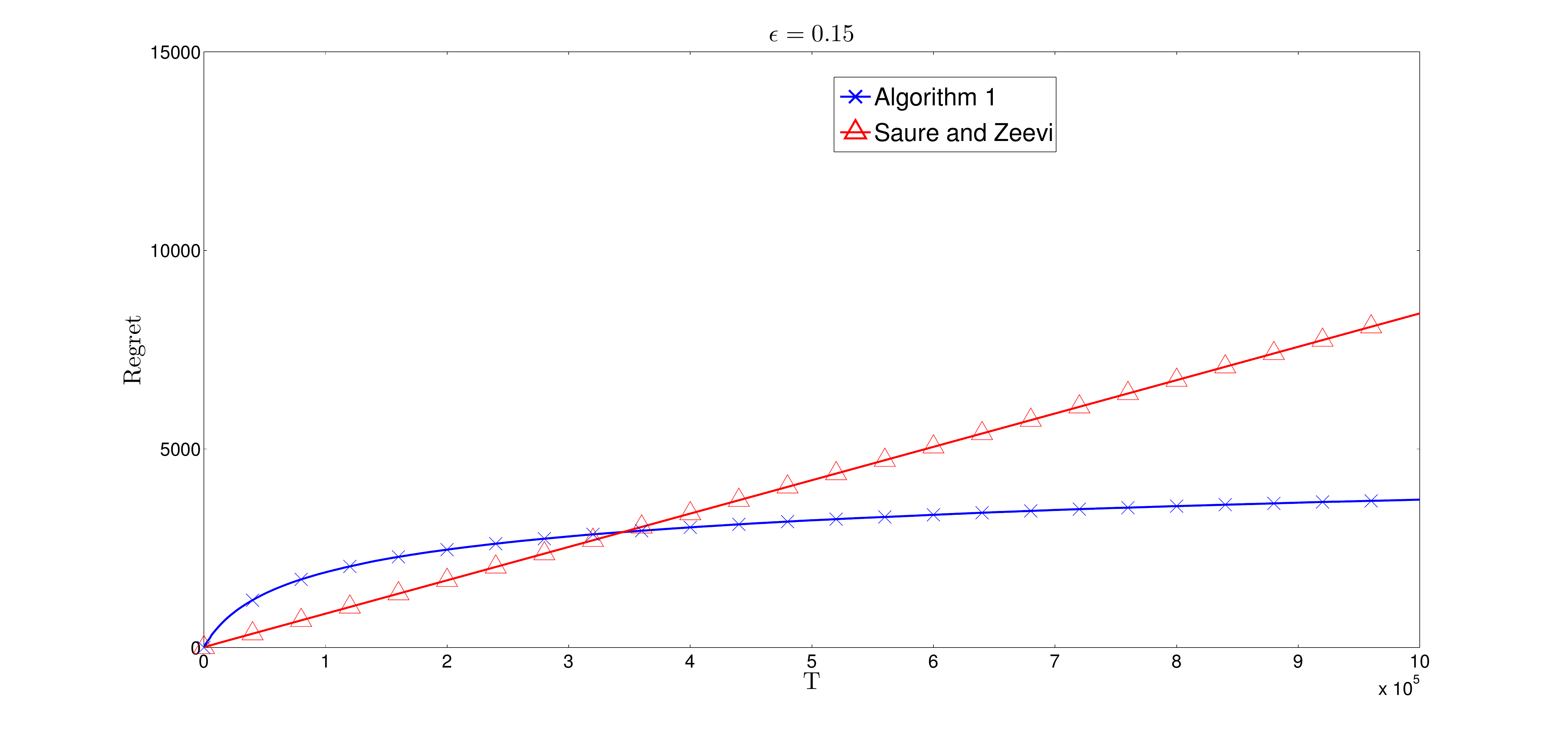}}$ & $\overset{(d)}{\includegraphics[width=3.5in,height=2in]{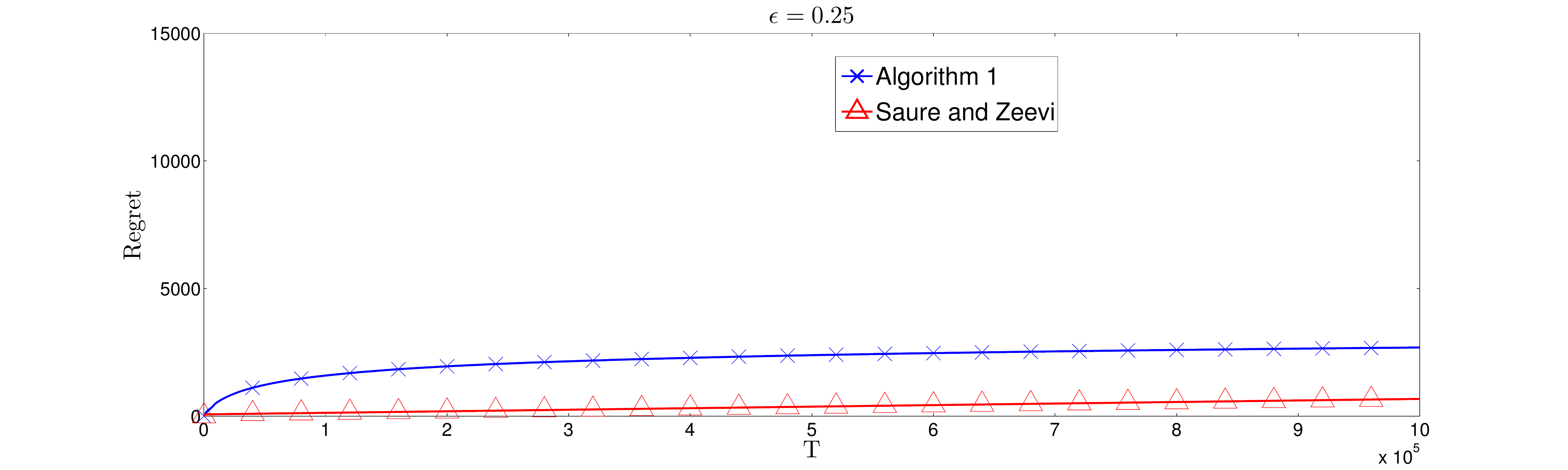}}$
\end{tabular}\caption{Comparison with the algorithm of \cite{Saure}. The graphs (a), (b), (c) and (d) compares the performance of Algorithm \ref{learn_algo} to that of \cite{Saure} on problem instance \eqref{eq:problem_instance}, for $\epsilon = 0.05, 0.1, 0.15 \; \text{and} \; 0.25$  respectively. }\label{fig:regret}
\end{center}
\end{figure}

\subsection{Performance of Algorithm \ref{learn_algo} on a simulation of real data}
Here, we present the results of a simulated study  on a real data set and compare the performance  of Algorithm \ref{learn_algo} to that of \cite{Saure}. 

\medskip \noindent {\bf Data description.} We consider the ``UCI Car Evaluation Database'' (see \cite{UCI}) which contains attributes for $N = 1728$ cars and consumer ratings for each car.  
The exact details of the attributes are provided in Table \ref{attributes: cars}. Rating for each car is also available. In particular, every car is associated with one of the following four ratings, unacceptable, acceptable, good and very good. 

\begin{table}
\begin{center}
\begin{tabular}{ c | c }
\hline
Attribute & Attribute Values\\
\hline
\hline
price & Very-high, high, medium, low\\
\hline
maintenance costs & Very-high, high, medium, low\\
\hline
\# doors & 2, 3, 4,  5 or more\\
\hline
passenger capacity  & 2, 4, more than 4\\
\hline
luggage capacity & small, medium and big \\
\hline
safety perception & low, medium, high\\
\end{tabular}\\
\caption{Attribute information of cars in the database}\label{attributes: cars}
\end{center}
\end{table}

\medskip \noindent {\bf Assortment optimization framework.} 
We assume that the consumer choice is modeled by the MNL model, where the mean utility of a product is linear in the values of attributes. More specifically, we convert the categorical attributes described in Table \ref{attributes: cars} to attributes with binary values by adding dummy attributes (for example ``price very high'', ``price low'' are considered as two different attributes that can take values 1 or 0). Now every car is associated with an attribute vector $m_i \in \{0,1\}^{22}$, which is known a priori and the mean utility of product $i$ is given by the inner product $$\mu_i = {\theta}\cdot m_i \;\; \; i =1,\ldots,N,$$ where ${\theta} \in \mathbb{R}^{22}$  is some fixed but initially unknown attribute weight vector. Under this model, the probability that a consumer purchases product $i$ when offered an assortment $S \subset \{1,\ldots, N\}$ is assumed to be, 
\begin{equation}\label{choice_probabilities_feature}
p_i(S) =\begin{cases}
 \displaystyle \frac{e^{{\theta}\cdot m_i}}{1+\sum_{j \in S}e^{{\theta}\cdot m_j}}, &\quad \text{if} \; i \in S \cup \{0\}\\
0, & \quad \; \text{otherwise},
\end{cases}
\end{equation}
Let $\mb{m} = \left(m_1, \ldots, m_N\right)$. Our goal is to offer assortments $S_1, \ldots, S_T$ at times $1,\ldots,T$ respectively such that the cumulative sales are maximized or alternatively, minimize the regret defined as 
\begin{equation}\label{Regret_sales}
Reg_\pi(T,\mb{m}) = \sum_{t=1}^T \left(\sum_{i \in S^*} p_i(S) - \sum_{i \in S_t} p_i(S_t)\right),
\end{equation}
where
\[ S^* = \arg\max_{S} \sum_{i \in S}  \frac{e^{{\theta}\cdot m_i}}{1+\sum_{j \in S}e^{{\theta}\cdot m_j}}.\]
Note that regret defined in \eqref{Regret_sales} is a special case formulation of the regret defined in \eqref{Regret} with $r_i = 1$ and $v_i = e^{{\theta}\cdot m_i}$ for all $i = 1, \ldots, N$.

\medskip \noindent {\bf Experimental setup and results.}  We first estimate a ground truth MNL model as follows. Using the available attribute level data and consumer rating for each car, we estimate a logistic model assuming every car's rating is independent of the ratings of other cars to estimate the attribute weight vector $\theta$. Specifically, under the logistic model, the probability that a consumer will purchase a car whose attributes are defined by the vector $m \in \{0,1\}^{22}$ and the attribute weight vector $\theta$ is given by
$$p_{\sf buy}(\theta, m) \overset{\Delta}{=} \; \mathbb{P}\left(\text{buy} \middle| \theta \right) = \frac{e^{\theta\cdot m}}{1+e^{\theta\cdot m}}.$$
For the purpose of training the logistic model on the available data, we consider the consumer ratings of ``acceptable,'' ``good,'' and ``very good'' as success or intention to buy and the consumer rating of ``unacceptable'' as a failure or no intention to buy.  We then use the maximum likelihood estimate $\theta_{\sf MLE}$ for $\theta$ to run simulations and study the performance of Algorithm \ref{learn_algo} for the realized $\theta_{\sf MLE}$. In particular, we compute $\theta_{\sf MLE}$ that maximizes the following regularized log-likelihood

\begin{figure}[t]
\begin{center}
\includegraphics[height=2.5in]{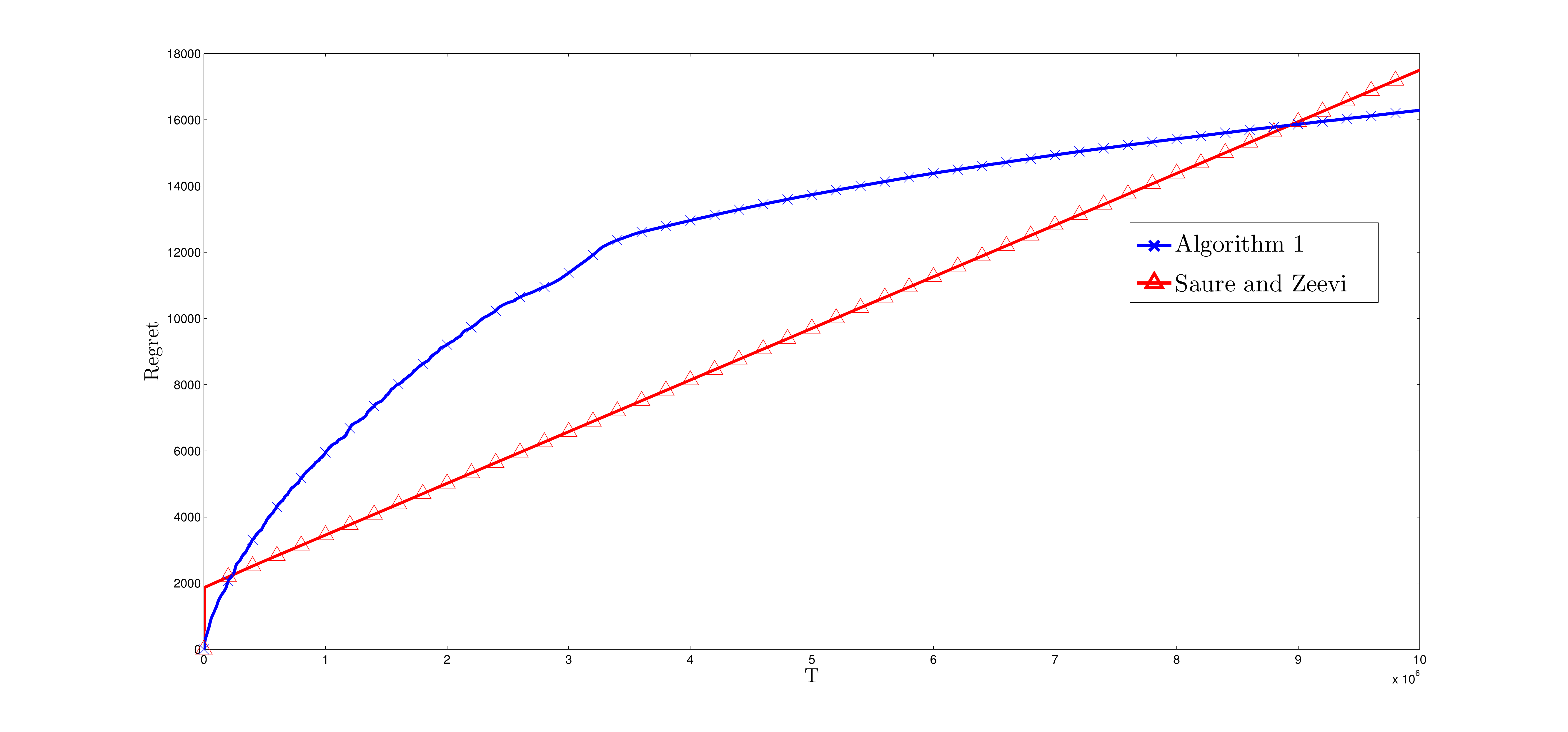}
\caption{\change{Comparison with the algorithm of \cite{Saure} on real data. The graph compares the performance of Algorithm \ref{learn_algo} to that of \cite{Saure} on the ``UCI Car Evaluation Databse' for $T = 10^7$.} }\label{fig:regret_car}
\end{center}
\end{figure}

$$\theta_{\sf MLE} = \underset{\theta}{\text{arg}\max} \;\sum_{i=1}^N \log{p_{\sf buy}(\theta,m_i)} - \|\theta\|_2.$$ The objective function in the preceding optimization problem is convex and therefore we can use any of the standard convex optimization techniques to obtain the estimate, $\theta_{\sf MLE}$. It is important to note that the logistic model is only employed to obtain an estimate for $\theta$, $\theta_{\sf MLE}$. The estimate $\theta_{\sf MLE}$ is assumed to be the ground truth MNL model and is used to simulate the feedback of consumer choices for our learning Algorithm \ref{learn_algo} {and the learning algorithm proposed by \cite{Saure}.} 

{We compare the performance of Algorithm~\ref{learn_algo} with that of \cite{Saure}, in terms of regret as defined in~\eqref{Regret_sales} with $\theta = \theta_{\sf MLE}$ and at each time index, the retailer can only show at most  $k=100$ cars. We implement \cite{Saure}'s approach with their suggested mandotary exploration period, which explores every product for at least $20 \log{T}$ periods.  Figure \ref{fig:regret_car} plots the regret of Algorithm \ref{learn_algo} and the \cite{Saure} policy, when the selling horizon is $T=10^7.$ The results are based on running $100$ independent simulations and have a standard error of $2\%$. We can observe that while the initial regret of \cite{Saure} is smaller, the regret grows linearly with time, suggesting that the exploration period was too small. This further illustrates the shortcomings of an explore-then-exploit approach which requires knowledge of underlying parameters. In contrast, the regret of Algorithm \ref{learn_algo} grows in a sublinear fashion with respect to the selling horizon and does not require any a priori knowledge on the parameters, making a case for the universal applicability of our approach. }

\section{Conclusions and future work}

\medskip \noindent {\bf {Summary and main insights.}} { In this paper, we have studied the dynamic assortment selection problem under the widely used multinomial logit  choice model. Formulating the problem as a parametric multi-arm bandit problem, we present a policy that learns the parameters of the choice model while simultaneously maximizing the cumulative revenue. Focusing on a policy that would be universally applicable, we highlight the limitations of existing approaches and present a novel computationally efficient algorithm, whose performance (as measured by the regret) is nearly-optimal. Furthermore, our policy is adaptive to the complexity of the problem instance, as measured by ``separability'' of items. The adaptive nature of the algorithm is manifest in its ``rate of learning'' the unknown instance parameters, which is more rapid if the problem instance is ``less complex.''   } 

\medskip \noindent {\bf {Limitations and future research.}}
{In this work we primarily focused on developing an algorithm that would be broadly applicable. In so doing,  we only consider the setting where every product has its own utility parameter and has to be estimated separately. However, in many settings a large number of products are effectively described by a small number of product features, via what is often referred to as factor model. An important extension of our problem would be to consider a policy that leverages the relation between products as measured via their features, and achieves a regret bound that is independent of the number of products and only depends on the dimensionality of feature space.}

\change{Another interesting direction is to consider the settings where the consumers are heterogeneous. If the consumer type is known a priori, then we can easily generalize our algorithm to learn only model parameters of that type. In a recent work, \cite{kallus} consider the setting of heterogeneous consumers where each consumer segment follows a separate MNL model, but the underlying structure of these parameters over all the segments has low dimension. Assuming the consumer type is observable a priori, they present an explore first exploit later approach to dynamically learn the preferences of heterogeneous consumer population. Their work also demonstrates significant improvements in performance in comparison to trivially extending the existing dynamic learning approaches (\citealp{Saure}, \citealp{rusme}) to learn a different MNL model for each consumer type. Generalizing our work to design a parameter independent algorithm to learn the preferences of heterogeneous consumers with an underlying low rank structure would be an important extension with significant practical implications.}

\change{As discussed earlier, Thompson Sampling is a natural algorithm for the \banditMNL\;problem. Despite being empirically superior to other bandit policies, TS-based algorithms remain challenging to analyze and theoretical work on TS is limited.  An interesting direction is to consider a TS-based approach for the \banditMNL\;problem and derive similar regret bounds to the ones obtained in this paper. Due to its combinatorial nature, selecting a suitable prior and efficiently updating the posterior present a significant challenge in designing a good TS-based algorithm for the \banditMNL~problem. Some preliminary results in this direction are reported in \cite{agrawal2017thompson}.}

\ACKNOWLEDGMENT{V. Goyal is supported in part by NSF Grants CMMI-1351838 (CAREER) and CMMI-1636046. A. Zeevi is supported in part by NSF Grants NetSE-0964170 and  BSF-2010466.}
{\Large{
\bibliographystyle{ormsv080}
\bibliography{acmsmall-sample-bibfile}

\begin{thebibliography}{42}
\expandafter\ifx\csname natexlab\endcsname\relax\def\natexlab#1{#1}\fi
\expandafter\ifx\csname url\endcsname\relax
  \def\url#1{{\tt #1}}\fi
\expandafter\ifx\csname urlprefix\endcsname\relax\def\urlprefix{URL }\fi
\expandafter\ifx\csname urlstyle\endcsname\relax
  \expandafter\ifx\csname doi\endcsname\relax
  \def\doi#1{doi:\discretionary{}{}{}#1}\fi \else
  \expandafter\ifx\csname doi\endcsname\relax
  \def\doi{doi:\discretionary{}{}{}\begingroup \urlstyle{rm}\Url}\fi \fi

\bibitem[{Agrawal et~al.(2017)Agrawal, Avadhanula, Goyal, and
  Zeevi}]{agrawal2017thompson}
Agrawal, S., V.~Avadhanula, V.~Goyal, A.~Zeevi. 2017.
\newblock Thompson sampling for the mnl-bandit.
\newblock {\it Proceedings of Machine Learning Research\/} {\bf (65)} 76--78.

\bibitem[{Agrawal and Goyal(2013)}]{AgrawalTS_lin}
Agrawal, S., N.~Goyal. 2013.
\newblock Thompson sampling for contextual bandits with linear payoffs.
\newblock {\it Proceedings of the 30th International Conference on
  International Conference on Machine Learning\/} {\bf 28}.

\bibitem[{Agrawal and Goyal(2017)}]{AgrawalTS_nearopt}
Agrawal, S., N.~Goyal. 2017.
\newblock Near-optimal regret bounds for thompson sampling.
\newblock {\it J. ACM\/} {\bf 64}(5).

\bibitem[{Angluin and Valiant(1977)}]{angluin}
Angluin, D., L.~G. Valiant. 1977.
\newblock Fast probabilistic algorithms for hamiltonian circuits and matchings.
\newblock {\it Proceedings of the Ninth Annual ACM Symposium on Theory of
  Computing\/}. STOC '77, 30--41.

\bibitem[{Auer(2003)}]{auer2003}
Auer, P. 2003.
\newblock Using confidence bounds for exploitation-exploration trade-offs.
\newblock {\it Journal of Machine Learning Research. \/} 

\bibitem[{Auer et~al.(2002)Auer, Cesa-Bianchi, and Fischer}]{auer}
Auer, P., N.~Cesa-Bianchi, P.~Fischer. 2002.
\newblock Finite-time analysis of the multiarmed bandit problem.
\newblock {\it Machine learning\/} {\bf 47}(2-3) 235--256.

\bibitem[{Avadhanula et~al.(2016)Avadhanula, Bhandari, Goyal, and
  Zeevi}]{avadhanula2016tightness}
Avadhanula, V., J.~Bhandari, V.~Goyal, A.~Zeevi. 2016.
\newblock On the tightness of an lp relaxation for rational optimization and
  its applications.
\newblock {\it Operations Research Letters\/} {\bf 44}(5) 612--617.

\bibitem[{Babaioff et~al.(2015)Babaioff, Dughmi, Kleinberg, and
  Slivkins}]{Babaioff}
Babaioff, M., S.~Dughmi, R.~Kleinberg, A.~Slivkins. 2015.
\newblock Dynamic pricing with limited supply.
\newblock {\it ACM Transactions on Economics and Computation\/} {\bf 3}(1) 4.

\bibitem[{Ben-Akiva and Lerman(1985)}]{BenLerman85}
Ben-Akiva, M., S.~Lerman. 1985.
\newblock {\it Discrete choice analysis: theory and application to travel
  demand\/}, vol.~9.
\newblock MIT press.

\bibitem[{Blanchet et~al.(2016)Blanchet, Gallego, and Goyal}]{BlanchetGG13}
Blanchet, J., G.~Gallego, V.~Goyal. 2016.
\newblock A markov chain approximation to choice modeling.
\newblock {\it Operations Research\/} {\bf 64}(4) 886--905.

\bibitem[{Borovkov(1984)}]{borovkov1984mathematical}
Borovkov, AA. 1984.
\newblock Mathematical statistics. estimation of parameters, testing of
  hypotheses .

\bibitem[{Bubeck and Cesa-Bianchi(2012)}]{MAL-024}
Bubeck, S., N.~Cesa-Bianchi. 2012.
\newblock Regret analysis of stochastic and nonstochastic multi-armed bandit
  problems.
\newblock {\it Foundations and Trends in Machine Learning\/} .

\bibitem[{Caro and Gallien(2007)}]{caro07}
Caro, F., J.~Gallien. 2007.
\newblock Dynamic assortment with demand learning for seasonal consumer goods.
\newblock {\it Management Science\/} {\bf 53}(2) 276--292.

\bibitem[{Chen et~al.(2013)Chen, Wang, and Yuan}]{Chen}
Chen, W., Y.~Wang, Y.~Yuan. 2013.
\newblock Combinatorial multi-armed bandit: General framework, results and
  applications.
\newblock {\it Proceedings of the 30th international conference on machine
  learning\/}. 151--159.

\bibitem[{Chen and Wang(2017)}]{xichen}
Chen, X., Y.~Wang. 2017.
\newblock A note on tight lower bound for mnl-bandit assortment selection
  models.
\newblock {\it ArXiv e-prints\/} .

\bibitem[{Davis et~al.(2013)Davis, Gallego, and
  Topaloglu}]{davis2013assortment}
Davis, J., G.~Gallego, H.~Topaloglu. 2013.
\newblock Assortment planning under the multinomial logit model with totally
  unimodular constraint structures.
\newblock {\it Technical Report, Cornell University.\/} .

\bibitem[{Davis et~al.(2014)Davis, Gallego, and Topaloglu}]{DGT12}
Davis, J.M., G.~Gallego, H.~Topaloglu. 2014.
\newblock Assortment optimization under variants of the nested logit model.
\newblock {\it Operations Research\/} {\bf 62}(2) 250--273.

\bibitem[{D{\'e}sir and Goyal(2014)}]{desir2014near}
D{\'e}sir, A., V.~Goyal. 2014.
\newblock Near-optimal algorithms for capacity constrained assortment
  optimization.
\newblock {\it Available at SSRN\/} .

\bibitem[{D{\'e}sir et~al.(2015)D{\'e}sir, Goyal, Segev, and
  Ye}]{desirmarkovchain}
D{\'e}sir, A., V.~Goyal, D.~Segev, C.~Ye. 2015.
\newblock Capacity constrained assortment optimization under the markov chain
  based choice model.
\newblock {\it Working Paper, Columbia University\/} .

\bibitem[{Farias et~al.(2013)Farias, Jagabathula, and Shah}]{FJS12}
Farias, V., S.~Jagabathula, D.~Shah. 2013.
\newblock A nonparametric approach to modeling choice with limited data.
\newblock {\it Management Science\/} {\bf 59}(2) 305--322.

\bibitem[{Filippi et~al.(2010)Filippi, Cappe, Garivier, and
  Szepesv{\'a}ri}]{genlinucb}
Filippi, S., O.~Cappe, A.~Garivier, C.~Szepesv{\'a}ri. 2010.
\newblock Parametric bandits: The generalized linear case.
\newblock {\it Advances in Neural Information Processing Systems\/}. 586--594.

\bibitem[{Gallego et~al.(2014)Gallego, Ratliff, and Shebalov}]{gallego}
Gallego, G., R.~Ratliff, S.~Shebalov. 2014.
\newblock A general attraction model and sales-based linear program for network
  revenue management under customer choice.
\newblock {\it Operations Research\/} {\bf 63}(1) 212--232.

\bibitem[{Gallego and Topaloglu(2014)}]{gallego2014constrained}
Gallego, G., H.~Topaloglu. 2014.
\newblock Constrained assortment optimization for the nested logit model.
\newblock {\it Management Science\/} {\bf 60}(10) 2583--2601.

\bibitem[{Kallus and Udell(2016)}]{kallus}
Kallus, N., M.~Udell. 2016.
\newblock Dynamic assortment personalization in high dimensions.
\newblock {\it ArXiv e-prints\/} .

\bibitem[{Kleinberg et~al.(2008)Kleinberg, Slivkins, and Upfal}]{Kleinberg}
Kleinberg, R., A.~Slivkins, E.~Upfal. 2008.
\newblock Multi-armed bandits in metric spaces.
\newblock {\it Proceedings of the Fortieth Annual ACM Symposium on Theory of
  Computing\/}. STOC '08, 681--690.

\bibitem[{K{\"o}k and Fisher(2007)}]{KokFisher07}
K{\"o}k, A.~G., M.~L. Fisher. 2007.
\newblock Demand estimation and assortment optimization under substitution:
  Methodology and application.
\newblock {\it Operations Research\/} {\bf 55}(6) 1001--1021.

\bibitem[{Lai and Robbins(1985)}]{Lai}
Lai, T.L., H.~Robbins. 1985.
\newblock Asymptotically efficient adaptive allocation rules.
\newblock {\it Adv. Appl. Math.\/} {\bf 6}(1) 4--22.

\bibitem[{Li et~al.(2015)Li, Rusmevichientong, and Topaloglu}]{dlevelnl}
Li, G., P.~Rusmevichientong, H.~Topaloglu. 2015.
\newblock The d-level nested logit model: Assortment and price optimization
  problems.
\newblock {\it Operations Research\/} {\bf 63}(2) 325--342.

\bibitem[{Lichman(2013)}]{UCI}
Lichman, M. 2013.
\newblock {UCI} machine learning repository.
\newblock \urlprefix\url{http://archive.ics.uci.edu/ml}.

\bibitem[{Luce(1959)}]{Luce59}
Luce, R.D. 1959.
\newblock {\it Individual choice behavior: A theoretical analysis\/}.
\newblock Wiley.

\bibitem[{May et~al.(2012)May, Korda, Lee, and Leslie}]{may2012optimistic}
May, B.~C., N.~Korda, A.~Lee, D.~S. Leslie. 2012.
\newblock Optimistic bayesian sampling in contextual-bandit problems.
\newblock {\it Journal of Machine Learning Research\/} {\bf (13)} 2069--2106.

\bibitem[{McFadden(1978)}]{McFadden78}
McFadden, D. 1978.
\newblock Modeling the choice of residential location.
\newblock {\it Transportation Research Record\/} (673).

\bibitem[{Mitzenmacher and Upfal(2005)}]{mitzenmacher}
Mitzenmacher, M., E.~Upfal. 2005.
\newblock {\it Probability and computing: Randomized algorithms and
  probabilistic analysis\/}.
\newblock Cambridge university press.

\bibitem[{Oliver and Li(2011)}]{chapelle}
Oliver, C., L.~Li. 2011.
\newblock An empirical evaluation of thompson sampling.
\newblock {\it In Advances in Neural Information Processing Systems (NIPS)\/}
  {\bf 24} 2249--2257.

\bibitem[{Plackett(1975)}]{Plackett75}
Plackett, R.~L. 1975.
\newblock The analysis of permutations.
\newblock {\it Applied Statistics\/}  193--202.

\bibitem[{Robbins(1952)}]{Robbins1952}
Robbins, H. 1952.
\newblock Some aspects of the sequential design of experiments.
\newblock {\it Bulletin of the American Mathematical Society\/} {\bf 58}(5)
  527--535.

\bibitem[{Rusmevichientong et~al.(2010)Rusmevichientong, Shen, and
  Shmoys}]{rusme}
Rusmevichientong, P., Z.~M. Shen, D.~B. Shmoys. 2010.
\newblock Dynamic assortment optimization with a multinomial logit choice model
  and capacity constraint.
\newblock {\it Operations research\/} {\bf 58}(6) 1666--1680.

\bibitem[{Rusmevichientong and Tsitsiklis(2010)}]{linucb}
Rusmevichientong, P., J.~N. Tsitsiklis. 2010.
\newblock Linearly parameterized bandits.
\newblock {\it Math. Oper. Res.\/} {\bf 35}(2) 395--411.

\bibitem[{Saur{\'e} and Zeevi(2013)}]{Saure}
Saur{\'e}, D., A.~Zeevi. 2013.
\newblock Optimal dynamic assortment planning with demand learning.
\newblock {\it Manufacturing \& Service Operations Management\/} {\bf 15}(3)
  387--404.

\bibitem[{Talluri and van Ryzin(2004)}]{TV04}
Talluri, K., G.~van Ryzin. 2004.
\newblock Revenue management under a general discrete choice model of consumer
  behavior.
\newblock {\it Management Science\/} {\bf 50}(1) 15--33.

\bibitem[{Train(2009)}]{train2003discrete}
Train, K.~E. 2009.
\newblock {\it Discrete choice methods with simulation\/}.
\newblock Cambridge university press.

\bibitem[{Williams(1977)}]{williams}
Williams, H.C.W.L. 1977.
\newblock {On the formation of travel demand models and economic evaluation
  measures of user benefit}.
\newblock {\it Environment and Planning A\/} {\bf 9}(3) 285--344.

\end{thebibliography}
}}

\begin{appendices}
\section{Proof of Theorem \ref{main_result}} \label{Appendix_proof_thm1}
{ In this section, we provide a detailed proof of Theorem \ref{main_result} following the outline discussed in Section \ref{proof_outline_thm1}. The proof is organized as follows. In Section \ref{sec:UCBva}, we complete the proof of Lemma \ref{lem:UCBv} and in Section \ref{sec:UCBSa}, we prove Lemma \ref{lem:UCBS1} and Lemma \ref{lem:UCBS2}. Finally, in Section \ref{sec:completeProof}, we utilize these results to complete the proof of Theorem \ref{main_result}.}

\subsection{Properties of estimates $v^{\sf UCB}_{i,\ell}$: Proof of Lemma \ref{lem:UCBv}}
\label{sec:UCBva}
First, we prove Lemma \ref{lem:UCBv}. To complete the proof, we establish certain properties of the estimates $v^{\sf UCB}_{i,\ell}$, and then extend these properties to establish the necessary properties of $\hat{v}_{i,\ell}$ and $\bar{v}_{i,\ell}$. 

\begin{lemma}[Moment Generating Function]\label{moment_generating}
The moment generating function of estimate conditioned on $S_\ell$, $\hat{v}_{i}$, is given by,
\begin{equation*}
\mathbb{E}_{\pi}\left(e^{\theta\hat{v}_{i,\ell}} {\Big |} S_\ell\right) = \frac{1}{1-v_i(e^\theta-1)},\;\text{for all}\; \;\theta \leq \log{\frac{1+v_i}{v_i}},  \; \text{for all} \;\; i=1,\cdots,N.
\end{equation*}
\end{lemma}
\proof{Proof.}
 From \eqref{choice_probabilities}, we have that probability of no purchase event when assortment $S_\ell$ is offered is given by \[\textstyle p_0(S_\ell) = \displaystyle \frac{1}{1+\sum_{j\in S_\ell} v_j}.\] 
Let $n_\ell$ be the total number of offerings in epoch $\ell$ before a no purchased occurred, i.e., $n_\ell = |\ep{E}_\ell|-1$. 
Therefore, $n_\ell$ is a geometric random variable with probability of success $p_0(S_\ell)$. 
And, given any fixed value of $n_{\ell}$,
$\hat{v}_{i,\ell}$ is a binomial random variable with $n_\ell$ trials and probability of success given by 
$$ \textstyle q_i(S_\ell)= \displaystyle \frac{v_i}{\sum_{j\in S_\ell} v_j}.$$ 
In the calculations below, for brevity we use $p_0$ and $q_i$ respectively to denote $p_0(S_\ell)$ and $q_i(S_\ell)$. Hence, we have 
\begin{equation}\label{eq:conditional_mgf}\mathbb{E}_{\pi}\left( e^{\theta\hat{v}_{i,\ell}}\right) = E_{n_\ell}\left\{\mathbb{E}_{\pi}\left( e^{\theta\hat{v}_{i,\ell}}\,\middle| \,n_\ell\right)\right\}.\end{equation}
Since the moment generating function for a binomial random variable with parameters $n,p$ is $\left(pe^\theta + 1-p\right)^n$, we have
\begin{equation}\label{eq:binomial_mgf}\mathbb{E}_{\pi}\left( e^{\theta\hat{v}_{i,\ell}}\, \middle| \, n_\ell\right) =  \mathbb{E}_{n_\ell}\left\{{\left(q_ie^\theta + 1-q_i\right)^{n_\ell}}\right\}.\end{equation}
For any $\alpha$, such that $\alpha(1-p) < 1$, if $n$ is a geometric random variable with parameter $p$, then we have \[\mathbb{E}(\alpha^n) = \frac{p}{1-\alpha(1-p)}.\] 
{Since $n_\ell$ is a geometric random variable with parameter $p_0$ and by definition of $q_i$ and $p_0$, we have, $q_i(1-p_0) = v_i p_0$, it follows that for any $\theta < \log{\frac{1+v_i}{v_i}}$, we have,
\begin{equation}\label{eq:geometric_mgf}\mathbb{E}_{n_\ell}\left\{{\left(q_ie^\theta + 1-q_i\right)^{n_\ell}}\right\} =  \frac{p_0}{1-\left(q_ie^\theta + 1-q_i\right)(1-p_0)} = \frac{1}{1-v_i(e^\theta-1)}.\end{equation}
The result follows from \eqref{eq:conditional_mgf}, \eqref{eq:binomial_mgf} and \eqref{eq:geometric_mgf}}. \hfill $\Halmos$
\\

From the moment generating function, we can establish that $\hat{v}_{i,\ell}$ is a geometric random variable with parameter $\frac{1}{1+v_i}$. Thereby also establishing that $\hat{v}_{i,\ell}$ and $\bar{v}_{i,\ell}$ are unbiased estimators of $v_i$.  More specifically, from Lemma \ref{moment_generating}, we have the following result.
\begin{corollary}[Unbiased Estimates]\label{unbiased_estimate}
We have the following results.
\begin{enumerate}
\item $ \hat{v}_{i,\ell}, \; \ell \leq L$ are  i.i.d geometrical  random variables with parameter $\frac{1}{1+v_i}$, i .e. for any $\ell,i$ $$\mathbb{P}_{\pi}\left(\hat{v}_{i,\ell} = m\right) = {\left(\frac{v_i}{1+v_i}\right)}^m\left(\frac{1}{1+v_i}\right) \ \;\forall\; m=\{0,1,2,\cdots\} .$$
\item $ \hat{v}_{i,\ell}, \; \ell \leq L$ are unbiased  i.i.d estimates of $v_i$, i .e. $\mathbb{E}_{\pi}\left(\hat{v}_{i,\ell}\right) = v_i \;\forall\; \ell, i.$ \label{p1}
\end{enumerate}
\end{corollary}


From \mbox{Corollary \ref{unbiased_estimate}}, it follows that $\hat{v}_{i,\tau},\tau \in \mathcal{T}_i(\ell)$ are i.i.d  geometric random variables with mean $v_i$. We will use this observation and extend the multiplicative Chernoff-Hoeffding bounds discussed in \cite{mitzenmacher} and \cite{Babaioff}  to geometric random variables and derive the following result. 
\begin{lemma}[Concentration Bounds]\label{chernoff_hoeffding_ineq}
If $v_i \leq v_0$ for all $i$, for every epoch $\ell$, in Algorithm $\ref{learn_algo}$, we have the following concentration bounds.
\begin{enumerate}
\item $\displaystyle \mathbb{P}_{\pi}\left(\left|\bar{v}_{i,\ell} - v_i \right | > \sqrt{48\displaystyle \bar{v}_{i,\ell}\frac{\log{({\sqrt{N}}\ell+1)}}{{T}_{i}(\ell)}} + \frac{48\log{(\ell+1)}}{T_i(\ell)}\;\right)  \leq \frac{6}{{N}\ell}$.
\item $\displaystyle \mathbb{P}_{\pi}\left(\left|\bar{v}_{i,\ell} - v_i \right | > \sqrt{24v_i\displaystyle \frac{\log{({\sqrt{N}}\ell+1)}}{{T}_{i}(\ell)}} + \frac{48\log{({\sqrt{N}}\ell+1)}}{{T}_i(\ell)}\;\right)  \leq  \frac{4}{{N}\ell}$.
\item  {$\displaystyle \mathbb{P}_{\pi}\left(\bar{v}_{i,\ell}   > \frac{3v_i}{2} + \frac{48 \log{({\sqrt{N}}\ell+1)}}{T_i(\ell)}\right)   \leq  \frac{3}{{N}\ell}.\;\; $}
\end{enumerate}
\end{lemma}
Note that to apply standard Chernoff-Hoeffding inequality (see p.66 in \citealp{mitzenmacher}), we must have the individual sample values bounded by some constant, which is not the case with our estimates $\hat{v}_{i,\tau}$. However, these estimates are geometric random variables and therefore have extremely small tails. Hence, we can extend the Chernoff-Hoeffding bounds discussed in \cite{mitzenmacher} and \cite{Babaioff} to geometric variables and prove the above result. Lemma \ref{lem:UCBv} follows directly from Lemma \ref{chernoff_hoeffding_ineq}  {(see below.)} The proof of Lemma \ref{chernoff_hoeffding_ineq} is long and tedious and in the interest of continuity, we complete the proof in Appendix \ref{proof:multiplicative_chernoff}. Following the proof of Lemma \ref{chernoff_hoeffding_ineq}, we obtain a very similar result that is useful to establish concentration bounds for the general parameter setting. 

\medskip \noindent {\bf {Proof of Lemma \ref{lem:UCBv}}:}
By design of Algorithm \ref{learn_algo}, we have, 
\begin{equation}\label{eq:ucb_define}
v^{\sf UCB}_{i,\ell} = \bar{v}_{i,\ell} + \sqrt{48\displaystyle \bar{v}_{i,\ell}\frac{\log{(\sqrt{N}\ell+1)}}{{T}_{i}(\ell)}} + \frac{48\log{(\sqrt{N}\ell+1)}}{T_i(\ell)}.
\end{equation}
Therefore from Lemma \ref{chernoff_hoeffding_ineq}, we have \begin{equation}\label{eq:ucb_ineq}
\ep{P}_\pi\left(v^{\sf UCB}_{i,\ell} < v_i\right) \leq \frac{6}{N\ell}.\end{equation}
The first inequality in Lemma \ref{lem:UCBv} follows from \eqref{eq:ucb_ineq}. 
From triangle inequality and \eqref{eq:ucb_define}, we have, 
\begin{equation}\label{eq:a4}
\begin{aligned}
\left|v^{\sf UCB}_{i,\ell} - v_i\right| &\leq \left|v^{\sf UCB}_{i,\ell} - \bar{v}_{i,\ell}\right| + \left|\bar{v}_{i,\ell} - v_i\right| \\
&= \sqrt{48\displaystyle \bar{v}_{i,\ell}\frac{\log{(\sqrt{N}\ell+1)}}{{T}_{i}(\ell)}} + \frac{48\log{(\sqrt{N}\ell+1)}}{T_i(\ell)}+\left|\bar{v}_{i,\ell} - v_i\right|.
\end{aligned}
\end{equation}
From Lemma \ref{chernoff_hoeffding_ineq}, we have 
\begin{equation*}
\displaystyle \mathbb{P}_{\pi}\left(\bar{v}_{i,\ell}  > \frac{3v_i}{2} + \frac{48\log{(\sqrt{N}\ell+1)}}{{T}_i(\ell)}\;\right)  \leq \frac{3}{N\ell},
\end{equation*}
which implies
\begin{equation*}
\displaystyle \mathbb{P}_{\pi}\left(48\bar{v}_{i,\ell}\frac{\log{(\sqrt{N}\ell+1)}}{{T}_{i}(\ell)}  > 72v_i\frac{\log{(\sqrt{N}\ell+1)}}{{T}_{i}(\ell)} + \left(\frac{48\log{(\sqrt{N}\ell+1)}}{{T}_i(\ell)}\right)^2\;\right)  \leq \frac{3}{N\ell},
\end{equation*}
Using the fact that $\sqrt{a+b} < \sqrt{a}+\sqrt{b}$, for any positive numbers  $a,b$, we have,
\begin{equation}\label{eq:a7}
\displaystyle \mathbb{P}_{\pi}\left(\sqrt{48\bar{v}_{i,\ell} \frac{\log{(\sqrt{N}\ell+1)}}{{T}_{i}(\ell)}} +\frac{48\log{(\sqrt{N}\ell+1)}}{T_i(\ell)} > \sqrt{72v_i \frac{\log{(\sqrt{N}\ell+1)}}{{T}_{i}(\ell)}} + \frac{96\log{(\sqrt{N}\ell+1)}}{{T}_i(\ell)}\;\right)  \leq \frac{3}{N\ell},
\end{equation}
From Lemma \ref{chernoff_hoeffding_ineq}, we have,
\begin{equation}\label{eq:a8}
\mathbb{P}_{\pi}\left(\left|\bar{v}_{i,\ell} - v_i \right | > \sqrt{24v_i\displaystyle \frac{\log{(\sqrt{N}\ell+1)}}{{T}_{i}(\ell)}} + \frac{48\log{(\sqrt{N}\ell+1)}}{{T}_i(\ell)}\;\right)  \leq  \frac{4}{N\ell}.
\end{equation}
From \eqref{eq:a4} and applying union bound on \eqref{eq:a7} and \eqref{eq:a8}, we obtain,
$$\ep{P}\left(\left|v^{\sf UCB}_{i,\ell} - v_i\right| > (\sqrt{72}+\sqrt{24})\sqrt{\frac{v_i \log{(\sqrt{N}\ell + 1)}}{T_i(\ell)}} + \frac{144\log{(\sqrt{N}\ell + 1)}}{T_i(\ell)} \right) \leq \frac{7}{N\ell}.$$
Lemma \ref{lem:UCBv} follows from the above inequality and \eqref{eq:ucb_ineq}. \hfill $\halmos$

\subsection{Properties of estimate $\tilde{R}(S)$: Proof of Lemma \ref{lem:UCBS1} and Lemma \ref{lem:UCBS2} }
\label{sec:UCBSa}
{In this section, we prove Lemma \ref{lem:UCBS1} and Lemma \ref{lem:UCBS2}. To complete the proofs, we will establish two auxiliary results, in the first result (see Lemma \ref{UCB_bound}) we show that the expected revenue corresponding to the optimal assortment is monotone in the MNL parameters $\mb{v}$ and in the second result (see Lemma \ref{lipschitz}) we bound the difference between the estimate of the optimal revenue and the true optimal revenue. } 


\begin{lemma}[Optimistic Estimates]\label{UCB_bound}
 Assume $0\leq w_i \leq v^{\sf UCB}_{i}$ for all $i = 1,\cdots,n$. Suppose $S$ is an optimal assortment when the MNL are parameters are given by $\mathbf{w}$. Then,
$R(S, \mathbf{v}^{UCB}) \geq R(S, \mathbf{w}).$
\end{lemma}
\proof{Proof.} We prove the result by first showing that for any $j\in S$, we have $R(S,{\mathbf{w}}^j) \geq R(S,\mathbf{w})$,
where ${\mathbf{w}}^j$ is vector $\mathbf{w}$ with the $j^{th}$ component increased to $v^{\sf UCB}_{j}$, i.e.  $w^j_i = w_i$ for all $i \neq j$ and $w^j_j = v^{\sf UCB}_j$. We can use this result iteratively to argue that increasing each parameter of MNL to the highest possible value increases the value of $R(S,\mathbf{w})$ to complete the proof. 

If there exists $j \in S$ such that $r_j < R(S)$, then removing the product $j$ from assortment $S$ yields higher expected revenue contradicting the optimality of $S$. Therefore, we have 
$$r_j \geq R(S). ~\forall j \in S.$$ 
Multiplying by $({v}^{UCB}_j-{w}_j)(\sum_{i \in S/j } w_i + 1)$ on both sides of the above inequality and re-arranging terms, we can show that $R(S,{\mathbf{w}}^j) \geq R(S,\mathbf{w})$.   \hfill~$\Halmos$

{We would like to remind the readers that Lemma \ref{UCB_bound} does not claim that the expected revenue is in general a monotone function, but only that the value of the expected revenue corresponding to the optimal assortment is monotone in the MNL parameters.}

\medskip \noindent {\bf {Proof of Lemma \ref{lem:UCBS1}:}}
{ Let $\hat{S},\mathbf{w}^*$ be maximizers of the  optimization problem, $$ \underset{S \in \mathcal{S}}{\text{max}}\underset{0\leq w_{i} \leq v^{\sf UCB}_{i,\ell}}{\text{max}}\; R(S,\mathbf{{w}}).$$
Assume $v^{\sf UCB}_{i,\ell}> v_{i}$ for all $i$. Then from Lemma \ref{UCB_bound} it follows that,
\begin{equation}\label{eq:ran_1}
\tilde{R}_\ell(S_\ell) = \underset{S \in \mathcal{S}}{\text{max}} \; R(S,\mathbf{v}_\ell^{UCB}) \geq \underset{S \in \mathcal{S}}{\text{max}}\underset{0\leq w_{i} \leq v^{\sf UCB}_{i,\ell}}{\text{max}}\; R(S,\mathbf{{w}}) \geq  R(S^*, \mb{v}).
\end{equation}
From Lemma \ref{lem:UCBv}, for each $\ell$ and $i \in \{ 1,\cdots,N\}$, we have that, 
 $$\ep{P}\left(v^{\sf UCB}_{i,\ell} < v_i\right) \leq \frac{6}{N\ell}.$$
Hence, from union bound, it follows that, 
\begin{equation}\label{eq:ran_2}
\ep{P}\left(\bigcap_{i=1}^N\left\{v^{\sf UCB}_{i,\ell} < v_i\right\}\right) \geq 1-\frac{6}{\ell}.
\end{equation}
Lemma \ref{lem:UCBS1} follows from \eqref{eq:ran_1} and \eqref{eq:ran_2}.  \hfill~$\Halmos$
\hfill $\halmos$}

%

\begin{lemma}[Bounding Regret]\label{lipschitz}
If $r_i \in [0,1]$ and $ 0\leq v_i \leq v^{\sf UCB}_{i,\ell}$ for all $i \in S_\ell$, then
\begin{equation*}
\tilde{R}_{\ell}(S_\ell) -R(S_\ell, \mb{v}) \leq \textstyle\frac{\textstyle\sum_{j \in S_\ell} \left(v^{\sf UCB}_{j,\ell} - v_j\right)}{1+\textstyle \sum_{j\in S_\ell}v_j}.
\end{equation*}
\end{lemma}
\proof{Proof.}
Since $1+\sum_{i \in S_\ell}v^{\sf UCB}_{i,\ell} \geq 1+\sum_{i \in S_\ell}v_{i,\ell}$, we have
\begin{equation*}
\begin{aligned}
\tilde{R}_{\ell}(S_\ell) -R(S_\ell, \mb{v}) &\leq \textstyle \frac{\textstyle\sum_{i \in S_\ell}r_i v^{\sf UCB}_{i,\ell}}{1+\textstyle\sum_{j \in S_\ell} v^{\sf UCB}_{j,\ell}} - \frac{\textstyle\sum_{i \in S_\ell}r_i v_i}{1+\textstyle\sum_{j \in S_\ell} v^{\sf UCB}_{j,\ell}},\\
&\leq  \frac{\textstyle \sum_{i \in S_\ell} \left(v^{\sf UCB}_{i,\ell} - v_i\right)}{1+\textstyle\sum_{j \in S_\ell} v^{\sf UCB}_{j,\ell}} \leq \frac{\textstyle \sum_{i \in S_\ell} \left(v^{\sf UCB}_{i,\ell} - v_i\right)}{1+\textstyle\sum_{j \in S_\ell} v_{j}}.
\end{aligned}
\end{equation*} 

\medskip \noindent {\bf {Proof of Lemma \ref{lem:UCBS2}:}}
{From Lemma \ref{lipschitz}, we have, 
\begin{equation}\label{eq:ran_3}
\left(1+\sum_{j \in S_\ell} v_j\right)\left(\tilde{R}_\ell(S_\ell) - R(S_\ell,\mb{v})\right) \leq \sum_{j \in S_\ell} \left(v^{\sf UCB}_{j,\ell} - v_j\right).
\end{equation}
From Lemma \ref{lem:UCBv}, we have that for each $i=1,\cdots, N$ and $\ell$, 
\begin{equation*}
\ep{P}\left(v^{\sf UCB}_{i,\ell} - v_i > C_1 \sqrt{\frac{v_i \log{(\sqrt{N}\ell+1)}}{T_i(\ell)}} + C_2\frac{\log{(\sqrt{N}\ell+1)}}{T_i(\ell)}\right) \leq \frac{7}{N\ell}.
\end{equation*}
Therefore, from union bound, it follows that, 
\begin{equation}\label{eq:ran_4}
\ep{P}\left(\bigcap_{i=1}^N\left\{v^{\sf UCB}_{i,\ell} - v_i < C_1 \sqrt{\frac{v_i \log{(\sqrt{N}\ell+1)}}{T_i(\ell)}} + C_2\frac{\log{(\sqrt{N}\ell+1)}}{T_i(\ell)}\right\}\right) \geq 1-\frac{7}{\ell}.
\end{equation}
Lemma \ref{lem:UCBS2} follows from \eqref{eq:ran_3} and \eqref{eq:ran_4}. 
}

\subsection{Putting it all together: Proof of Theorem \ref{main_result}}
\label{sec:completeProof}
In this section, we utilize the results established in the previous sections and complete the proof of Theorem \ref{main_result}.

Let $S^*$ denote the optimal assortment, our objective is to minimize the {\it regret} defined in \eqref{Regret}, which is same as
\begin{equation}\label{eq:regret_simple}
Reg_\pi(T,\mb{v})  = \mathbb{E}_{\pi}\left\{\sum_{\ell=1}^L |\ep{E}_\ell| \left(R(S^*, \mb{v}) - R(S_\ell,\mb{v})\right)\right\},
\end{equation} 
Note that $L$, $\ep{E}_\ell$ and $S_\ell$ are all random variables and the expectation in equation \eqref{eq:regret_simple} is over these random variables. 
Let $\ep{H}_\ell $ be the filtration (history) associated with the policy upto epoch $\ell$. In particular,   $$\ep{H}_\ell= \sigma(U, C_1, \cdots, C_{t(\ell)}, S_1, \cdots, S_{t(\ell)}), $$ where $t(\ell)$ is the time index corresponding to the end of epoch $\ell$.   The length of the $\ell^{th}$ epoch, $|\ep{E}_\ell|$ conditioned on $S_\ell$ is a geometric random variable with success probability defined as the probability of no-purchase in $S_\ell$, i.e.
\[\textstyle p_0(S_\ell) = \displaystyle \frac{1}{1+\sum_{j\in S_\ell} v_j}.\] 
Let $V(S_\ell) = \sum_{j\in S_\ell} v_j$, then we have $\mathbb{E}_{\pi}\left(|\ep{E}_\ell| \;{\Big | }\; S_\ell \right) = 1+V(S_\ell)$. Noting that $S_\ell$ in our policy is determined by $\ep{H}_{\ell-1}$, we have $\mathbb{E}_{\pi}\left(|\ep{E}_\ell| {\Big | } \ep{H}_{\ell-1} \right) = 1+V(S_\ell)$.   Therefore, by law of conditional expectations, we have
\begin{equation*}
Reg_\pi(T,\mb{v})  = \mathbb{E}_{\pi}\left\{\sum_{\ell=1}^L \mathbb{E}_{\pi}\left[|\ep{E}_\ell| \left(R(S^*, \mb{v}) - R(S_\ell,\mb{v})  \right)  \; {\Big |} \; \ep{H}_{\ell-1} \right]\right\},
\end{equation*}
and hence the regret can be reformulated as 
\begin{equation}\label{eq:ce_regret}
Reg_\pi(T,\mb{v}) =  \mathbb{E}_{\pi}\left\{\sum_{\ell=1}^L \left(1+V(S_\ell) \right)\left(R(S^*, \mb{v}) - R(S_\ell,\mb{v})  \right)\right\},
\end{equation}
the expectation in equation \eqref{eq:ce_regret} is over the random variables $L$ and $S_\ell$. For the sake of brevity, for each $\ell \in 1,\cdots,L,$ let 
\begin{equation}\label{eq:brevity_regret}
\Delta R_\ell {=} (1+V(S_\ell))\left(R(S^*,\mb{v})-R(S_\ell,\mb{v})\right). 
\end{equation}
Now the regret can be reformulated as 
\begin{equation}\label{eq:ce_regret_2}
\begin{aligned}
Reg_\pi(T,\mb{v}) &=  \mathbb{E}_{\pi}\left\{\sum_{\ell=1}^L \Delta R_\ell\right\}.\\
\end{aligned}
\end{equation}
 Let $T_i$ denote the total number of epochs that offered an assortment containing \mbox{product $i$}.  For all $\ell =1, \ldots,L$,  define events $\mathcal{A}_{\ell}$ as,
{\begin{equation*}
\mathcal{A}_{\ell} = \bigcup_{i=1}^N\left\{v^{\sf UCB}_{i,\ell} < v_i \;\text{or} \; v^{\sf UCB}_{i,\ell} > v_i + C_1 \sqrt{\frac{v_i\log{({\sqrt{N}}\ell+1)}}{T_i(\ell)}} + C_2\frac{\log{({\sqrt{N}}\ell+1)}}{T_i(\ell)} \right\}.
\end{equation*}
From union bound, it follows that
\begin{equation*}
\begin{aligned}
\mathbb{P}_\pi\left(\ep{A}_\ell\right) &\leq \sum_{i=1}^N \mathbb{P}_\pi\left(v^{\sf UCB}_{i,\ell} < v_i \;\text{or} \; v^{\sf UCB}_{i,\ell} > v_i + C_1 \sqrt{\frac{v_i\log{(\sqrt{N}\ell+1)}}{T_i(\ell)}} + C_2\frac{\log{(\sqrt{N}\ell+1)}}{T_i(\ell)}\right),\\
&\leq \sum_{i=1}^N \mathbb{P}_\pi\left(v^{\sf UCB}_{i,\ell}<v_i\right) + \mathbb{P}_\pi\left(v^{\sf UCB}_{i,\ell} > v_i + C_1\sqrt{\frac{v_i\log{(\sqrt{N}\ell+1)}}{T_i(\ell)}}+ C_2\frac{\log{(\sqrt{N}\ell+1)}}{T_i(\ell)}\right).
\end{aligned}
\end{equation*} 
Therefore, from Lemma \ref{lem:UCBv}, we have, 
\begin{equation}\label{eq:low_prob_event}
\mathbb{P}_\pi(\ep{A}_\ell) \leq \frac{13}{\ell}.
\end{equation}
Since $\mathcal{A}_\ell$ is a ``low probability'' event (see \eqref{eq:low_prob_event}), we analyze the regret in two scenarios, one when $\ep{A}_\ell$ is true and another when $\ep{A}^c_\ell$ is true.  We break down the regret in an epoch into the following two terms:
\begin{equation*}
\mathbb{E}_{\pi}\left(\Delta R_\ell\right)= E\left[\Delta R_\ell\cdot\mathbbm{1}(\mathcal{A}_{\ell-1}) + \Delta R_\ell\cdot\mathbbm{1}(\mathcal{A}^c_{\ell-1})\right].
\end{equation*}
}
Using the fact that $R(S^*,\mb{v})$ and $R(S_\ell,\mb{v})$ are both bounded by one and $V(S_\ell) \leq N$ in \eqref{eq:brevity_regret}, we have $\Delta R_\ell \leq N+1.$ Substituting the preceding inequality in the above equation, we obtain,
\begin{equation*}
\begin{aligned}
\mathbb{E}_{\pi}\left(\Delta R_\ell\right) \leq (N+1)\mathbb{P}_{\pi}(\mathcal{A}_{\ell-1}) +\mathbb{E}_{\pi}\left[ \Delta R_\ell\cdot\mathbbm{1}(\mathcal{A}^c_{\ell-1})\right].
\end{aligned}
\end{equation*}
Whenever $\mathbbm{1}(\mathcal{A}^c_{\ell-1}) = 1$, from Lemma \ref{UCB_bound}, we have $\tilde{R}_\ell(S^*) \geq R(S^*,\mb{v})$ and by our algorithm design, we have $\tilde{R}_\ell(S_\ell) \geq \tilde{R}_\ell(S^*)$ for all $\ell \geq 1$. Therefore, it follows that 
\begin{equation*}
\mathbb{E}_{\pi}\left\{\Delta R_\ell\right\} \leq (N+1)\mathbb{P}_{\pi}(\mathcal{A}_{\ell-1}) +\mathbb{E}_{\pi}\left\{\left[(1+V(S_\ell))(\tilde{R}_\ell(S_\ell) - R(S_\ell,\mb{v}))\right]\cdot \mathbbm{1}(\mathcal{A}^c_{\ell-1}) \right\}.
\end{equation*}
From the definition of the event, $\ep{A}_\ell$ and Lemma \ref{lipschitz}, it follows that, 
\begin{equation*}
\begin{aligned}
\left[(1+V(S_\ell))(\tilde{R}_\ell(S_\ell) - R(S_\ell,\mb{v}))\right]\cdot \mathbbm{1}(\mathcal{A}^c_{\ell-1}) \leq \sum_{i\in S_\ell}\left(C_1\sqrt{\frac{v_i\log{({\sqrt{N}\ell}+1)}}{T_i(\ell)}} + \frac{C_2\log{({\sqrt{N}\ell}+1)}}{T_i(\ell)}\right).\;
\end{aligned}
\end{equation*}

Therefore, we have
\begin{equation}\label{eq:stoch_dom}
\begin{aligned}
\mathbb{E}_{\pi}\left\{\Delta R_\ell\right\} &\leq (N+1)\mathbb{P}_{\pi}\left(\mathcal{A}_{\ell-1}\right) + C \sum_{i\in S_\ell}\mathbb{E}_{\pi}\left(\sqrt{\frac{v_i\log{{\sqrt{N}}T}}{T_i(\ell)}} + \frac{\log{{\sqrt{N}}T}}{T_i(\ell)}\right),\\
\end{aligned}
\end{equation}
where $C = \max\{C_1,C_2\}$. Combining equations \eqref{eq:ce_regret} and \eqref{eq:stoch_dom}, we have 
\begin{equation*}
Reg_\pi(T,\mb{v}) \leq \mathbb{E}_{\pi}\left\{\sum_{\ell=1}^L \left[(N+1)\mathbb{P}_{\pi}\left(\mathcal{A}_{\ell-1}\right) + C\sum_{i\in S_\ell}\left(\sqrt{\frac{v_i\log{{\sqrt{N}}T}}{T_i(\ell)}} + \frac{\log{{\sqrt{N}}T}}{T_i(\ell)}\right)\right] \right\}.
\end{equation*}
Therefore, from Lemma \ref{lem:UCBv}, we have
\begin{equation}\label{eq:regret_bound_first_step}
\begin{aligned}
Reg_\pi(T,\mb{v}) & \leq C\mathbb{E}_{\pi}\left\{\sum_{\ell=1}^L \frac{N+1}{\ell} +\sum_{i\in S_\ell}\sqrt{\frac{v_i\log{{\sqrt{N}}T}}{T_i(\ell)}} + \sum_{i\in S_\ell}\frac{\log{{\sqrt{N}}T}}{T_i(\ell)}\right\},\\
& \overset{(a)}{\le} CN\log{T}+ CN\log^2{{\sqrt{N}}T} + C\mathbb{E}_{\pi}\left(\sum_{i=1}^n  \sqrt{v_iT_i\log{{\sqrt{N}}T}} \right), \\
& \overset{(b)}{\le} CN\log{T} + CN\log^2{{N}T} + C\sum_{i=1}^N  \sqrt{v_i\log{({N}T)}\mathbb{E}_{\pi}(T_i)}.
\end{aligned}
\end{equation}
Inequality (a) follows from the observation that $L \leq T$, $T_i \leq T$, $$\displaystyle \sum_{T_i(\ell)=1}^{T_i} \frac{1}{\sqrt{T_i(\ell)}} \leq \sqrt{T_i},\;\text{and}\; \displaystyle \sum_{T_i(\ell)=1}^{T_i} \frac{1}{{T_i(\ell)}} \leq \log{T_i},$$ while Inequality (b) follows from Jensen's inequality. 

\medskip \noindent For any realization of $L$, $\ep{E}_\ell$, $T_i$, and $S_\ell$ in Algorithm \ref{learn_algo}, we have the following relation $$\sum_{\ell=1}^L n_\ell  \leq T.$$ Hence, we have $\mathbb{E}_{\pi}\left(\sum_{\ell=1}^L n_\ell\right)  \leq T.$
Let $\ep{F}$ denote the filtration corresponding to the offered assortments $S_1,\cdots,S_L$, then by law of total expectation, we have, 
\begin{equation*}
\begin{aligned}
\mathbb{E}_{\pi}\left(\sum_{\ell=1}^L n_\ell\right) &= \mathbb{E}_{\pi}\left\{\sum_{\ell=1}^L E_{\mathcal{F}}\left( n_\ell\right)\right\}= \mathbb{E}_{\pi}\left\{\sum_{\ell=1}^L1+\sum_{i\in S_\ell} v_i\right\} ,\\
&=  \mathbb{E}_{\pi}\left\{L+\sum_{i=1}^n  v_i T_i\right\} = \mathbb{E}_{\pi}\{L\}+\sum_{i=1}^n  v_i \mathbb{E}_{\pi}(T_i).
\end{aligned}
\end{equation*}
Therefore, it follows that  
\begin{equation}\label{eq:conditional_ineq}
\sum v_i\mathbb{E}_{\pi}(T_i) \leq T.
\end{equation}
To obtain the worst case upper bound, we maximize the bound in equation \eqref{eq:regret_bound_first_step} subject to the condition \eqref{eq:conditional_ineq} and hence, we have $Reg_\pi(T,\mb{v})  =   O( \sqrt{NT\log{{N}T}}   + N\log^2{{N}T}).$ \hfill $\Halmos$

\subsection{Improved regret bounds for the unconstrained MNL-Bandit}\label{sec:unconstrained_mnl_bandit}
\noindent \change{Here, we focus on the special case of the unconstrained \banditMNL\;problem and use the analysis of Appendix \ref{sec:completeProof} to establish a tighter bound on the regret for Algorithm \ref{learn_algo}. First, we note that, in the case of the unconstrained problem, for any epoch $\ell$, with high probability, the assortment, $S_\ell$ suggested by Algorithm \ref{learn_algo} is a subset of the optimal assortment, $S^*.$ More specifically, the following holds. 
\begin{lemma}\label{unconstrained_subset}
Let $S^* = \underset{S \in \{1,\cdots,N\}}{\text{argmax}} \;R(S,\mathbf{v})$ and $S_\ell$ be the assortment suggested by Algorithm \ref{learn_algo}. Then for any $\ell=1,\cdots,L$, we have, 
$$\mathbb{P}_\pi\left(S_\ell \subset S^*\right) \geq 1-\frac{6}{\ell}.$$
\end{lemma} 
\proof{Proof.}If there exists a product $i$, such that $r_i \geq R(S^*,\mb{v})$, then following the proof of Lemma \ref{UCB_bound}, we can show that $R(S^*\cup i,\mb{v}) \geq R(S^*,\mb{v})$ and similarly, if there exists a product $i$, such that $r_i < R(S^*,\mb{v})$, we can show that $R(S^*\backslash \{i\},\mb{v}) \geq R(S^*,\mb{v}).$ Since there are no constraints on the set of feasible assortment, we can add and remove products that will improve the expected revenue. Therefore, we have, 
\begin{equation}\label{optimal_threshold_revenue}
i \in S^* \;\text{if and only if} \; r_i \geq  R(S^*,\mb{v}).
\end{equation}
Fix an epoch $\ell$, let $S_\ell$ be the assortment suggested by Algorithm \ref{learn_algo}.  Using similar arguments as above, we can show that, 
\begin{equation}\label{ucb_threshold_revenue}
i \in S_\ell \;\text{if and only if} \; r_i \geq R(S_\ell,\mb{v}^{\sf UCB}_\ell).
\end{equation}
From Lemma \ref{lem:UCBS1}, we have , 
\begin{equation}\label{eq:final_ucb_bound}
\mathbb{P}_\pi\left(R(S_\ell,\mb{v}^{\sf UCB}_\ell) \geq R(S^*,\mb{v})\right) \geq 1-\frac{6}{\ell}.
\end{equation}
Lemma \ref{unconstrained_subset} follows from \eqref{optimal_threshold_revenue}, \eqref{ucb_threshold_revenue} and \eqref{eq:final_ucb_bound}.\hfill $\halmos$\\
From Lemma \ref{unconstrained_subset}, it follows that Algorithm \ref{learn_algo} only considers products from the set $S^*$ with high probability, and hence, we can follow the proof in Appendix \ref{sec:completeProof} (by replacing $N$ with $|S^*|$) to derive sharper regret bounds. In particular, we have the following result,
\begin{corollary}[Performance Bounds for unconstrained case]\label{unconstrained_bound}
For any instance, $\mb{v} = (v_0, \ldots, v_N)$ of the \banditMNL~problem with $N$ products and no constraints,  $r_i\in [0,1]$ and $v_0\ge v_i$ for $i=1,\ldots, N$, there exists finite constants $C_1$ and $C_2,$ such that the regret of the policy defined in Algorithm~\ref{learn_algo} at any time $T$ is bounded as,
\begin{equation*} 
Reg_\pi(T,\mb{v}) \leq C_1\sqrt{|S^*|T\log{NT}} + C_2N\log{NT}.
\end{equation*}
\end{corollary}
 }

\section{Proof of Theorem \ref{main_result_extn}}\label{sec:proof_main_result_extn}
The proof for Theorem \ref{main_result_extn} is very similar to the proof of Theorem \ref{main_result}. Specifically, we first prove that the initial exploratory phase is indeed bounded and then follow the proof of Theorem \ref{main_result} to establish the correctness of confidence intervals, optimistic assortment and finally deriving the convergence rates and regret bounds.

\medskip \noindent{\bf Bounding Exploratory Epochs.} We would denote an epoch $\ell$ as an ``exploratory epoch'' if the assortment offered in the epoch contains a product that has been offered in less than $48 \log{({\sqrt{N}}\ell+1)}$ epochs. It is easy to see that the number of exploratory epochs is bounded by $48 N \log{{N}T}$, where $T$ is the selling horizon under consideration. We then use the observation that the length of any epoch is a geometric random variable to bound the total expected duration of the exploration phase. Hence, we bound the expected regret due to explorations.
\begin{lemma}\label{exploration_bound}
Let $L$ be the total number of epochs in Algorithm \ref{learn_algo_extn} and let $\mathcal{E}_L$ denote the set of ``exploratory epochs,'' i.e.
$$E_L = \left \{ \ell \; \left |\; \exists \; i \in S_\ell \;\; \text{such that} \;\; T_i(\ell) < 48\log{({\sqrt{N}}\ell+1)}\right. \right \},$$ where $T_i(\ell)$ is the number of epochs product $i$ has been offered before epoch $\ell$.  If $\mathcal{E}_\ell$ denote the time indices corresponding to epoch $\ell$ and $v_i \leq B v_0$ for all $i=1,\ldots, N$, for some $B \geq 1$, then we have that,
$$\mathbb{E}_{\pi}\left( \sum_{\ell \in E_L} |\mathcal{E}_\ell|  \right)< 49NB\log{{N}T},$$ where the expectation is over all possible outcomes of Algorithm \ref{learn_algo_extn}.
\end{lemma} 
\proof{Proof.}
Consider an $\ell \in E_L$, note that $|\mathcal{E}_\ell|$ is a geometric random variable with parameter ${1}/{V(S_\ell)+1}$. Since $v_i \leq B v_0$, for all $i$ and we can assume without loss of generality $v_0 = 1$, we have $|\mathcal{E}_\ell|$ as a geometric random variable with parameter $p$, where $p \geq {1}/{\left(B|S_\ell|+1\right)}$. Therefore, we have the conditional expectation of $|\ep{E}_\ell|$ given that assortment $S_\ell$ is offered is bounded as, 
\begin{equation}\label{eq:conditional_exploratory}
\mathbb{E}_{\pi}\left(|\mathcal{E}_\ell| \;\;|\;\; S_\ell\right) \leq B|S_\ell| + 1.
\end{equation}
Note that after every product has been offered in at least $48 \log{{N}T}$ epochs, then we do not have any exploratory epochs. Therefore, we have that 
$$\sum_{\ell \in E_L} |S_\ell| \leq 48 N\log{{N}T}.$$
Substituting the above inequality in \eqref{eq:conditional_exploratory}, we obtain $$\mathbb{E}_{\pi}\left( \sum_{\ell \in E_L} |\mathcal{E}_\ell|  \right) \leq 48BN\log{NT} + 48N\log{NT}. \hfill \halmos$$

\medskip \noindent{\bf Confidence Intervals.} We will now show a result analogous to Lemma \ref{lem:UCBv}, that establish the updates in Algorithm \ref{learn_algo_extn}, $v^{\sf UCB2}_{i,\ell}$, as upper confidence bounds converging to actual parameters $v_i$. Specifically, we have the following result. 

\begin{lemma}\label{lem:UCBv_extn}
For every epoch $\ell$, if $T_i(\ell) \geq 48\log{({\sqrt{N}}\ell+1)}$ for all $i \in S_\ell$, then we have,
\begin{enumerate}
\item $v^{\sf UCB2}_{i,\ell}\geq v_i$ with probability at least $1-\frac{6}{{N}\ell}$ for all $i=1,\cdots,N$. 
\item There exists constants $C_1$ and $C_2$ such that \[\hspace{-25mm}\displaystyle v^{\sf UCB2}_{i,\ell} - v_i \leq C_1\max{\left\{\sqrt{{v}_{i}}, {v}_{i}\right\}}\sqrt{\frac{\log{({\sqrt{N}}\ell+1)}}{T_i(\ell)}} + C_2\frac{ \log{({\sqrt{N}}\ell+1)}}{T_i(\ell)},\]
with probability at least $1-\frac{7}{{N}\ell}.$
\end{enumerate}
\end{lemma}
The  proof is very similar to the proof of Lemma \ref{lem:UCBv}, where we first establish the following concentration inequality for the estimates $\hat{v}_{i,\ell}$, when $T_i(\ell) \geq 48 \log{({\sqrt{N}}\ell+1)}$ from which the above result follows. The proof of Lemma  \ref{chernoff_hoeffding_ineq_extn} is provided in Appendix \ref{proof:multiplicative_chernoff}. 
\begin{lemma}\label{chernoff_hoeffding_ineq_extn}
If in epoch $\ell$, $T_i(\ell) \geq 48\log{({\sqrt{N}}\ell+1)}$ for all $i \in S_\ell$, then we have the following concentration bounds
\begin{enumerate}
\item $\displaystyle \mathbb{P}_{\pi}\left(\left|\bar{v}_{i,\ell} - v_i \right | \geq \max{\left\{\sqrt{\bar{v}_{i,\ell}}, \bar{v}_{i,\ell}\right\}}\sqrt{\frac{48\log{({\sqrt{N}}\ell+1)}}{n}} + \frac{48 \log{({\sqrt{N}}\ell+1)}}{n}\right) \leq \frac{6}{{N}\ell}$.
\item $\displaystyle \mathbb{P}_{\pi}\left(\left|\bar{v}_{i,\ell} - v_i \right | \geq \max{\left\{\sqrt{v_i}, v_i\right\}}\sqrt{\frac{24\log{({\sqrt{N}}\ell+1)}}{n}} + \frac{48 \log{({\sqrt{N}}\ell+1)}}{n}\;\right)  \leq \frac{4}{{N}\ell}$.
\item  {$\displaystyle \mathbb{P}_{\pi}\left(\bar{v}_{i,\ell}   > \frac{3v_i}{2} + \frac{48 \log{({\sqrt{N}}\ell+1)}}{T_i(\ell)}\right)   \leq  \frac{3}{{N}\ell}\;\; .$}
\end{enumerate}
\end{lemma}

\medskip \noindent {\bf {Proof of Lemma \ref{lem:UCBv_extn}}:}
{By design of Algorithm \ref{learn_algo_extn}, we have, 
\begin{equation}\label{eq:ucb_defineB}
v^{\sf UCB2}_{i,\ell} = \bar{v}_{i,\ell} + \max{\left\{\sqrt{\bar{v}_{i,\ell}}, \bar{v}_{i,\ell}\right\}}\sqrt{\frac{48\log{({\sqrt{N}}\ell+1)}}{T_i(\ell)}} + \frac{48 \log{({\sqrt{N}}\ell+1)}}{T_i(\ell)}.	
\end{equation}
Therefore from Lemma \ref{chernoff_hoeffding_ineq_extn}, we have \begin{equation}\label{eq:ucb_ineqB}
\mathbb{P}_\pi\left(v^{\sf UCB2}_{i,\ell} < v_i\right) \leq \frac{6}{N\ell}.\end{equation}
The first inequality in Lemma \ref{lem:UCBv} follows from \eqref{eq:ucb_ineqB}. 
From \eqref{eq:ucb_defineB}, we have, 
\begin{equation}\label{eq:a4B}
\begin{aligned}
\left|v^{\sf UCB2}_{i,\ell} - v_i\right| &\leq \left|v^{\sf UCB}_{i,\ell} - \bar{v}_{i,\ell}\right| + \left|\bar{v}_{i,\ell} - v_i\right| \\
&= \max{\left\{\sqrt{\bar{v}_{i,\ell}}, \bar{v}_{i,\ell}\right\}}\sqrt{48\displaystyle \frac{\log{(\sqrt{N}\ell+1)}}{{T}_{i}(\ell)}} + \frac{48\log{(\sqrt{N}\ell+1)}}{T_i(\ell)}+\left|\bar{v}_{i,\ell} - v_i\right|. 
\end{aligned}
\end{equation}
From Lemma \ref{chernoff_hoeffding_ineq_extn}, we have 
\begin{equation*}
\displaystyle \mathbb{P}_{\pi}\left(\bar{v}_{i,\ell}  > \frac{3v_i}{2} + \frac{48\log{(\sqrt{N}\ell+1)}}{{T}_i(\ell)}\;\right)  \leq \frac{3}{N\ell},
\end{equation*}
which implies
\begin{equation*}
\displaystyle \mathbb{P}_{\pi}\left(48\bar{v}_{i,\ell}\frac{\log{(\sqrt{N}\ell+1)}}{{T}_{i}(\ell)}  > 72v_i\frac{\log{(\sqrt{N}\ell+1)}}{{T}_{i}(\ell)} + \left(\frac{48\log{(\sqrt{N}\ell+1)}}{{T}_i(\ell)}\right)^2\;\right)  \leq \frac{3}{N\ell},
\end{equation*}
Using the fact that $\sqrt{a+b} < \sqrt{a}+\sqrt{b}$, for any positive numbers  $a,b$, we have,
\begin{equation}\label{eq:a7B}
\hspace{-5mm}\mathbb{P}_{\pi}\left(\max{\left\{\sqrt{\bar{v}_{i,\ell}}, \bar{v}_{i,\ell}\right\}}\sqrt{48\bar{v}_{i,\ell} \frac{\log{(\sqrt{N}\ell+1)}}{{T}_{i}(\ell)}}  > \max{\left\{\sqrt{{v}_{i}}, {v}_{i}\right\}}\sqrt{72\frac{\log{(\sqrt{N}\ell+1)}}{{T}_{i}(\ell)}} + \frac{48\log{(\sqrt{N}\ell+1)}}{{T}_i(\ell)}\;\right)  \leq \frac{3}{N\ell},
\end{equation}
From Lemma \ref{chernoff_hoeffding_ineq_extn}, we have,
\begin{equation}\label{eq:a8B}
\mathbb{P}_{\pi}\left(\left|\bar{v}_{i,\ell} - v_i \right | > \max{\left\{\sqrt{{v}_{i}}, {v}_{i}\right\}}\sqrt{24\displaystyle \frac{\log{(\sqrt{N}\ell+1)}}{{T}_{i}(\ell)}} + \frac{48\log{(\sqrt{N}\ell+1)}}{{T}_i(\ell)}\;\right)  \leq  \frac{4}{N\ell}.
\end{equation}
From \eqref{eq:a4B} and applying union bound on \eqref{eq:a7B} and \eqref{eq:a8B}, we obtain,
$$\mathbb{P}_\pi\left(\left|v^{\sf UCB2}_{i,\ell} - v_i\right| > (\sqrt{72}+\sqrt{24})\max{\left\{\sqrt{v_{i}}, v_{i}\right\}}\sqrt{\frac{v_i \log{(\sqrt{N}\ell + 1)}}{T_i(\ell)}} + \frac{144\log{(\sqrt{N}\ell + 1)}}{T_i(\ell)} \right) \leq \frac{7}{N\ell}.$$
Lemma \ref{lem:UCBv_extn} follows from the above inequality and \eqref{eq:ucb_ineqB}. \hfill $\halmos$}

\medskip \noindent {\bf Optimistic Estimate and Convergence Rates}: 
We will now establish two results analogous to Lemma \ref{lem:UCBS1} and \ref{lem:UCBS2}, that show that the estimated revenue converges to the optimal expected revenue from above and also specify the convergence rate. In particular, we have the following two results. {The proofs of Lemma \ref{lem:UCBS1_extn} and \ref{lem:UCBS2_extn} follow similar arguments to the proofs of Lemma \ref{lem:UCBS1} and \ref{lem:UCBS2} respectively and we skip the proofs in interest of avoiding redundancy. }
\begin{lemma}\label{lem:UCBS1_extn}
 Suppose $S^* \in {\cal S}$ is the assortment with highest expected revenue, and Algorithm \ref{learn_algo_extn} offers $S_\ell \in {\cal S}$ in each epoch $\ell$. Further,  if $T_i(\ell) \geq 48\log{({\sqrt{N}}\ell+1)}$ for all $i \in S_\ell$, then we have, 
\[\tilde{R}_\ell(S_\ell) \geq \tilde{R}_\ell(S^*) \geq R(S^*,\mb{v}) \; \text{with probability at least}\; \;1-\frac{6}{{N}\ell}.\]
\end{lemma}

\begin{lemma}\label{lem:UCBS2_extn}
For every epoch $\ell$, if $r_i\in [0,1]$ and $T_i(\ell) \geq 48\log{({\sqrt{N}}\ell+1)}$ for all $i \in S_\ell$, then there exists constants $C_1$ and $C_2$ such that for every $\ell $, we have
\[\textstyle (1+\sum_{j \in S_\ell} v_{j})(\tilde{R}_{\ell}(S_\ell) - R(S_\ell,\mb{v})) \leq C_1\max{\left\{\sqrt{{v}_{i}}, {v}_{i}\right\}}\sqrt{\frac{\log{({\sqrt{N}}\ell+1)}}{|\mathcal{T}_i(\ell)|}} + C_2\frac{\log{({\sqrt{N}}\ell+1)}}{|\mathcal{T}_i(\ell)|}, \]
\text{with probability at least} $1-\frac{13}{{N}\ell}.$
\end{lemma}

\subsection{Putting it all together: Proof of Theorem \ref{main_result_extn}}
{Proof of Theorem \ref{main_result_extn} is very similar to the proof of Theorem \ref{main_result}. We use the key results discussed above instead of similar results in Section \ref{sec:regretAnalysis} to complete the proof. Note that $E_\ell$ is the set of ``exploratory epochs,'' i.e. epochs in which at least one of the offered product is offered less than the required number of times. We breakdown the regret as follows: 
\begin{equation*}
Reg_\pi(T,\mb{v})   = \underbrace{\mathbb{E}_{\pi}\left\{\sum_{\ell\in E_L} |\ep{E}_\ell| \left(R(S^*,\mb{v}) - R(S_\ell,\mb{v})\right)\right\}}_{Reg_1(T,\mb{v})} + \underbrace{\mathbb{E}_{\pi}\left\{\sum_{\ell \not \in E_L}|\ep{E}_\ell|\left( R(S^*,\mb{v}) - R(S_\ell,\mb{v})\right)\right\}}_{Reg_2(T,\mb{v})}.
\end{equation*} 
Since for any $S$, we have, $R(S,\mb{v}) \leq R(S^*,\mb{v}) \leq 1$, it follows that,
\begin{equation*}
Reg_\pi(T,\mb{v})   \leq \mathbb{E}_{\pi}\left\{\sum_{\ell\in E_L} |\ep{E}_\ell| \right\} + Reg_2(T,\mb{v}).
\end{equation*} 
From Lemma \ref{exploration_bound}, it follows that, 
\begin{equation}\label{eq:decompose_ineq}
Reg_\pi(T,\mb{v})   \leq 49NB \log{NT} + Reg_2(T,\mb{v}).
\end{equation} 
We will focus on the second term in the above equation, $Reg_2(T,\mb{v})$. Following the analysis in Appendix \ref{sec:completeProof}, we can show that,
\begin{equation}
{Reg}_2(T,\mb{v}) =  \mathbb{E}_{\pi}\left\{\sum_{\ell\not \in E_L} \left(1+V(S_\ell) \right)\left(R(S^*, \mb{v}) - R(S_\ell,\mb{v})  \right)\right\}.\end{equation} 
Similar to the analysis in Appendix \ref{sec:completeProof}, for the sake of brevity, we define, 
\begin{equation}\label{eq:brevity_regret_extn}
{\Delta R_\ell} {=} (1+V(S_\ell))\left(R(S^*,\mb{v})-R(S_\ell,\mb{v})\right). 
\end{equation}
Now, $Reg_2(T,\mb{v})$ can be reformulated as 
\begin{equation}\label{eq:ce_regret_2_extn}
\begin{aligned}
{Reg}_2(T,\mb{v}) &=  \mathbb{E}_{\pi}\left\{\sum_{\ell\not\in E_L} {\Delta R_\ell}\right\}.\\
\end{aligned}
\end{equation}
Let $T_i$ denote the total number of epochs that offered an assortment containing \mbox{product $i$}.  For all $\ell =1, \ldots,L$,  define events $\mathcal{B}_{\ell}$ as,
\begin{equation*}
\mathcal{B}_{\ell} = \bigcup_{i=1}^N\left\{v^{\sf UCB2}_{i,\ell} < v_i \;\text{or} \; v^{\sf UCB2}_{i,\ell} > v_i + C_1 \max{\left\{\sqrt{{v}_{i}}, {v}_{i}\right\}}\sqrt{\frac{\log{({\sqrt{N}}\ell+1)}}{T_i(\ell)}} + C_2\frac{\log{({\sqrt{N}}\ell+1)}}{T_i(\ell)} \right\}.
\end{equation*}
From union bound, it follows that
\begin{equation*}
\begin{aligned}
\mathbb{P}_\pi\left(\ep{B}_\ell\right) &\leq \sum_{i=1}^N \mathbb{P}_\pi\left(v^{\sf UCB2}_{i,\ell} < v_i \;\text{or} \; v^{\sf UCB2}_{i,\ell} > v_i + C_1 \max{\left\{\sqrt{{v}_{i}}, {v}_{i}\right\}}\sqrt{\frac{\log{(\sqrt{N}\ell+1)}}{T_i(\ell)}} + C_2\frac{\log{(\sqrt{N}\ell+1)}}{T_i(\ell)}\right),\\
&\leq \sum_{i=1}^N \mathbb{P}_\pi\left(v^{\sf UCB2}_{i,\ell}<v_i\right) + \mathbb{P}_\pi\left(v^{\sf UCB2}_{i,\ell} > v_i + C_1 \max{\left\{\sqrt{{v}_{i}}, {v}_{i}\right\}}\sqrt{\frac{\log{(\sqrt{N}\ell+1)}}{T_i(\ell)}} + C_2\frac{\log{(\sqrt{N}\ell+1)}}{T_i(\ell)}\right).
\end{aligned}
\end{equation*} 
Therefore, from Lemma \ref{lem:UCBv_extn}, we have, 
\begin{equation}\label{eq:low_prob_event_extn}
\mathbb{P}_\pi(\ep{B}_\ell) \leq \frac{13}{\ell}.
\end{equation}
Since $\mathcal{B}_\ell$ is a ``low probability'' event (see \eqref{eq:low_prob_event_extn}), we analyze the regret in two scenarios: one when $\ep{B}_\ell$ is true and another when $\ep{B}^c_\ell$ is true.  We break down the regret in an epoch into the following two terms. 
\begin{equation*}
\mathbb{E}_{\pi}\left({\Delta R_\ell}\right)= E\left[{\Delta R_\ell}\cdot\mathbbm{1}(\mathcal{B}_{\ell-1}) + {\Delta R_\ell}\cdot\mathbbm{1}(\mathcal{B}^c_{\ell-1}).\right]
\end{equation*}
Using the fact that $R(S^*,\mb{v})$ and $R(S_\ell,\mb{v})$ are both bounded by one and $V(S_\ell) \leq BN$ in \eqref{eq:brevity_regret_extn}, we have $\Delta R_\ell \leq N+1.$ Substituting the preceding inequality in the above equation, we obtain,
\begin{equation*}
\begin{aligned}
\mathbb{E}_{\pi}\left({\Delta R_\ell}\right) \leq B(N+1)\mathbb{P}_{\pi}(\mathcal{B}_{\ell-1}) +\mathbb{E}_{\pi}\left[ {\Delta R_\ell}\cdot\mathbbm{1}(\mathcal{B}^c_{\ell-1})\right].
\end{aligned}
\end{equation*}
Whenever $\mathbbm{1}(\mathcal{B}^c_{\ell-1}) = 1$, from Lemma \ref{UCB_bound}, we have $\tilde{R}_\ell(S^*) \geq R(S^*,\mb{v})$ and by our algorithm design, we have $\tilde{R}_\ell(S_\ell) \geq \tilde{R}_\ell(S^*)$ for all $\ell \geq 1$. Therefore, it follows that 
\begin{equation}\label{eq:randbvbv_b12}
\mathbb{E}_{\pi}\left\{{\Delta R_\ell}\right\} \leq B(N+1)\mathbb{P}_{\pi}(\mathcal{B}_{\ell-1}) +\mathbb{E}_{\pi}\left\{\left[(1+V(S_\ell))(\tilde{R}_\ell(S_\ell) - R(S_\ell,\mb{v}))\right]\cdot \mathbbm{1}(\mathcal{B}^c_{\ell-1}) \right\}.
\end{equation}
From the definition of the event, $\ep{B}_\ell$ and Lemma \ref{lem:UCBS2_extn}, we have,
\begin{equation*}
\begin{aligned}
\left[(1+V(S_\ell))(\tilde{R}_\ell(S_\ell) - R(S_\ell,\mb{v}))\right]\cdot \mathbbm{1}(\mathcal{B}^c_{\ell-1}) \leq \sum_{i\in S_\ell}\left(C_1\max\{v_i,\sqrt{v_i}\}\sqrt{\frac{\log{({\sqrt{N}\ell}+1)}}{T_i(\ell)}} + \frac{C_2\log{({\sqrt{N}\ell}+1)}}{T_i(\ell)}\right),\;
\end{aligned}
\end{equation*}
and therefore, substituting above inequality in \eqref{eq:randbvbv_b12}, we have
\begin{equation}\label{eq:stoch_domB}
\begin{aligned}
\mathbb{E}_{\pi}\left\{{\Delta R_\ell}\right\} &\leq B(N+1)\mathbb{P}_{\pi}\left(\mathcal{B}_{\ell-1}\right) + C \sum_{i\in S_\ell}\mathbb{E}_{\pi}\left(\max\{v_i,\sqrt{v_i}\}\sqrt{\frac{\log{{\sqrt{N}}T}}{T_i(\ell)}} + \frac{\log{{\sqrt{N}}T}}{T_i(\ell)}\right),\\
\end{aligned}
\end{equation}
where $C = \max\{C_1,C_2\}$. Combining equations \eqref{eq:decompose_ineq}, \eqref{eq:ce_regret_2_extn} and \eqref{eq:stoch_domB}, we have 
\begin{equation*}
\begin{aligned}
\hspace{-10mm}Reg_\pi(T,\mb{v}) &\leq 49BN\log{NT} + \mathbb{E}_{\pi}\left\{\sum_{\ell=1}^L B(N+1)\mathbb{P}_{\pi}\left(\mathcal{A}_{\ell-1}\right) \right\}\\
& + \sum_{\ell=1}^L \mathbb{E}_{\pi}\left[C\max\{v_i, \sqrt{v}_i\}\sum_{i\in S_\ell}\left(\sqrt{\frac{\log{{\sqrt{N}}T}}{T_i(\ell)}} + \frac{\log{{\sqrt{N}}T}}{T_i(\ell)}\right)\right].
\end{aligned}
\end{equation*}}
Define sets $\ep{I} = \{i | v_i \geq 1\}$ and $\ep{D} = \{i | v_i < 1\}$. Therefore, we have, 
\begin{equation}\label{eq:regret_bound_first_step_extn}
\begin{aligned}
Reg_\pi(T,\mb{v}) & \leq 98NB\log{NT} + C\mathbb{E}_{\pi}\left\{\sum_{\ell =1}^L  \sum_{i\in S_\ell}\left(\max{\left\{\sqrt{{v}_{i}}, {v}_{i}\right\}}\sqrt{\frac{\log{\sqrt{N}T}}{T_i(\ell)}} + \frac{\log{\sqrt{N}T}}{T_i(\ell)}\right) \right\},\\
& \overset{(a)}{\le} 98NB\log{NT}+ CN\log^2{NT} + C\mathbb{E}_{\pi}\left(\sum_{i\in \ep{D}}  \sqrt{v_iT_i\log{NT}} +  \sum_{i\in \ep{I}}  v_i\sqrt{T_i\log{NT}}\right), \\
& \overset{(b)}{\le} 98NB\log{NT}+ CN\log^2{NT} + C\sum_{i\in \ep{D}}  \sqrt{v_i\mathbb{E}_{\pi}(T_i)\log{NT}} +  \sum_{i\in \ep{I}}  v_i\sqrt{\mathbb{E}_{\pi}(T_i)\log{NT}},
\end{aligned}
\end{equation}
inequality (a) follows from the observation that $\sqrt{N} \leq N$,$L \leq T$, $T_i \leq T$, $$\sum_{T_i(\ell)=1}^{T_i} \frac{1}{\sqrt{T_i(\ell)}} \leq \sqrt{T_i}\;\; \text{and} \sum_{T_i(\ell)=1}^{T_i} \frac{1}{{T_i(\ell)}} \leq \log{T_i},$$ while inequality (b) follows from Jensen's inequality. From \eqref{eq:conditional_ineq}, we have that,
\begin{equation*}
\sum v_i\mathbb{E}_{\pi}(T_i) \leq T.
\end{equation*}
To obtain the worst case upper bound, we maximize the bound in equation \eqref{eq:regret_bound_first_step_extn} subject to the above constraint. Noting that the objective in \eqref{eq:regret_bound_first_step_extn} is concave, we use the KKT conditions to derive the worst case bound as $Reg_\pi(T,\mb{v})  =   O( \sqrt{BNT\log{{N}T}}   + N\log^2{{N}T} + BN\log{{N}T}).$ \hfill $\Halmos$

\section{Improved regret bounds for ``well separated'' instances} \label{parametric_bounds}
\medskip \noindent {\bf \change{Proof of Lemma \ref{sensitivity_analysis}}}: 
{Let $V(S_\ell) = \sum_{i \in S_\ell} v_i.$} From Lemma \ref{lem:UCBS2}, {and definition of $\tau$ (see \eqref{eq:tau})}, we have, 
\begin{equation}\label{eq:6.2}
\begin{aligned}
 R(S^*, \mb{v}) - R(S_\ell, \mb{v}) & \leq \frac{1}{V(S_\ell) + 1} \sum_{i \in S_\ell}\left( C_1 \sqrt{\frac{v_i\log{({\sqrt{N}}\ell+1)}}{T_i(\ell)}} + C_2\frac{\log{({\sqrt{N}}\ell+1)}}{T_i(\ell)} \right),\\
&  { \leq  \Delta(\mb{v}) \left(\frac{C_1\sum_{i \in S_\ell} \sqrt{v_i}}{2\sqrt{NC}\left(V(S_\ell) + 1\right)} + \frac{C_2}{4C} \right).}
\end{aligned}
\end{equation}
From Cauchy-Schwartz inequality, we have {$$\sum_{i \in S_\ell} \sqrt{v_i} \leq \sqrt{|S_\ell| \sum_{i \in S_\ell} v_i} \leq \sqrt{N V(S_\ell)} \leq \sqrt{N} \left(V(S_\ell) + 1\right).$$ Substituting the above inequality in \eqref{eq:6.2} and using the fact that $C = \max\{{C^2_1},C_2\}$, we obtain $R(S^*, \mb{v}) - R(S_\ell, \mb{v}) \leq \frac{3\Delta(\mb{v})}{4}.$ } The result follows from the definition of $\Delta(\mb{v})$.
 \halmos

\medskip \noindent {\bf {Proof of Lemma \ref{no_sub_opt}:}} We complete the proof using an inductive argument on $N$. 

\medskip \noindent Lemma \ref{no_sub_opt} trivially holds for $N=1$, since when there is only one product, every epoch offers the optimal product and the number of epochs offering sub-optimal assortment is $0$, which is less than $\tau$. Now assume that for any $N \leq M$, we have that the  number of ``good epochs'' offering sub-optimal products is bounded by $N\tau,$ where $\tau$ is as defined in \eqref{eq:tau}.  

\medskip \noindent Now consider the setting, $N = M+1$. We will now show that the number of ``good epochs'' offering sub-optimal products cannot be more than $(M+1)\tau$ to complete the induction argument. We introduce some notation, let $\hat{N}$ be the number of products that are offered in more than $\tau$ epochs by Algorithm \ref{learn_algo}, $\ep{E}_{\ep{G}}$ denote the set of ``good epochs'', i.e.,
\begin{equation}\label{eq:good_epoch}
\mathcal{E}_{\ep{G}} = \left\{\ell \; \middle | v^{\sf UCB}_{i,\ell} \geq v_i \;\text{or} \; v^{\sf UCB}_{i,\ell} \leq v_i + C_1 \sqrt{\frac{v_i\log{({\sqrt{N}}\ell+1)}}{T_i(\ell)}} + C_2\frac{\log{({\sqrt{N}}\ell+1)}}{T_i(\ell)}\;\text{for all} \; i \right\},
\end{equation}
and $\ep{E}^{\sf sub\_opt}_{\ep{G}}$ be the set of ``good epochs'' that offer sub-optimal assortments,
\begin{equation}\label{eq:good_epoch_sub_opt}
{\ep{E}}^{\sf sub\_opt}_{\ep{G}} = \left\{\ell \in \ep{E}_{\ep{G}} \; \middle| \; R(S_\ell) < R(S^*) \right\}.
\end{equation}

\medskip \noindent {\bf Case 1: $\hat{N} = N$:}
Let $L$ be the total number of epochs and $S_1, \cdots, S_L$ be the assortments offered by Algorithm \ref{learn_algo} in epochs $1,\cdots, L$ respectively. Let $\ell_i$ be the epoch that offers product $i$ for the $\tau^{\sf th}$ time, specifically,  
$$\ell_i\; \overset{\Delta}{=}\min \left\{\ell \; \middle|\; T_i(\ell) = \tau\right\}.$$ 
Without loss of generality, assume that, $\ell_1 \leq \ell_2 \leq \cdots \leq \ell_N.$ Let $\hat{\ep{E}}^{\sf sub\_opt}_{\ep{G}}$ be the set of ``good epochs'' that offered sub-optimal assortments before epoch $\ell_{N-1},$ 
\begin{equation*}
\hat{\ep{E}}^{\sf sub\_opt}_{\ep{G}} = \left\{\ell \in \ep{E}^{\sf sub\_opt}_{\ep{G}} \; \middle| \; \ell \leq \ell_{N-1} \;\right\},
\end{equation*}
where ${\ep{E}}^{\sf sub\_opt}_{\ep{G}}$ is as defined as in \eqref{eq:good_epoch_sub_opt}. Finally, let $\hat{\ep{E}}^{\sf sub\_opt(N)}_{\ep{G}}$ be the set of ``good epochs'' that offered sub-optimal assortments not containing product $N$ before epoch $\ell_{N-1},$ 
\begin{equation*}
\hat{\ep{E}}^{\sf sub\_opt(N)}_{\ep{G}} = \left\{\ell \in \hat{\ep{E}}^{\sf sub\_opt}_{\ep{G}} \; \middle| \; N \not \in S_\ell \;\right\}.
\end{equation*}
Every assortment $S_\ell$ offered in epoch $\ell \in \hat{\ep{E}}^{\sf sub\_opt(N)}_{\ep{G}}$ can contain at most $N-1 = M$ products, therefore by the inductive hypothesis, we have $
|\hat{\ep{E}}^{\sf sub\_opt(N)}_{\ep{G}}| \leq M\tau.$ 
We can partition $ \hat{\ep{E}}^{\sf sub\_opt}_{\ep{G}} $ as,  $$ \hat{\ep{E}}^{\sf sub\_opt}_{\ep{G}}  =  \hat{\ep{E}}^{\sf sub\_opt(N)}_{\ep{G}} \cup \left\{\ell \in \ep{E}^{\sf sub\_opt}_{\ep{G}} \; \middle| \; N \in S_\ell \;\right\},$$ and hence it follows that, 
$$
|\hat{\ep{E}}^{\sf sub\_opt}_{\ep{G}}| \leq M\tau + \left|\left\{\ell \in \ep{E}^{\sf sub\_opt}_{\ep{G}} \; \middle| \; N \in S_\ell \;\right\}\right|.$$
Note that $T_N(\ell_{N-1})$ is the number of epochs until epoch $\ell_{N-1}$, in which product $N$ has been offered. Hence, it is higher than the number of ``good epochs'' before epoch $\ell_{N-1}$ that offered a sub-optimal assortment containing product $N$ and it follows that, 
\begin{equation}\label{eq:inductive_step}
|\hat{\ep{E}}^{\sf sub\_opt}_{\ep{G}}| \leq M\tau + T_N(\ell_{N-1}).
\end{equation}
Note that from Lemma \ref{sensitivity_analysis}, we have that any ``good epoch'' offering sub-optimal assortment must offer product $N$, since all the the other products have been offered in at least $\tau$ epochs. Therefore, we have, for any $\ell \in {\ep{E}}^{\sf sub\_opt}_{\ep{G}}\backslash \hat{\ep{E}}^{\sf sub\_opt}_{\ep{G}}$, $N \in S_\ell$ and thereby, 
$$T_N(\ell_N) - T_N(\ell_{N-1}) \geq |{\ep{E}}^{\sf sub\_opt}_{\ep{G}}| - |\hat{\ep{E}}^{\sf sub\_opt}_{\ep{G}}| .$$
From definition of $\ell_{N}$, we have that $T_N(\ell_N) = \tau$ and substituting \eqref{eq:inductive_step} in the above inequality, we obtain
$$|{\ep{E}}^{\sf sub\_opt}_{\ep{G}}| \leq (M+1)\tau.$$
\medskip \noindent {\bf Case 2: $\hat{N} < N$:} The proof for the case when $\hat{N}<N$ is similar along the lines of the previous case (we will make the same arguments using $\hat{N}-1$ instead of $N-1.$) and is skipped in the interest of avoiding redundancy.  \hfill $\halmos$

\noindent Following the proof of Lemma \ref{no_sub_opt}, we can establish the following result. 
\begin{corollary}\label{cor_no_sub_opt}
The number of epochs that offer a product that does not satisfy the condition, $T_i(\ell) \geq \log{NT}$, is bounded by $N\log{NT}.$ In particular, 
$$\left|\left\{\ell\;\Big|\; T_i(\ell) < \log{NT} \;\text{for some}\;i\in S_\ell\right\}\right| \leq N\log{NT}.$$
\end{corollary}

\medskip \noindent {\bf {Proof of Theorem~\ref{assymptotic_bound}:}}
{We will re-use the ideas from proof of Theorem \ref{main_result} to prove Theorem \ref{assymptotic_bound}. Briefly, we breakdown the regret into regret over ``good epochs'' and ``bad epochs.''  First we argue using Lemma \ref{lem:UCBv}, that the probability of an epoch being ``bad epoch'' is  ``small,'' and hence the expected  cumulative regret over the bad epochs is ``small.'' We will then use Lemma \ref{no_sub_opt} to argue that there are only ``small'' number of ``good epochs'' that offer sub-optimal assortments.  Since, Algorithm \ref{learn_algo} do not incur regret in epochs that offer the optimal assortment, we can replace the length of the horizon $T$ with the cumulative length of the time horizon that offers sub-optimal assortments (which is ``small'') and re-use analysis from Appendix \ref{sec:completeProof}. We will now make these notions rigorous and complete the proof of Theorem \ref{assymptotic_bound}.}

\medskip \noindent Following the analysis in Appendix \ref{sec:completeProof}, we reformulate the regret as 
\begin{equation}\label{eq:ce_regret_assymptotic}
Reg_\pi(T,\mb{v}) =  \mathbb{E}_{\pi}\left\{\sum_{\ell=1}^L \left(1+V(S_\ell) \right)\left(R(S^*, \mb{v}) - R(S_\ell,\mb{v})  \right)\right\},
\end{equation}
where $S^*$ is the optimal assortment, $V(S_\ell) = \sum_{j\in S_\ell} v_j$ and the expectation in equation \eqref{eq:ce_regret_assymptotic} is over the random variables $L$ and $S_\ell$. Similar to the analysis in Appendix \ref{sec:completeProof}, for the sake of brevity, we define,  
\begin{equation}\label{eq:brevity_regret_assymptotic}
\Delta R_\ell {=} (1+V(S_\ell))\left(R(S^*,\mb{v})-R(S_\ell,\mb{v})\right). 
\end{equation}
Now the regret can be reformulated as 
\begin{equation}\label{eq:ce_regret_2_assymptotic}
\begin{aligned}
Reg_\pi(T,\mb{v}) &=  \mathbb{E}_{\pi}\left\{\sum_{\ell=1}^L \Delta R_\ell\right\}.\\
\end{aligned}
\end{equation}
For all $\ell =1, \ldots,L$,  define events $\mathcal{A}_{\ell}$ as,
\begin{equation*}
\mathcal{A}_{\ell} = \bigcup_{i=1}^N\left\{v^{\sf UCB}_{i,\ell} < v_i \;\text{or} \; v^{\sf UCB}_{i,\ell} > v_i + C_1 \sqrt{\frac{v_i\log{({\sqrt{N}}\ell+1)}}{T_i(\ell)}} + C_2\frac{\log{({\sqrt{N}}\ell+1)}}{T_i(\ell)} \right\}.
\end{equation*}
Let $\xi = \left\{\ell\;\Big|\; T_i(\ell) < \log{NT} \;\text{for some}\;i\in S_\ell\right\}$. We breakdown the regret in an epoch into the following  terms. 
\begin{equation*}
\mathbb{E}_{\pi}\left(\Delta R_\ell\right)= \mathbb{E}_\pi\left[\Delta R_\ell\cdot\mathbbm{1}(\mathcal{A}_{\ell-1}) + \Delta R_\ell\cdot\mathbbm{1}(\mathcal{A}^c_{\ell-1})\cdot \mathbbm{1}(\ell \in \xi) + \Delta R_\ell\cdot\mathbbm{1}(\mathcal{A}^c_{\ell-1})\cdot \mathbbm{1}(\ell \in \xi^c)\right].
\end{equation*}
Using the fact that $R(S^*,\mb{v})$ and $R(S_\ell,\mb{v})$ are both bounded by one and $V(S_\ell) \leq N$ in \eqref{eq:brevity_regret_assymptotic}, we have $\Delta R_\ell \leq N+1.$ Substituting the preceding inequality in the above equation, we obtain,
\begin{equation*}
\mathbb{E}_{\pi}\left(\Delta R_\ell\right) \leq (N+1)\mathbb{P}_{\pi}(\mathcal{A}_{\ell-1}) + (N+1)\mathbb{P}_\pi\left(\ell \in \xi \right)+ \mathbb{E}\left[\Delta R_\ell\cdot\mathbbm{1}(\mathcal{A}^c_{\ell-1})\cdot \mathbbm{1}(\ell \in \xi^c)\right].
\end{equation*}
From the analysis in Appendix \ref{sec:completeProof} (see \eqref{eq:low_prob_event}), we have 
$\ep{P}(\ep{A}_\ell) \leq \frac{13}{\ell}.$
Therefore, it follows that,
\begin{equation*}
\begin{aligned}
\mathbb{E}_{\pi}\left(\Delta R_\ell\right) \leq \frac{13(N+1)}{\ell} + (N+1)\mathbb{P}_\pi\left(\ell \in \xi \right)+ \mathbb{E}\left[\Delta R_\ell\cdot\mathbbm{1}(\mathcal{A}^c_{\ell-1})\cdot \mathbbm{1}(\ell \in \xi^c)\right].
\end{aligned}
\end{equation*}
Substituting the above inequality in \eqref{eq:ce_regret_2_assymptotic}, we obtain
\begin{equation*}
Reg_\pi(T,\mb{v}) \leq 14N\log{T} + (N+1)\sum_{\ell=1}^L\mathbb{P}_\pi\left(\ell \in \xi \right)+  \mathbb{E}_{\pi}\left[ \sum_{\ell =1}^L \Delta R_\ell\cdot\mathbbm{1}(\mathcal{A}^c_{\ell-1})\cdot \mathbbm{1}(\ell \in \xi^c)\right].
\end{equation*}
From Corollary \ref{cor_no_sub_opt}, we have that $\sum_{\ell =1}^L \mathbbm{1}(\ell \in \xi) \leq N\log{NT}.$ Hence, we have, 
\begin{equation}\label{eq:assymptotic_first_term_bound}
Reg_\pi(T,\mb{v}) \leq 14N\log{T} + N(N+1)\log{NT}+  \mathbb{E}_{\pi}\left[ \sum_{\ell =1}^L \Delta R_\ell\cdot\mathbbm{1}(\mathcal{A}^c_{\ell-1})\cdot \mathbbm{1}(\ell \in \xi^c)\right].
\end{equation}

\noindent  Let $\ep{E}^{\sf sub\_opt}_\ep{G}$ be the set of ``good epochs'' offering sub-optimal products, more specifically,
$$\ep{E}^{\sf sub\_opt}_\ep{G}\; \overset{\Delta}{=} \left\{\ell \; \middle|\;\mathbbm{1}(\ep{A}^c_{\ell})=1\;\text{and}\; R(S_\ell,\mb{v}) < R(S^*,\mb{v}) \right\}.$$ 
 If $R(S_\ell,\mb{v}) = R(S^*,\mb{v})$, then by definition, we have $\Delta R_\ell = 0.$ Therefore, it follows that,
\begin{equation}\label{eq:assymptotic_reg_good_epoch_subopt}
\mathbb{E}_{\pi}\left[ \sum_{\ell =1}^L \Delta R_\ell\cdot\mathbbm{1}(\mathcal{A}^c_{\ell-1})\cdot \mathbbm{1}(\ell \in \xi^c)\right] = \mathbb{E}_{\pi}\left[\sum_{\ell \in \ep{E}^{\sf sub\_opt}_\ep{G} } \Delta R_\ell\cdot \mathbbm{1}(\ell \in \xi^c) \right].
\end{equation}
Whenever $\mathbbm{1}(\mathcal{A}^c_{\ell-1}) = 1$, from Lemma \ref{UCB_bound}, we have, $\tilde{R}_\ell(S^*) \geq R(S^*,\mb{v})$ and by our algorithm design, we have $\tilde{R}_\ell(S_\ell) \geq \tilde{R}_\ell(S^*)$ for all $\ell \geq 1$. Therefore, it follows that 
\begin{equation}\label{eq:good_epoch_subopt_reg}
\begin{aligned}
\mathbb{E}_{\pi}\left\{\Delta R_\ell \cdot \mathbbm{1}(\ep{A}^c_\ell)\right\} & \leq \mathbb{E}_{\pi}\left\{\left[(1+V(S_\ell))(\tilde{R}_\ell(S_\ell) - R(S_\ell,\mb{v}))\right]\cdot \mathbbm{1}(\mathcal{A}^c_{\ell-1}) \cdot \mathbbm{1}(\ell \in \xi^c)\right\},\\
& \leq \sum_{i\in S_\ell}\left(C_1\sqrt{\frac{v_i\log{({\sqrt{N}\ell}+1)}}{T_i(\ell)}} + \frac{C_2\log{({\sqrt{N}\ell}+1)}}{T_i(\ell)}\right) \cdot \mathbbm{1}(\ell \in \xi^c),\\
& \leq C \sum_{i\in S_\ell}\sqrt{\frac{v_i\log{({\sqrt{N}\ell}+1)}}{T_i(\ell)}}.
\end{aligned}
\end{equation}
where $C = C_1+C_2$, the second inequality in \eqref{eq:good_epoch_subopt_reg} follows from the definition of the event, $\ep{A}_\ell$ and the last inequality follows from the definition of set $\xi$.  From equations \eqref{eq:assymptotic_first_term_bound}, \eqref{eq:assymptotic_reg_good_epoch_subopt}, and \eqref{eq:good_epoch_subopt_reg} , we have, 
\begin{equation} \label{eq:dependent_bound_pre_final}
Reg_{\pi}(T,\mb{v}) \leq 14N^2\log{NT} + C\mathbb{E}_{\pi}\left\{\sum_{\ell \in  \ep{E}^{\sf sub\_opt}_{\ep{G}}}  \sum_{i\in S_\ell}\sqrt{\frac{\log{NT}}{T_i(\ell)}}  \right\},
\end{equation}

\noindent Let $T_i$ be the number of ``good epochs'' that offered sub-optimal assortments containing product $i$, specifically, 
$$T_i = \left|\left\{\ell \in \ep{E}^{\sf sub\_opt}_{\ep{G}} \; \middle| \; i \in S_\ell \right\}\right|.$$
Substituting the inequality $\sum_{\ell \in \ep{E}^{\sf sub\_opt}_{\ep{G}}} \frac{1}{\sqrt{T_i(\ell)}} \leq \sqrt{T_i} $ in \eqref{eq:dependent_bound_pre_final} and noting that $T_i \leq T$, we obtain,
\begin{equation*}
Reg_\pi(T,\mb{v}) \leq 14N^2\log{NT} + C \sum_{i=1}^N \mathbb{E}_\pi\left( \sqrt{T_i \log{T}}\right).
\end{equation*}
From Jenson's inequality, we have $\mathbb{E}_\pi\left(\sqrt{T}_i\right) \leq \sqrt{\mathbb{E}_\pi\left(T_i\right)}$ and therefore, it follows that,
\begin{equation*}
Reg_{\pi}(T,\mb{v}) \leq 14N\log{T} + NC\log{NT} +  C \sum_{i=1}^N  \sqrt{\mathbb{E}_\pi\left(T_i\right) \log{NT}}.
\end{equation*}
From Cauchy-Schwartz inequality, we have, $\sum_{i=1}^N  \sqrt{\mathbb{E}_\pi\left(T_i\right)} \leq \sqrt{N \sum_{i=1}^N \mathbb{E}_\pi\left(T_i\right)}.$ Therefore, it follows that,
\begin{equation*}
\begin{aligned}
Reg_{\pi}(T,\mb{v}) \leq 14N^2\log{NT} +  C \sqrt{N \sum_{i=1}^N \mathbb{E}_\pi\left(T_i\right)\log{NT}}.
\end{aligned}
\end{equation*}
For any epoch $\ell,$ we have $|S_\ell| \leq N$. Hence, we have $\sum_{i=1}^N T_i \leq N |\ep{E}^{\sf sub\_opt}_{\ep{G}}|.$ From Lemma \ref{no_sub_opt}, we have $|\ep{E}^{\sf sub\_opt}_{\ep{G}}| \leq N \tau$. Therefore, we have $\sum_{i=1}^N \mathbb{E}_\pi\left(T_i\right) \leq N^2 \tau$ and hence, it follows that,
\begin{equation}
\begin{aligned}
Reg_{\pi}(T,\mb{v}) &\leq 14N^2\log{NT} +  CN \sqrt{N \tau\log{NT}},\\
&\leq 14N^2\log{NT} + C\frac{N^2\log{NT}}{\Delta(\mb{v})}. 
\end{aligned}
\end{equation}
\hfill $\halmos$

\section{Multiplicative Chernoff Bounds}\label{proof:multiplicative_chernoff}

We will extend the Chernoff bounds as discussed in \cite{mitzenmacher} \footnote{(originally discussed in \cite{angluin})} to geometric random variables and establish the following concentration inequality.
\begin{theorem1}\label{multiplicative_chernoff_geometric}
Consider $n$ i.i.d geometric random variables $X_1,\cdots,X_n$ with parameter $p$, i.e. for any $i$
$$Pr(X_i = m) = (1-p)^m p\; \; \forall m = \{0,1,2, \cdots\},$$
and let $\mu = \mathbb{E}(X_i) = \frac{1-p}{p}$. We have, 
\begin{enumerate}
\item
\begin{equation*}
Pr\left(\frac{1}{n} \sum_{i=1}^n X_i > (1+\delta) \mu\right) \leq 
\left\{ 
\begin{array}{ll}
\exp{\left(-\frac{n\mu \delta^2}{2(1+\delta)(1+\mu)^2}\right)} \; & \text{if} \; \; \mu \leq 1,\\
\exp{\left(- \frac{n \delta^2 \mu^2}{6 (1+\mu)^2}\left(3 - \frac{2\delta\mu}{1+\mu}\right)\right)} \;  & \text{if} \; \; \mu \geq 1 \; \text{and} \; \delta \in (0,1).
\end{array}\right.
\end{equation*}
and 
\item \begin{equation*}
Pr\left(\frac{1}{n} \sum_{i=1}^n X_i < (1-\delta) \mu\right) \leq 
\left\{ 
\begin{array}{ll}
\exp{\left(- \frac{n \delta^2 \mu}{6 (1+\mu)^2}\left(3 - \frac{2\delta\mu}{1+\mu}\right)\right)} \; & \text{if} \; \; \mu \leq 1,\\
\exp{\left(-  \frac{n\delta^2 \mu ^2}{2(1+\mu)^2}\right)} \;  & \text{if} \; \; \mu \geq 1.
\end{array}\right.
\end{equation*}
\end{enumerate}

\end{theorem1}
\proof{Proof.} We will first bound $Pr\left(\frac{1}{n} \sum_{i=1}^n X_i > (1+\delta) \mu\right)$ and then follow a similar approach for bounding $Pr\left(\frac{1}{n} \sum_{i=1}^n X_i < (1-\delta) \mu\right)$ to complete the proof. 
\vspace{4pt}
\subsection*{ Bounding ${Pr\left(\frac{1}{n} \sum_{i=1}^n X_i > (1+{\delta}) {\mu}\right)}$:} 
For all $i$ and for any $0<t < \log{\frac{1+\mu}{\mu}},$ we have, $$\mathbb{E}(e^{tX_i}) = \frac{1}{1-\mu(e^{t}-1)}.$$
Therefore, from Markov Inequality,   we have 
\begin{equation*}
\begin{aligned}
Pr\left(\frac{1}{n} \sum_{i=1}^n X_i > (1+\delta) \mu\right) &= Pr \left(e^{t \sum_{i=1}^n X_i} > e^{(1+\delta) n\mu t}\right),\\
& \leq e^{-(1+\delta) n\mu t}\prod_{i=1}^n \mathbb{E}(e^{tX_i}), \\
& = e^{-(1+\delta) n\mu t} \left(\frac{1}{1-\mu(e^t - 1)}\right)^n.
\end{aligned}
\end{equation*} 
Therefore, we have, 
\begin{equation}\label{eq:one_sided_prelimn}
Pr\left(\frac{1}{n} \sum_{i=1}^n X_i > (1+\delta) \mu\right) \leq  \underset{0 < t< \log{\frac{1+\mu}{\mu}}}{\min} e^{-(1+\delta) n\mu t} \left(\frac{1}{1-\mu(e^t - 1)}\right)^n.
\end{equation}
We have,
\begin{equation}\label{convex_opt}
\underset{0 < t< \log{\frac{1+\mu}{\mu}}}{\text{argmin}} e^{-(1+\delta) n\mu t} \left(\frac{1}{1-\mu(e^t - 1)}\right)^n = \underset{0 < t< \log{\frac{1+\mu}{\mu}}}{\text{argmin}} -(1+\delta) n\mu t  -  n\log\left({1-\mu(e^t - 1)}\right),
\end{equation}
Noting that the right hand side in the above equation is a convex function in $t$, we obtain the optimal $t$ by solving for the zero of the derivative. Specifically, at optimal $t$, we have 
$$e^t = \frac{(1+\delta)(1+\mu)}{1+\mu(1+\delta)}.$$
Substituting the above expression in \eqref{eq:one_sided_prelimn}, we obtain the following bound. 
\begin{equation}\label{eq:one_sided_prelimn_bound}
Pr\left(\frac{1}{n} \sum_{i=1}^n X_i > (1+\delta) \mu\right) \leq {\left(1 - \frac{\delta}{(1+\delta)(1+\mu)}\right)}^{n\mu(1+\delta)} {\left(1+\frac{\delta \mu}{1+\mu}\right)^n}.
\end{equation}
First consider the setting where $\mu \in (0,1)$.

\medskip \noindent \textbf{Case 1a: If ${{\mu} \in (0,1)}$:} 
From Taylor series of $\log{(1-x)}$, we have that 
\begin{equation*}\label{geq_second_order_del1}
n\mu(1+\delta)\log{\left(1- \frac{\delta}{(1+\delta)(1+\mu)} \right)} \leq -\frac{n\delta \mu}{1+\mu} - \frac{n\delta^2 \mu}{2(1+\delta)(1+\mu)^2},
\end{equation*}
 From Taylor series for $\log{(1+x)}$, we have 
\begin{equation*}\label{geq_first_order_del1}
n \log{\left(1+\frac{\delta \mu}{1+\mu}\right)} \leq \frac{n\delta \mu}{(1+\mu)},
\end{equation*}
Note that if $\delta > 1$, we can use the fact that $\log{(1+\delta x)} \leq \delta \log{(1+x)}$ to arrive at the preceding result.
Substituting the preceding two equations  in \eqref{eq:one_sided_prelimn_bound}, we have
\begin{equation}\label{eq:geq_mul1}
Pr\left(\frac{1}{n} \sum_{i=1}^n X_i > (1+\delta) \mu\right) \leq 
\exp{\left(-\frac{n\mu \delta^2}{2(1+\delta)(1+\mu)^2}\right)} ,
\end{equation}

\medskip \noindent \textbf{Case 1b: If ${{\mu} \geq 1}$ :} 
From Taylor series of $\log{(1-x)}$, we have that 
\begin{equation*}\label{geq_first_order_del1_mug1}
n\mu(1+\delta)\log{\left(1- \frac{\delta}{(1+\delta)(1+\mu)} \right)} \leq -\frac{n\delta \mu}{1+\mu},
\end{equation*}
If $\delta < 1$, from Taylor series for $\log{(1+x)}$, we have 
\begin{equation*}\label{geq_second_order_del1_mug1}
n \log{\left(1+\frac{\delta \mu}{1+\mu}\right)} \leq \frac{n\delta \mu}{(1+\mu)} - \frac{n \delta^2 \mu^2}{6 (1+\mu)^2}\left(3 - \frac{2\delta\mu}{1+\mu}\right).
\end{equation*}
If $\delta \geq 1$, we have $\log{(1+\delta x)} \leq \delta \log{(1+x)}$ and from Taylor series for $\log{(1+x)}$,  it follows that, 
\begin{equation*}
n \log{\left(1+\frac{\delta \mu}{1+\mu}\right)} \leq \frac{n\delta \mu}{(1+\mu)} - \frac{n \delta \mu^2}{6 (1+\mu)^2}\left(3 - \frac{2\mu}{1+\mu}\right).
\end{equation*}
Therefore, substituting preceding results in \eqref{eq:one_sided_prelimn_bound}, we have
\begin{equation} \label{eq:geq_mug1}
Pr\left(\frac{1}{n} \sum_{i=1}^n X_i > (1+\delta) \mu\right) \leq 
\left\{ 
\begin{array}{ll}
\exp\left(- \frac{n \delta^2 \mu^2}{6 (1+\mu)^2}\left(3 - \frac{2\delta\mu}{1+\mu}\right)\right) \;  & \text{if} \; \; \mu \geq 1 \; \text{and} \; \delta \in (0,1), \\ 
\exp{\left(- \frac{n \delta \mu^2}{6 (1+\mu)^2}\left(3 - \frac{2\mu}{1+\mu}\right)\right)} & \text{ if }\; \; \mu \geq 1 \;\text{and}\; \delta \geq {1}.
\end{array}\right.
\end{equation}
\subsection*{ Bounding ${Pr\left(\frac{1}{n} \sum_{i=1}^n X_i < (1-{\delta}) {\mu}\right)}$:} Now to bound the other one sided inequality,  we use the fact that $$\mathbb{E}(e^{-tX_i}) = \frac{1}{1-\mu(e^{-t}-1)},$$ and follow a similar approach. More specifically, from Markov Inequality,  for any $t > 0$ and $0 <\delta < 1,$  we have 
\begin{equation*}
\begin{aligned}
Pr\left(\frac{1}{n} \sum_{i=1}^n X_i < (1-\delta) \mu\right) &= Pr \left(e^{ -t \sum_{i=1}^n X_i} > e^{-(1-\delta) n\mu t}\right),\\
& \leq e^{(1-\delta) n\mu t}\prod_{i=1}^n \mathbb{E}(e^{-tX_i}), \\
& = e^{(1-\delta) n\mu t} \left(\frac{1}{1-\mu(e^{-t} - 1)}\right)^n.
\end{aligned}
\end{equation*} 
Therefore, we have 
\begin{equation*}
Pr\left(\frac{1}{n} \sum_{i=1}^n X_i < (1-\delta) \mu\right) \leq  \underset{ t > 0}  {\min} \; e^{-(1+\delta) n\mu t} \left(\frac{1}{1-\mu(e^{-t} - 1)}\right)^n,
\end{equation*}
Following similar approach as in optimizing the previous bound (see \eqref{eq:one_sided_prelimn}) to establish the following result. 
\begin{equation*}
Pr\left(\frac{1}{n} \sum_{i=1}^n X_i < (1-\delta) \mu\right) \leq {\left(1 + \frac{\delta}{(1-\delta)(1+\mu)}\right)}^{n\mu(1-\delta)} {\left(1-\frac{\delta \mu}{1+\mu}\right)^n}.
\end{equation*}
Now we will use Taylor series for $\log{(1+x)}$ and $\log{(1-x)}$ in a similar manner as described for the other bound to obtain the required result. In particular, since $1-\delta \leq 1$, we have for any $x > 0$ it follows that $(1+\frac{ x}{1-\delta})^{(1-\delta)} \leq (1+x)$ . Therefore, we have
\begin{equation}\label{eq:one_sided_prelimn_bound_leq}
Pr\left(\frac{1}{n} \sum_{i=1}^n X_i < (1-\delta) \mu\right) \leq {\left(1 + \frac{\delta}{1+\mu}\right)}^{n\mu} {\left(1-\frac{\delta \mu}{1+\mu}\right)^n}.
\end{equation}

\medskip \noindent {\bf Case 2a. If $\mu \in (0,1)$}:
Note that since $X_i \geq 0$ for all $i$, we have a zero probability event if $\delta > 1$. Therefore, we assume $\delta < 1$ and from Taylor series for $\log{(1-x)}$, we have 
\begin{equation*}\label{leq_first_order_del1}
n\log{\left(1- \frac{\delta\mu}{1+\mu} \right)} \leq -\frac{n\delta \mu}{1+\mu}, 
\end{equation*}
 and from Taylor series for $\log{(1+x)}$, we have 
\begin{equation*}
n\mu \log{\left(1+\frac{\delta }{1+\mu}\right)} \leq \frac{n\delta \mu}{(1+\mu)} - \frac{n \delta^2 \mu}{6 (1+\mu)^2}\left(3 - \frac{2\delta\mu}{1+\mu}\right).
\end{equation*}
Therefore, substituting the preceding equations  in \eqref{eq:one_sided_prelimn_bound_leq}, we have, 
\begin{equation}\label{eq:leq_mul1}
Pr\left(\frac{1}{n} \sum_{i=1}^n X_i < (1-\delta) \mu\right) \leq 
\exp{\left(- \frac{n \delta^2 \mu}{6 (1+\mu)^2}\left(3 - \frac{2\delta\mu}{1+\mu}\right)\right)}.
\end{equation}

\medskip \noindent {\bf Case 2b. If $\mu \geq 1$}:
For similar reasons as discussed above, we assume $\delta < 1$ and from Taylor series for $\log{(1-x)}$, we have 
\begin{equation*}
n\log{\left(1- \frac{\delta\mu}{1+\mu} \right)} \leq -\frac{n\delta \mu}{1+\mu} - \frac{n\delta^2 \mu ^2}{2(1+\mu)^2}, 
\end{equation*}
 and from Taylor series for $\log{(1+x)}$, we have 
\begin{equation*}\label{leq_second_order_del1}
n \log{\left(1+\frac{\delta \mu}{1+\mu}\right)} \leq \frac{n\delta }{(1+\mu)} .
\end{equation*}
Therefore, substituting the preceding equations  in \eqref{eq:one_sided_prelimn_bound_leq}, we have, 
\begin{equation}\label{eq:leq_mug1}
Pr\left(\frac{1}{n} \sum_{i=1}^n X_i < (1-\delta) \mu\right) \leq 
\exp{\left(-  \frac{n\delta^2 \mu ^2}{2(1+\mu)^2}\right)}.
\end{equation}
The result follows from \eqref{eq:geq_mul1}, \eqref{eq:geq_mug1}, \eqref{eq:leq_mul1} and \eqref{eq:leq_mug1}.   \hfill \halmos

%


Now, we will adapt a non-standard corollary from \cite{Babaioff} and \cite{Kleinberg} to our estimates to obtain sharper bounds.

\begin{lemma}\label{multiplicative_chernoff_estimates}
Consider $n$ i.i.d geometric random variables $X_1,\cdots,X_n$ with parameter $p$, i.e. for any $i$, $P(X_i = m) = (1-p)^m p\; \; \forall m = \{0,1,2, \cdots\}.$ Let $\mu = \mathbb{E}_{\pi}(X_i) = \frac{1-p}{p}$ and $\bar{X} = \frac{\sum_{i=1}^n X_i}{n} $. If $n > 48\log{({\sqrt{N}}\ell+1)}$, then we have for all $n=1,2, \cdots,$
\begin{enumerate}
\item \begin{equation}\label{eq:multi_barx}
\mathcal{P}\left(\left|\bar{X} - {\mu}\right| > \max{\left\{\sqrt{\bar{X}}, \bar{X}\right\}}\sqrt{\frac{48\log{({\sqrt{N}}\ell+1)}}{n}} + \frac{48 \log{({\sqrt{N}}\ell+1)}}{n}\right)   \leq  \frac{6}{\ell^2}.
\end{equation}
\item \begin{equation}\label{eq:multi_mu} \displaystyle \mathcal{P}\left(\left|\bar{X} -\mu \right | \geq  \max{\left\{\sqrt{\mu}, \mu\right\}}\sqrt{\frac{24\log{({\sqrt{N}}\ell+1)}}{n}} + \frac{48 \log{({\sqrt{N}}\ell+1)}}{n}\;\right)  \leq \frac{4}{\ell^2},\;  \end{equation}
\item \begin{equation}\label{eq:mu} {\displaystyle \mathcal{P}\left(\bar{X}  \geq  \frac{3\mu}{2} + \frac{48 \log{({\sqrt{N}}\ell+1)}}{n}\;\right)  \leq \frac{3}{\ell^2}}.\;  \end{equation}
\end{enumerate}
\end{lemma}
\proof{Proof.}
We will analyze the cases $\mu < 1$ and $\mu \geq 1$ separately. 

\medskip \noindent {\bf Case-1: $\mu \leq 1.$} Let $\delta = (\mu+1) \sqrt{\frac{6\log{({\sqrt{N}}\ell+1)}}{\mu n}} .$
First assume that $\delta \leq \frac{1}{2}.$ Substituting the value of $\delta$ in Theorem \ref{multiplicative_chernoff_geometric},  we obtain, 
\begin{equation}\label{eq:d9}
\begin{aligned}
\ep{P}\left(\bar{X} - \mu  > \delta \mu \right) &\leq\frac{1}{{N}\ell^2},\\
\ep{P}\left(\bar{X} - \mu  < -\delta \mu \right) &\leq\frac{1}{{N}\ell^2},\\
 \ep{P}\left(\left|\bar{X} - {\mu}\right| <  (\mu+1)\sqrt{\frac{6\mu\log{({\sqrt{N}}\ell+1)}}{n}}\right) &\geq  1-\frac{2}{{N}\ell^2}.
\end{aligned}
\end{equation}
{Since $\delta \leq \frac{1}{2}$, we have $\mathcal{P}\left(\bar{X} -\mu \leq -\frac{\mu}{2}\right) \leq \mathcal{P}\left(\bar{X} -\mu \leq -\delta \mu\right)$. Hence, from \eqref{eq:d9}, we have, 
\begin{equation*}
\begin{aligned} 
\mathcal{P}\left(\bar{X} -\mu \leq -\frac{\mu}{2}\right)  \leq \frac{1}{{N}\ell^2},\\
\end{aligned}
\end{equation*}
and hence, it follows that,
\begin{equation}\label{eq:d10}
\begin{aligned} 
\mathcal{P}\left(2\bar{X}\geq {\mu}\right) & \geq  1-\frac{1}{{N}\ell^2}.\\
\end{aligned}
\end{equation}}
From \eqref{eq:d9} and \eqref{eq:d10}, we have,
\begin{equation}\label{ineq_part_2_mu_l1}
\mathcal{P}\left(\left|\bar{X} - {\mu}\right| < \sqrt{\frac{48\bar{X}\log{({\sqrt{N}}\ell+1)}}{n}}\right) \geq \mathcal{P}\left(\left|\bar{X} - {\mu}\right| < \sqrt{\frac{24\mu \log{({\sqrt{N}}\ell+1)}}{n}}\right) \geq  1-\frac{3}{{N}\ell^2}.
\end{equation}
Since $\delta \leq \frac{1}{2}$, we have, $\mathcal{P}\left(\bar{X}  \leq \frac{3\mu}{2}\right) \geq \mathcal{P}\left(\bar{X} < (1+\delta) \mu\right)$. {Hence, from \eqref{eq:d9}, we have 
\begin{equation}\label{eq:barv_ub_v}
\ep{P}\left(\bar{X} \leq \frac{3\mu}{2}\right) \geq 1-\frac{1}{{N}\ell^2}.
\end{equation}}
Since, $\mu \leq 1$, we have $\ep{P}\left(\bar{X} \leq \frac{3}{2}\right) \geq 1-\frac{1}{{N}\ell^2}$ and
\begin{equation*}
\mathcal{P}\left(\bar{X} \leq \sqrt{\frac{3\bar{X}}{2}}\right) \geq  1-\frac{1}{{N}\ell^2}.
\end{equation*}
Therefore, substituting above result in \eqref{ineq_part_2_mu_l1}, the inequality \eqref{eq:multi_barx} follows. 
\begin{equation}\label{mu_l1_est}
\mathcal{P}\left(\left|\bar{X} - {\mu}\right| > \max{\left\{\sqrt{\bar{X}}, \sqrt{\frac{2}{3}}\bar{X}\right\}}\sqrt{\frac{48\log{({\sqrt{N}}\ell+1)}}{n}}\right)   \leq  \frac{4}{{N}\ell^2}.
\end{equation}

\noindent Now consider the scenario, when $(\mu+1) \sqrt{\frac{6\log{({\sqrt{N}}\ell+1)}}{\mu n}} > \frac{1}{2}$. Then, we have, 
$$\delta_1 \overset{\Delta}{=} {\frac{12(\mu+1) ^2\log{({\sqrt{N}}\ell+1)}}{\mu n}}  \geq \frac{1}{2},$$ which implies, 
\begin{equation*}
\begin{aligned}
\exp{\left(-\frac{n\mu \delta_1^2}{2(1+\delta_1)(1+\mu)^2}\right)} & \leq \exp{\left(-\frac{n\mu \delta_1}{6(1+\mu)^2}\right)},\\
\exp{\left(- \frac{n \delta_1^2 \mu}{6 (1+\mu)^2}\left(3 - \frac{2\delta_1\mu}{1+\mu}\right)\right)} & \leq \exp{\left(-\frac{n\mu \delta_1}{6(1+\mu)^2}\right)}.
\end{aligned}
\end{equation*}
Therefore, substituting the value of $\delta_1$ in Theorem \ref{multiplicative_chernoff_geometric}, we have 
\begin{equation}\label{eq:log_add}
\mathcal{P}\left(\left|\bar{X} - \mu\right|> \frac{48\log{({\sqrt{N}}\ell+1)}}{n}\right) \leq \frac{2}{{N}\ell^2}. 
\end{equation} 
{Hence, from \eqref{eq:log_add} and \eqref{mu_l1_est}, it follows that, 
\begin{equation}\label{eq:d13}
\mathcal{P}\left(\left|\bar{X} - \mu\right|> \max{\left\{\sqrt{\bar{X}}, \sqrt{\frac{2}{3}}\bar{X}\right\}}\sqrt{\frac{48\log{({\sqrt{N}}\ell+1)}}{n}} + \frac{48\log{({\sqrt{N}}\ell+1)}}{n}\right) \leq \frac{6}{{N}\ell^2}. 
\end{equation} }

\medskip \noindent \textbf{Case 2: $\pmb{\mu}\geq1$} 

Let $\delta = \sqrt{\frac{12\log{({\sqrt{N}}\ell+1)}}{ n}}$, then by our assumption, we have $\delta \leq \frac{1}{2}$. Substituting the value of $\delta$ in Theorem \ref{multiplicative_chernoff_geometric}, we obtain,  
\begin{equation*}
\begin{aligned}
 \mathcal{P}\left(\left|\bar{X} - {\mu}\right| <  \mu\sqrt{\frac{12\log{({\sqrt{N}}\ell+1)}}{n}}\right) &\geq  1-\frac{2}{{N}\ell^2},\\
\mathcal{P}\left(2\bar{X} \geq {\mu}\right) &\geq  1-\frac{1}{{N}\ell^2}.
\end{aligned}
\end{equation*}
 Hence we have,
\begin{equation}\label{ineq_mu_g1}
\mathcal{P}\left(\left|\bar{X} - {\mu}\right| < \bar{X}\sqrt{\frac{48\log{({\sqrt{N}}\ell+1)}}{n}}\right) \geq \mathcal{P}\left(\left|\bar{X} - {\mu}\right| < \mu\sqrt{\frac{12 \log{({\sqrt{N}}\ell+1)}}{n}}\right) \geq  1-\frac{3}{{N}\ell^2}.
\end{equation}
By assumption $\mu \geq 1$. Therefore, we have $\ep{P}\left(\bar{X} \geq \frac{1}{2}\right) \geq 1-\frac{1}{{N}\ell^2}$ and, 
\begin{equation}\label{safety_ub}
\mathcal{P}\left(\bar{X} \geq \sqrt{\frac{\bar{X}}{2}}\right) \geq  1-\frac{1}{{N}\ell^2}.
\end{equation}
Therefore, from \eqref{ineq_mu_g1} and \eqref{safety_ub}, we have 
\begin{equation}\label{mu_g1_est}
\mathcal{P}\left(\left|\bar{X} - {\mu}\right| > \max{\left\{{\bar{X}}, \sqrt{\frac{\bar{X}}{2}}\right\}}\sqrt{\frac{48\log{({\sqrt{N}}\ell+1)}}{n}} \right)   \leq  \frac{4}{{N}\ell^2}.
\end{equation}
{We complete the proof by stating that \eqref{eq:multi_barx} follows from \eqref{eq:d13} and \eqref{mu_g1_est}, while \eqref{eq:multi_mu} follows from \eqref{ineq_part_2_mu_l1} and \eqref{ineq_mu_g1} and \eqref{eq:mu} follows from \eqref{eq:barv_ub_v} and \eqref{eq:log_add}.}
\hfill~\halmos

 From the proof of Lemma \ref{multiplicative_chernoff_estimates}, the following result follows. 
\begin{corollary}\label{multiplicative_chernoff_no_estimates}
Consider $n$ i.i.d geometric random variables $X_1,\cdots,X_n$ with parameter $p$, i.e. for any $i$, $P(X_i = m) = (1-p)^m p\; \; \forall m = \{0,1,2, \cdots\}.$ Let $\mu = \mathbb{E}_{\pi}(X_i) = \frac{1-p}{p}$ and $\bar{X} = \frac{\sum_{i=1}^n X_i}{n} $. If $\mu \leq 1$, then we have,
\begin{enumerate}
\item $\mathcal{P}\left(\left|\bar{X} - {\mu}\right| > \sqrt{\frac{48\bar{X}\log{({\sqrt{N}}\ell+1)}}{n}} + \frac{48 \log{({\sqrt{N}}\ell+1)}}{n}\right)   \leq  \frac{6}{{N}\ell^2}.$ for all $n=1,2, \cdots$.
\item $ \mathcal{P}\left(\left|\bar{X} -\mu \right | \geq  \sqrt{\frac{24\mu\log{({\sqrt{N}}\ell+1)}}{n}} + \frac{48 \log{({\sqrt{N}}\ell+1)}}{n}\;\right)  \leq \frac{4}{{N}\ell^2}$ for all $n=1, 2, \cdots $.
\item $ \mathcal{P}\left(\bar{X}  \geq  \frac{3\mu}{2} + \frac{48 \log{({\sqrt{N}}\ell+1)}}{n}\;\right)  \leq \frac{3}{{N}\ell^2}.\;  $
\end{enumerate}
\end{corollary}

\begin{proofof}{Lemma \ref{chernoff_hoeffding_ineq}}
 Fix $i$ and $\ell$, define the events, \[\mathcal{A}_{i,\ell} = \left\{\left|\bar{v}_{i,\ell} - v_i\right| > \sqrt{48\bar{v}_{i,\ell}\frac{\log{({\sqrt{N}}\ell+1)}}{|\mathcal{T}_i(\ell)|}} + \frac{48\log{({\sqrt{N}}\ell+1)}}{|\mathcal{T}_i(\ell)|}\right\}.\]
Let $\bar{v}_{i,m} = \displaystyle \frac{\sum_{\tau=1}^m \hat{v}_{i,\tau}}{m}$. Then, we have, 
\begin{equation}\label{final_argument}
\begin{aligned}
\mathbb{P}_{\pi}\left( \mathcal{A}_{i,\ell}\right) &\leq \mathbb{P}_{\pi}\left\{ \underset{m \leq \ell }{\max}\;\left(|\bar{v}_{i,m} - v_i| - \sqrt{48\bar{v}_{i,m}\frac{\log{({\sqrt{N}}\ell+1)}}{m}} - \frac{48\log{({\sqrt{N}}\ell+1)}}{m} \right) > 0 \right\},\\
& = \mathbb{P}_{\pi}\left(\bigcup_{m=1}^\ell\;\left\{|\bar{v}_{i,m} - v_i| - \sqrt{48\bar{v}_{i,m}\frac{\log{({\sqrt{N}}\ell+1)}}{m}} - \frac{48\log{({\sqrt{N}}\ell+1)}}{m} > 0 \right\}  \right),\\
& \leq \sum_{m=1}^\ell \mathbb{P}_{\pi}\left(|\bar{v}_{i,m} - v_i| > \sqrt{48\bar{v}_{i,m}\frac{\log{({\sqrt{N}}\ell+1)}}{m}} - \frac{48\log{({\sqrt{N}}\ell+1)}}{m} \right), \\
& \overset{(a)}{\le} \sum_{m=1}^\ell \frac{6}{{N}\ell^2} \leq \frac{6}{{N}\ell}.
\end{aligned}
\end{equation}
where inequality (a) in \eqref{final_argument} follows from Corollary \ref{multiplicative_chernoff_no_estimates}. The first inequality in Lemma \ref{chernoff_hoeffding_ineq} follows from definition of $v^{\sf UCB}_{i,\ell}$, Corollary \ref{multiplicative_chernoff_no_estimates} and \eqref{final_argument}. The second and third inequality in Lemma \ref{chernoff_hoeffding_ineq} also can be derived in a similar fashion by appropriately modifying the definition of set $\ep{A}_{i,\ell}.$
\end{proofof}

Proof of Lemma  \ref{chernoff_hoeffding_ineq_extn} is similar to the proof of Lemma \ref{chernoff_hoeffding_ineq}.

\section{Lower Bound} \label{lower_bound_parametric_mab}
We follow the proof of $\Omega(\sqrt{NT})$ lower bound for the Bernoulli instance with parameters $\frac{1}{2}$. We first establish a bound on KL divergence, which will be useful for us later. 
\begin{lemma}\label{KL_bound}
Let $p$ and $q$ denote two Bernoulli distributions with parameters $\alpha +\epsilon$ and $\alpha$ respectively. Then, the KL divergence between the distributions $p$ and $q$ is bounded by $4K\epsilon^2$,
$$KL(p\|q) \leq \frac{4}{\alpha}\epsilon^2.$$
\end{lemma}
\proof{Proof.}
\begin{equation*}
\begin{aligned}
KL(p\|q) &= \alpha\cdot \log{\frac{\alpha}{\alpha+\epsilon }} + \left(1-\alpha\right)\log{\frac{1-\alpha}{1-\alpha-\epsilon}}\\
&= \alpha\left[\log{\displaystyle\frac{1-\displaystyle\frac{\epsilon}{1-\alpha}}{1+ \frac{\epsilon}{\alpha}}}\right] -\log{\left(1-\frac{\epsilon}{1-\alpha}\right)},\\
& = \alpha \log{\left(1-\frac{\epsilon}{(1-\alpha)(\alpha+\epsilon )}\right)}  -\log{\left(1-\frac{\epsilon}{1-\alpha}\right)},
\end{aligned}
\end{equation*}
using $1-x \leq e^{-x}$ and bounding the Taylor series for $-\log{1-x}$ by $x+2*x^2$ for $\displaystyle x = \frac{\epsilon}{1-\alpha}$, we have 
\begin{equation*}
\begin{aligned}
KL(p\|q) &\leq \frac{-\alpha\epsilon}{(1-\alpha)(\alpha+\epsilon )} + \frac{\epsilon}{1-\alpha} + 4\epsilon^2, \\
& = (\frac{2}{\alpha}+4)\epsilon^2 \leq \frac{4}{\alpha}\epsilon^2.
\end{aligned}
\end{equation*}
~\hfill~\halmos.


\vspace{3mm}
Fix a guessing algorithm, which at time $t$ sees the output of a coin $a_t$. Let $P_1, \cdots, P_n$ denote the distributions for the view of the algorithm from time $1$ to $T$, when the biased coin is hidden in the $i^{th}$ position. The following result establishes for any guessing algorithm, there are at least $\frac{N}{3}$ positions that a biased coin could be and will not be played by the guessing algorithm with probability at least $\frac{1}{2}$ . Specifically, 
\begin{lemma}\label{prob_bound}
Let $\mathcal{A}$ be any guessing algorithm operating as specified above and let $\textstyle t \leq \frac{N\alpha}{60\epsilon^2}$, for $\epsilon \leq \frac{1}{4}$ and $N \geq 12$. Then, there exists $J \subset \{1,\cdots,N\}$ with $|J| \geq \frac{N}{3}$ such that
\[\forall j \in J,\; \mathcal{P}_{j}(a_t=j) \leq \frac{1}{2}.\]
\end{lemma}
%
\proof{Proof.}
Let $N_i$ to be the number of times the algorithm plays coin $i$ up to time t. 
Let $P_0$ be the hypothetical distribution for the view of the algorithm when none of the $N$ coins are biased. We shall define the set $J$ by considering the behavior of the algorithm if tosses it saw were according to the distribution $P_0$. We define, 
\begin{equation}\label{eq:counting_arg}
J_1 = \left\{ i \, \middle| \, E_{P_0}(N_i) \leq \frac{3t}{N}\right\}, J_2 =\left\{i \, \middle| \mathcal{P}_{0}(a_t=i) \leq \frac{3}{N}\, \right\} \; \text{and} \; J = J_1\cap J_2.
\end{equation}
Since $\sum_{i} E_{P_0}(N_i) = t$ and $\sum_{i} \mathcal{P}_{0}(a_t=i) = 1$, a counting argument would give us $\displaystyle |J_1| \geq \frac{2N}{3}$ and $\displaystyle |J_2| \geq \frac{2n}{3}$ and hence $\displaystyle |J| \geq \frac{N}{3}$.
Consider any $j \in J$, we will now prove that if the biased coin is at position $j$, then the probability of algorithm guessing the biased coin will not be significantly different from the $P_0$ scenario. By Pinsker's inequality, we have 
\begin{equation}
\left|\mathcal{P}_{j}(a_t = j) - \mathcal{P}_{0}(a_t = j) \right| \leq \frac{1}{2} \sqrt{2\log{2}\cdot KL(P_0 \| P_j)},
\end{equation}
where $KL(P_0\|P_j)$ is the KL divergence of probability distributions $P_0$ and $P_j$ over the algorithm. Using the chain rule for KL-divergence, we have 
\[KL(P_0 \| P_j) = E_{P_0}(N_j) KL(p || q),\]
where $p$ is a Bernoulli distribution with parameter $\alpha$ and q is a Bernoulli distribution with parameter $\alpha+\epsilon$. From Lemma \ref{KL_bound} and \eqref{eq:counting_arg}, we have that 
Therefore,
\begin{equation*}
KL(P_0 \| P_j) \leq \frac{4\epsilon^2}{\alpha},
\end{equation*}
Therefore,
\begin{equation}
\begin{aligned}
\mathcal{P}_{j}(a_t = j) &\leq \mathcal{P}_{0}(a_t = j) + \frac{1}{2} \sqrt{2\log{2}\cdot KL(P_0 \| P_j)},\\
& \leq \frac{3}{N} + \frac{1}{2}\sqrt{(2\log{2})\frac{4\epsilon^2}{\alpha}E_{P_0}(N_j)},\\
& \leq \frac{3}{N} + \sqrt{2\log{2}}\sqrt{\frac{3t\epsilon^2}{N\alpha}} \leq \frac{1}{2}.
\end{aligned}
\end{equation}
The second inequality follows from \eqref{eq:counting_arg}, while the last inequality follows from the fact that $N > 12$ and $\textstyle t \leq \frac{N\alpha}{60\epsilon^2} \hfill~\halmos$.

\begin{proofof} {Lemma~\ref{lower_bound_MAB}}.
Let $\epsilon = \sqrt{\frac{N}{60\alpha T}}$. Suppose algorithm $\mathcal{A}$ plays coin $a_t$ at time $t$ for each $t = 1,\cdots,T$. Since $\textstyle T \leq \frac{N\alpha}{60\epsilon^2}$, for all $t \in \{1,\cdots,T-1\}$ there exists a set $J_t \subset \{1,\cdots,N\}$ with $|J_t| \geq \frac{N}{3}$ such that 
\begin{equation*}
\forall\; j \in J_t, P_j(j \in S_t) \leq \frac{1}{2}.
\end{equation*}
Let $i^*$ denote the position of the biased coin. Then,
\begin{equation*}
\mathbb{E}_{\pi}\left(\mu_{a_t}\, \middle| \, i^* \in J_t\right) \leq \frac{1}{2}\cdot \left(\alpha+\epsilon\right) + \frac{1}{2}\cdot \alpha = \alpha + \frac{\epsilon}{2},
\end{equation*}
\begin{equation*}
\mathbb{E}_{\pi}\left(\mu_{a_t}\, \middle| \, i^* \not \in J_t\right) \leq   \alpha + \epsilon.
\end{equation*}
Since $|J_t| \geq \frac{N}{3}$ and $i^*$ is chosen randomly, we have $P(i^* \in J_t) \geq \frac{1}{3}$. Therefore, we have 
\begin{equation*}
\mu_{a_t} \leq \frac{1}{3}\cdot \left(\alpha + \frac{\epsilon}{2}\right) + \frac{2}{3}\cdot \left(\alpha + \epsilon\right) = \alpha + \frac{5\epsilon}{6}
\end{equation*}
We have $\mu^* = \alpha + \epsilon$ and hence the $Regret \geq \frac{T\epsilon}{6}$.
\end{proofof}

\begin{lemma}\label{lemmaTL}
Let $L$ be the total number of calls to \amnl when \amab is executed for $T$ time steps. Then,
$$\mathbb{E}(L) \leq  3T.$$
\end{lemma}

\begin{proof}{Proof.}
Let $\eta_\ell$ be the random variable that denote the duration, assortment $S_\ell$ has been considered by $\ep{A}_{MAB}$, i.e. $\eta_\ell = 0$, if we no arm is pulled when $\ep{A}_{MNL}$ suggested assortment $S_\ell$ and $\eta_\ell \geq 1$, otherwise. We have 
$$\sum_{\ell=1}^{L-1} \eta_\ell \leq T.$$ Therefore, we have $\mathbb{E}\left(\sum_{\ell=1}^{L-1} \eta_\ell \right) \leq T$. Note that $\mathbb{E}(\eta_\ell) \geq \frac{1}{2}$. Hence, we have 
$\mathbb{E}(L) \leq 2T+1 \leq 3T.$ \halmos
\end{proof}

\subsection{{Lower Bound for the unconstrained \banditMNL\;problem ($K = N$)}} \label{lower_boundN}

We will complete proof of Theorem \ref{lower_bound} by showing that the lower bound holds true for the case when $K=N.$ We will show this by reduction to a parametric multi armed bandit problem with 2 arms.

\begin{definition}[{{\banditMNL}}~instance $\hat{I}_{\sf MNL}$] Define $\hat{I}_{\sf MNL}$ as the following (randomized) instance of unconstrained \banditMNL~problem, $N$ products, with revenues, $r_1 = 1$, $r_2 = \frac{1+\epsilon}{3+2\epsilon}$ and 
$r_i = 0.01$ for all $i = 3,\cdots,N,$ and MNL parameters $v_0 = 1$, $v_i= \frac{1}{2}$ for all $i=2,\cdots, N$, while
$v_1$ is randomly set at  $\{\frac{1}{2},\frac{1}{2}+\epsilon\}$, where $\epsilon=\sqrt{\frac{1}{32T}}.$
\end{definition}

\medskip \noindent  {{\bf Preliminaries on the \banditMNL\;instance $\hat{I}_{\sf MNL}$}:} Note that unlike the \banditMNL\;instance, $I_{\sf MNL}$, where any product can have the biased (higher) MNL parameter, in the \banditMNL\;instance $\hat{I}_{\sf MNL}$, only one product (product $1$) can be biased.  From the proof  of Lemma \ref{unconstrained_subset}, we have that, 
\begin{equation}\label{eq:threshold}
 i \in S^* \; \text{if and only if}\; r_i \geq R(S^*,\mb{v}),
 \end{equation}
where $S^*$ is the optimal assortment for $\hat{I}_{\sf MNL}$. 

\noindent Note that the revenue corresponding to assortment $\{1\}$ is
\begin{equation*}
R(\{1\},\mb{v}) = \left\{\begin{array}{ll}
\displaystyle \frac{1+2\epsilon}{3+2\epsilon}  , & \text{ if } v_1 = \frac{1}{2} + \epsilon \\ 
\displaystyle \frac{1}{3}, & \text{ if } v_1 = \frac{1}{2}. \\
\end{array}\right.
\end{equation*}
Note that $\frac{1+2\epsilon}{3+2\epsilon} > r_2 = \frac{1+\epsilon}{3+2\epsilon} > \frac{1}{3} > r_3 = 0.01$ and since $R(S^*,\mb{v}) \geq R(\{1\},\mb{v})$, from \eqref{eq:threshold}, we have that optimal assortment is either $\{1\}$ or $\{1,2\}$, specifically, we have that
\begin{equation*}\label{eq:optimal_set_possibility}
S^* \in \left\{\{1\}, \{1,2\}\right\}.
\end{equation*}  
Therefore, we have, 
\begin{equation}\label{optimal_set} 
S^* = \left\{ 
\begin{array}{ll}
\{1\}, & \text{ if } v_1 = \frac{1}{2} + \epsilon, \\ 
\{1,2\}, & \text{ if } v_1 = \frac{1}{2}.
\end{array}\right.
\end{equation}
Note that since $r_3 < \frac{1}{3}$, for any $S$ and $i$, such that $i \geq 3$ and $i \not \in S$, we have $$R(S,\mb{v}) > R(S\cup \{i\},\mb{v}).$$ Therefore, if $v_i = \frac{1}{2}+\epsilon$, for any $S \neq \{1\}$, we have 
\begin{equation}\label{p_1}
R(\{1\},\mb{v}) - R(S,\mb{v}) \geq R(\{1\},\mb{v}) - R(\{1,2\},\mb{v}) \geq \frac{\epsilon}{20},
\end{equation}
and similarly if $v_i = \frac{1}{2}$, for any $S \neq \{1,2\}$, we have, 
\begin{equation}\label{p_0}
R(\{1\},\mb{v}) - R(S,\mb{v}) \geq R(\{1,2\},\mb{v}) - R(\{1\},\mb{v}) = \frac{\epsilon}{12} \geq\frac{\epsilon}{20},
\end{equation}

\medskip \noindent  Before providing the formal proof, we first present the intuition behind the result. Any algorithm that does not offer product $2$ when $v_1 = {1}/{2}$ will incur a regret and similarly any algorithm that offers product $2$ when $v_1 = {1}/{2} + \epsilon$. Hence, any algorithm that attempts to minimize regret on instance $\hat{I}_{\sf MNL}$ has to quickly learn if $v_1 = {1}/{2}+\epsilon$ or $v_1 = {1}/{2}.$ From Chernoff bounds, we know that we need approximately  ${1}/{\epsilon^2}$ observations to conclude with high probability if $v_1 = {1}/{2} + \epsilon$  or ${1}/{2}.$ Therefore in each of these ${1}/{\epsilon^2}$ time steps, we are likely to incur a regret of $\epsilon$, leading to a cumulative regret of ${1}/{\epsilon} \approx \sqrt{T}.$ In what follows, we will formalize this intuition on similar lines to the proof of Lemma \ref{lower_bound_MAB}. First, we present two auxillary results required to prove Lemma \ref{lower_bound}. 

\begin{lemma}\label{KL_div_categorical}
Let $S$ be an arbitrary subset of $\{1,\cdots,N\}$ and $\ep{P}^S_{0}, \ep{P}^S_{1}$ denote the probability distributions over the discrete space $\{0,1,\ldots, N\}$ governed by the MNL feedback on instance $\hat{I}_{MNL}$ when the offer set is ${S}$ and $v_1 = {1}/{2}$ and $v_1={1}/{2}+\epsilon$ respectively. In particular, we assume, 
$$
\ep{P}^S_{0}(i) = \frac{1}{2 + |{S}| }\times\left\{ 
\begin{array}{ll}
0, & \text{ if } i \not \in {S} \cup \{0\} ,\\
\displaystyle {2}, & \text{ if } i =0, \\ 
\displaystyle {1} & \text{ if } i \in S.
\end{array}\right.
\;, \;\ep{P}^S_{1}(i) = \frac{1}{2 + |{S}| + 2\epsilon\mathbbm{1}\left(1 \in S\right)}\times\left\{ 
\begin{array}{ll}
0, & \text{ if } i \not \in {S} \cup \{0\} ,\\
\displaystyle {2}, & \text{ if } i =0, \\ 
\displaystyle {1} & \text{ if } i \in S \backslash \{1\}\\
1+2\epsilon & \text{if}\;\; i = 1.
\end{array}\right.
$$
Then for any ${S}$, 
\begin{equation}
{\sf KL}\left( \ep{P}^S_{0}  \middle\| \ep{P}^S_{1} \right) \leq 4\epsilon^2,
\end{equation}
where {\sf KL} is the Kullback-Leibler divergence.
\end{lemma}
\begin{proof}{Proof.}
If $1 \not \in {S}$ , we have $\ep{P}^S_{0}$ and $\ep{P}^S_{1}$ to be the same distributions and the Kullback-Leibler divergence between them is 0. Therefore without loss of generality, assume that $1 \in {S} $. 
\begin{equation*}
\begin{aligned}
{\sf KL}\left( \ep{P}^S_{0}\middle\| \ep{P}^S_{1} \right)  &= \sum_{j=0}^N \ep{P}^S_{0}(j) \log\left(\frac{\ep{P}^S_{0}(j)}{\ep{P}^S_{1}(j)}\right),\\
&= \ep{P}^S_{0}(0) \log\left(\frac{\ep{P}^S_{0}(0)}{\ep{P}^S_{1}(0)}\right) + \sum_{j \in \{S\}\backslash 1} \ep{P}^S_{0}(j) \log\left(\frac{\ep{P}^S_{0}(j)}{\ep{P}^S_{1}(j)}\right)  + \ep{P}^S_{0}(1) \log\left(\frac{\ep{P}^S_{0}r(1)}{\ep{P}^S_{1}(1)}\right),\\
& = \frac{|S|+1}{|S|+2}\log{\left(1+\frac{2\epsilon}{2+|S|}\right)} +  \frac{1}{|S|+2} \log\left(1-\frac{2\epsilon(|S|+1)}{(2+|S|)(1+2\epsilon)} \right),  \\
& \leq \frac{2(|S|+1)\epsilon}{(|S|+2)^2}\left(1 -  \frac{1}{(1+2\epsilon)}\right) \leq 4\epsilon^2,
\end{aligned}
\end{equation*}
where the first inequality follows from the fact that for any $x \in (0,1)$, $$\log{(1+x)} \leq x \;\text{and}\; \log(1-x) \leq -x.$$ \hfill $\halmos.$
\end{proof}

\begin{lemma}\label{lem:KL_MNL}
Let $\mathbb{P}_{0}$ and $\mathbb{P}_{1}$ denote the probability distribution over consumer choices (throughout the planning horizon $T$) when assortments are offered according to algorithm $\ep{A}_{\sf MNL}$ and feedback to the algorithm is provided via the \banditMNL\; instances $\hat{I}_{MNL}$, when $v_1 ={1}/{2}$ and $v_1 ={1}/{2} + \epsilon$ respectively. Then, we have, 
$${\sf KL}\left( \mathbb{P}_{0}\middle\| \mathbb{P}_{1} \right)  \leq 4T\epsilon^2,$$
where ${\sf KL}\left( \mathbb{P}_{0}\middle\| \mathbb{P}_{1} \right)$ is the Kullback-Leibler divergence between the distributions $\mathbb{P}_{0}$ and $\mathbb{P}_{1}$. Specifically, 
\begin{equation}\label{KL_MNL}
{\sf KL}\left( \mathbb{P}_{0}\middle\| \mathbb{P}_{1} \right)  = \sum_{\mb{c} \in \{0,1, \cdots, N\}^T} \ep{P}(\mb{c})\log{\left(\frac{\ep{P}(\mb{c})}{\ep{P}_{1}(\mb{c})}\right)},
\end{equation}
where $\mb{c} \in \{0,1, \cdots, N\}^T$ is the observed set of choices by the algorithm $\ep{A}_{\sf MNL}$.  
\end{lemma}

\begin{proof}{Proof.}
From the chain rule for Kullback-Liebler divergence, it follows that, 
\begin{equation}\label{KL_CI}
{\sf KL}\left( \mathbb{P}_{0}\middle\| \mathbb{P}_{1} \right)  =  \sum_{t = 1}^T\; \sum_{\{c_1, \cdots, c_{t-1}\} \in \{0,1, \cdots, N\}^{t-1}} \mathbb{P}_{0}(\mb{c^t}) {\sf KL}\left( \mathbb{P}_{0}(c_t)\middle\| \mathbb{P}_{1}(c_t) \middle | c_1, \cdots, c_{t-1}\right),
\end{equation}
where, 
\begin{equation*}{\sf KL}\left( \mathbb{P}_{0}(c_t)\middle\| \mathbb{P}_{1}(c_t) \middle | c_1, \cdots, c_{t-1}\right) = \sum_{c_t} \mathbb{P}_{0}\left\{{c}_t \middle| c_1, \cdots, c_{t-1} \right\}\log{\left(\frac{\mathbb{P}_{0}\left\{{c}_t \middle| c_1, \cdots, c_{t-1} \right\}}{\mathbb{P}_{1}\left\{{c}_t \middle| c_1, \cdots, c_{t-1} \right\}}\right)}.\end{equation*}
Note that assortment  offered by $\ep{A}_{\sf MNL}$ at time $t$, $S_t$ is completely determined by the reward history $c_1, \ldots, c_{t-1}$ and conditioned on $S_t$, the reward at time $t$, $c_t$ is independent of the reward history $c_1, \cdots, c_{t-1}.$ Therefore, it follows that,
\begin{equation*}
\begin{aligned}
\mathbb{P}_{0}\left({c}_t \middle| c_1, \cdots, c_{t-1} \right)   = \ep{P}^{S_t}_0(c_t)\;\;\text{and}\;\;\mathbb{P}_{1}\left({c}_t \middle| c_1, \cdots, c_{t-1} \right)   &= \ep{P}^{S_t}_1(c_t),
\end{aligned}
\end{equation*}
and hence, we have, \begin{equation}\label{eq:indicator_KL}
{\sf KL}\left( \mathbb{P}_{0}(c_t)\middle\| \mathbb{P}_{1}(c_t) \middle | c_1, \cdots, c_{t-1}\right) =  {\sf KL}\left( \ep{P}^{S_t}_{0}(c_t) \middle\| \ep{P}^{S_t}_{1}(c_t) \right),
\end{equation}
where $\ep{P}^{S_t}_{0}$ and $ \ep{P}^{S_t}_{1}$ are defined as in Lemma \ref{KL_div_categorical}. Therefore from \eqref{KL_CI}, \eqref{eq:indicator_KL} and Lemma \ref{KL_div_categorical}, we have,
\begin{equation*}
\begin{aligned}
{\sf KL}\left( \mathbb{P}_{0}\middle\| \mathbb{P}_{1} \right)  &= \sum_{t = 1}^T  {\sf KL}\left( \ep{P}^{S_t}_{0} \middle\| \ep{P}^{S_t}_{1} \right)\leq 4T\epsilon^2.
\end{aligned}
\end{equation*}
\hfill $\halmos$
\end{proof}

\medskip \noindent {{\bf Proof of {Theorem~\ref{lower_bound}:}}} Fix a guessing algorithm $\ep{A}_{\sf MNL}$, which at time $t$ sees the consumer choice based on the offer assortment $S_t$. Let $\mathbb{P}_{0}$ and $\mathbb{P}_{1}$ denote the distributions for the view of the algorithm from time $1$ to $T$, when $v_1 = \frac{1}{2}$ and $v_1= \frac{1}{2}+\epsilon$ respectively. Let $T_2$ be the number of times $\ep{A}$ offers product $2$ and let $\mathbb{E}_{\mathbb{P}_{0}}(T_2)$ and $\mathbb{E}_{\mathbb{P}_{1}}(T_2)$ be the expected number of times product $2$ is offered by $\ep{A}$. 
\begin{equation}\label{pinsker_mnl}
\begin{aligned}
\left|\mathbb{E}_{\mathbb{P}_{0}}(T_2) - \mathbb{E}_{\mathbb{P}_{1}}(T_2)\right| &\leq  \left|\sum_{t=1}^T\ep{P}_0(2 \in S_t) - \mathbb{P}_{1}(2 \in S_t)\right|,\\
&{\leq} \sum_{t=1}^T \left|\mathbb{P}_{0}(2 \in S_t) - \mathbb{P}_{1}(2 \in S_t)\right|,\\
&{\leq} \sum_{t=1}^T \frac{1}{2}\sqrt{2\log{2} \cdot {\sf KL}\left( \mathbb{P}_{0}\middle\| \mathbb{P}_{1} \right)} =  \frac{T}{2}\sqrt{2\log{2} \cdot {\sf KL}\left( \mathbb{P}_{0}\middle\| \mathbb{P}_{1} \right)},
\end{aligned}
\end{equation}
where ${\sf KL}\left( \mathbb{P}_{0}\middle\| \mathbb{P}_{1} \right)$ as the Kullback-Leibler divergence between the distributions $\mathbb{P}_{0}$ and $\mathbb{P}_{1}$ as defined in \eqref{KL_MNL} and the last inequality follows from Pinsker's inequality. From Lemma \ref{lem:KL_MNL}, we have that, $${\sf KL}\left( \mathbb{P}_{0}\middle\| \mathbb{P}_{1} \right) \leq 4T\epsilon^2.$$ Substituting the value of $\epsilon$, we obtain ${\sf KL}\left( \mathbb{P}_{0}\middle\| \mathbb{P}_{1} \right) \leq \frac{1}{2}$ and from \eqref{pinsker_mnl}, we have
\begin{equation}\label{close_expected_value}
\left|\mathbb{E}_{\mathbb{P}_{0}}(T_2) - \mathbb{E}_{\mathbb{P}_{1}}(T_2)\right| \leq \frac{T}{4}.
\end{equation}
Since $v_1$ can be $\frac{1}{2}$ and $\frac{1}{2} + \epsilon$ with equal probability, we have
\begin{equation}\label{lb_conditional_reg_decompose}
\begin{aligned}
Reg_{\ep{A}_{\sf MNL}}(T,\mb{v}) = \frac{1}{2}Reg_{\ep{A}_{\sf MNL}}\left(T,\mb{v},\Big| v_1=\frac{1}{2}\right) + \frac{1}{2}Reg_{\ep{A}_{\sf MNL}}\left(T,\mb{v},\Big| v_1=\frac{1}{2}+\epsilon\right).
\end{aligned}
\end{equation}
From \eqref{p_0} we have that, in every time step we don't offer product $\{2\}$, we incur a regret of at least $\frac{\epsilon}{20}$ and hence it follows that, 
\begin{equation*}
\begin{aligned}
 Reg_{\ep{A}_{\sf MNL}}\left(T,\mb{v},\Big| v_1=\frac{1}{2}\right) \geq \frac{\epsilon}{20} (T-\mathbb{E}_{\mathbb{P}_{0}}(T_2)),
\end{aligned}
\end{equation*}
and similarly from \eqref{p_1} we have that, in every time step we offer product $\{2\}$, we incur a regret of at least $\frac{\epsilon}{20}$ and hence it follows that, 
\begin{equation*}
\begin{aligned}
 Reg_{\ep{A}_{\sf MNL}}\left(T,\mb{v},\Big| v_1=\frac{1}{2}+\epsilon\right) \geq \frac{\epsilon}{20} \mathbb{E}_{\mathbb{P}_{1}}(T_2).
\end{aligned}
\end{equation*}
Therefore, from \eqref{lb_conditional_reg_decompose} and \eqref{close_expected_value}, it follows that,
\begin{equation*}
\begin{aligned}
Reg_{\ep{A}_{\sf MNL}}(T,\mb{v}) &\geq \frac{\epsilon}{20} \left[T- (\mathbb{E}_{\mathbb{P}_{1}}(T_2)-\mathbb{E}_{\mathbb{P}_{0}}(T_2))\right]\geq \frac{3T\epsilon}{80}. 
\end{aligned}
\end{equation*}
\hfill $\halmos$



\end{appendices}

\end{document}